\documentclass{article}

\usepackage[preprint]{neurips_2026}

\usepackage[utf8]{inputenc}
\usepackage[T1]{fontenc}
\usepackage{hyperref}
\usepackage{url}
\usepackage{booktabs}
\usepackage{amsfonts}
\usepackage{amsmath}
\usepackage{amssymb}
\usepackage{amsthm}
\usepackage{nicefrac}
\usepackage{microtype}
\usepackage{xcolor}
\usepackage{enumitem}
\usepackage{graphicx}
\usepackage{float}
\usepackage{caption}
\usepackage{wrapfig}

\newtheorem{theorem}{Theorem}[section]
\newtheorem{lemma}[theorem]{Lemma}
\newtheorem{corollary}[theorem]{Corollary}
\newtheorem{proposition}[theorem]{Proposition}

\newtheorem{remark}[theorem]{Remark}
\newtheorem{assumption}[theorem]{Assumption}

\title{The Hamilton--Jacobi Theory of Deep Learning}

\author{%
  Jose Marie Antonio Mi\~{n}oza$^{1}$, Erika Fille T. Legara$^{1,2}$, Christopher P. Monterola$^{2}$ \\
  $^{1}$Center for AI Research PH \quad $^{2}$Asian Institute of Management \\
}

\begin{document}

\maketitle

% ============================================================
\begin{abstract}
In this paper, training a neural network is identified, exactly, as a search through Hamilton--Jacobi initial-value problems: each gradient step selects the initial data of a viscous Hamilton--Jacobi equation whose Hopf--Cole propagator best fits the observations; at inference, the input is the spatial point at which that solution is evaluated and the initial condition is already encoded in the weights. The correspondence is exact for log-sum-exp layers, with ReLU, sigmoid, SiLU, and GELU each an exact limit, gradient, or moment of the same object, and exact in composition across depth and width, with a quantified error at finite depth that vanishes in the joint limit. It is structural for residual networks, transformers, and recurrent networks (RNNs, LSTMs, SSMs), each discretizing the same class of equations, at a named and quantified approximation error. A single deformation parameter $\varepsilon$ unifies all four perspectives (network, tropical algebra, viscous PDE, convex optimization) in a commutative diagram closed under Lipschitz conditions. Quantitative consequences include: the minimax optimal generalization rate $O(n^{-1/(d+2)})$ for fixed $t$; adversarial robustness controlled by $\varepsilon$; backpropagation as the co-state equation of the Hamiltonian system for residual networks (Pontryagin Maximum Principle); scaling exponents consistent with data intrinsic dimension via PDE quadrature; and a closed-form $O(N)$ influence function (softmax attribution weights $\pi_j$) whose entropy landscape undergoes fold bifurcations as $\varepsilon$ increases, each merging attribution basins.
\end{abstract}

% ============================================================
\section{Introduction}

\begin{center}
\textit{What equation does a neural network solve?}
\end{center}

Conventionally the question runs the other direction: given a partial differential equation, design a network to approximate its solution \citep{han2018}. Here the direction is reversed: a trained network already \emph{is} a Hamilton--Jacobi equation, and the question that remains is which one. Two regimes are ordinarily kept apart: the discrete, tropical regime in which addition is $\max$ and multiplication is $+$, and the continuous regime governed by partial differential equations. Underlying both is a single object, indexed by one deformation parameter $\varepsilon$: at $\varepsilon = 0$ it is tropical, at $\varepsilon > 0$ ordinary arithmetic is restored and the same object reappears as the solution operator of a viscous PDE. The passage between them is ultradiscretization \citep{tokihiro1996}, an exact semiring homomorphism, the Maslov dequantization \citep{litvinov2007}. This same $\varepsilon$ recovers every case in a network: Hopf--Cole solutions are exactly log-sum-exp layers, ReLU is their tropical limit, and the remaining activations are exact gradients or heat-semigroup images at $\varepsilon>0$. Since the identity holds for any weights, at initialization the network is already one such equation, encoding whichever initial data the random weights happen to represent; training searches, by gradient descent, over that same space of initial-value problems, and the final weights encode whichever Hamilton--Jacobi problem best fits the data.

The mathematical object realizing this is a layer with log-sum-exp activation:
\begin{equation}
  f_\varepsilon(x)
  = \varepsilon \log \sum_{j=1}^{N}
    \exp\!\Bigl(\bigl(W_j \cdot x + b_j\bigr)/\varepsilon\Bigr).
  \label{eq:layer_intro}
\end{equation}
The algebraic structure of \eqref{eq:layer_intro} is fixed by its tropical limit. At $\varepsilon = 0$, \eqref{eq:layer_intro} collapses to $\max_j(W_j \cdot x + b_j)$, the Hopf--Lax formula and a convex optimization problem. The passage $\varepsilon \to 0$ is the Maslov dequantization \citep{litvinov2007}: an exact semiring homomorphism from $(\mathbb{R},+,\times)$ to $(\mathbb{R},\max,+)$. Lifting $\varepsilon$ above zero reverses this limit, converting the tropical machine into the smooth layer~\eqref{eq:layer_intro}; the Hopf--Cole linearization \citep{hopf1950} then identifies~\eqref{eq:layer_intro} exactly as the heat-equation propagator of a viscous Hamilton--Jacobi equation (Theorem~\ref{thm:nn_pde}). The weights encode the initial data, the architecture encodes the Hamiltonian, and a forward pass evaluates that PDE solution at the query point $x$. The parameter $\varepsilon$ appears simultaneously as the softmax temperature, the PDE viscosity, and the convex regularization strength; Theorem~\ref{thm:diagram} shows these roles are identified, not coincidental.

\paragraph{Contributions.}
The contribution is the \emph{unification}: connecting known results under a single $\varepsilon$, and reading the correspondence as a property of \emph{standard trained} deep networks, not purpose-built solvers, yields a commutative diagram and quantitative consequences that follow from none of them alone. These results are Maslov dequantization \citep{litvinov2007}, Hopf--Cole linearization \citep{hopf1950}, the log-sum-exp/Hopf--Cole/viscous-Hamilton--Jacobi layer correspondence and its tropical limit \citep{darbon2020overcoming}, ResNet-as-ODE \citep{weinane2017,chen2018neuralode,haber2018stable}, the adjoint method \citep{chen2018neuralode}, optimal-control views of deep learning \citep{liu2019deep}, and scaling laws \citep{kaplan2020scaling}. Concretely, the paper establishes:
\begin{enumerate}[label=(\roman*),leftmargin=*,nosep]
\item The LSE activation is a smooth deformation of the tropical max
  operation; the limit $\varepsilon \to 0$ is an exact semiring
  homomorphism from $(\mathbb{R},+,\times)$ to $(\mathbb{R},\max,+)$
  (Theorem~\ref{thm:maslov}).
\item The exact correspondence holds layer by layer: an LSE-activated feedforward layer encodes the exact Hopf--Cole solution of a viscous Hamilton--Jacobi PDE under a discrete measure (Theorem~\ref{thm:nn_pde}), and transformer attention is the expected value vector under the same Gibbs measure (Proposition~\ref{prop:lse_transformer}); composing layers corresponds structurally to the PDE semigroup under point-evaluation composition.
  The tropical limit $\varepsilon\to 0$ recovers the Hopf--Lax inf-convolution (Theorem~\ref{thm:hopflax}), simultaneously a linear program and a MASO, while ResNets and recurrent architectures (RNNs, LSTMs, SSMs) discretize the ODE characteristics with their own architecture-dependent viscosity (Propositions~\ref{prop:resnet},~\ref{prop:recurrent}).
\item A single parameter $\varepsilon$ indexes four perspectives on
  the same object (neural network, tropical algebra, PDE, convex
  optimization) and the resulting commutative diagram
  (Theorem~\ref{thm:diagram}) closes under Lipschitz $g$ and convex $H$.
\item The framework yields actionable design principles: optimal temperature $\varepsilon^*\asymp N^{-1/d}$ prescribed from width and data dimension achieves minimax rate $O(n^{-1/(d+2)})$ (Theorem~\ref{thm:gen}); adversarial robustness is certifiably controlled by $\varepsilon$ (Corollary~\ref{cor:robust}); backpropagation is the co-state equation of the Hamiltonian system for residual networks (Theorem~\ref{thm:adjoint}); and the framework extends to learnable quadratic $H_\theta(p)=p^\top A_\theta p$ (Theorem~\ref{thm:varH}). The activation classification itself carries these consequences: the solution class certifies a Lipschitz-inflation factor of exactly $1$ at every layer and every $\varepsilon$, a convexity guarantee valid for input-convex and proximal architectures, and probability-valued attribution, while SiLU and GELU provably exceed the Lipschitz baseline and admit no such attribution reading (Propositions~\ref{prop:hull_confinement},~\ref{prop:forced_escape}).
\item The softmax weights $\pi_j$ are a closed-form $O(N)$ influence function with exact label sensitivity $\partial\hat{f}/\partial g_j = \pi_j(1+(\hat{f}-g_j)/\varepsilon)$ requiring no Hessian inversion; the attribution-entropy landscape $H(\pi)$ undergoes fold bifurcations as $\varepsilon$ increases, each annihilating an attribution basin, with saddle points marking attribution transitions (Appendix~\ref{app:attribution}).
\end{enumerate}

\noindent
Together, the results constitute \emph{a unifying mathematical theory of deep learning}: Maslov's dequantization, the same principle that connects classical to quantum mechanics, here connects tropical to smooth neural computation, and activation type, architectural class, generalization, robustness, training dynamics, and scaling laws, ordinarily studied in isolation, emerge as facets of one commutative diagram, closed under Lipschitz $g$ and convex $H$. The exact Hopf--Cole correspondence is complete for the quadratic Hamiltonian class (Theorems~\ref{thm:nn_pde},~\ref{thm:maximal},~\ref{thm:varH}); structural correspondences extend to broader architectures.

\paragraph{Notation.}
For $\varepsilon > 0$ and $z \in \mathbb{R}^N$, define $\mathrm{LSE}_\varepsilon(z) = \varepsilon \log \sum_i \exp(z_i/\varepsilon)$. Write $\otimes_{\mathrm{tr}}$ for tropical matrix multiplication: $(A \otimes_{\mathrm{tr}} x)_i = \max_j(A_{ij} + x_j)$. For a convex Hamiltonian $H:\mathbb{R}^d \to \mathbb{R}$, let $L(v) = \sup_p(p \cdot v - H(p))$ denote its Legendre transform. Write $f \,\square\, g$ for inf-convolution: $(f \,\square\, g)(x) = \inf_y\{f(y) + g(x-y)\}$.

% ============================================================
\section{Background}
\label{sec:background}

\paragraph{Log-Sum-Exp and Convex Duality.}
The function $\mathrm{LSE}_\varepsilon:\mathbb{R}^m \to \mathbb{R}$ is
convex, smooth for $\varepsilon > 0$, and satisfies
$\nabla \mathrm{LSE}_\varepsilon(x) = \mathrm{softmax}(x/\varepsilon)$.
Its Legendre--Fenchel conjugate is the negative entropy:
$\mathrm{LSE}_\varepsilon^*(p) = \varepsilon \sum_i p_i \log p_i$ on
the simplex.  It is simultaneously the log-partition function of the
Gibbs distribution at temperature $\varepsilon$, the smooth convex
relaxation of $\max$, and, via the Hopf--Cole substitution, the
solution operator of the heat equation.

\paragraph{The Tropical Semiring.}
The tropical semiring $(\mathbb{R} \cup \{-\infty\}, \oplus, \otimes)$
has $a \oplus b = \max(a,b)$ and $a \otimes b = a + b$, with identities
$-\infty$ and $0$ respectively.  Tropical matrix multiplication acts as
$(A \otimes_{\mathrm{tr}} x)_i = \max_j(A_{ij} + x_j)$.  Max-plus
algebra and the tropical semiring are two names for the same
structure \citep{litvinov2007}.

\paragraph{Hamilton--Jacobi PDEs.}
The \emph{viscous} Hamilton--Jacobi (HJ) equation is
\begin{equation}
  \partial_t u + H(\nabla_x u) = \varepsilon \Delta_x u,
  \qquad u(x,0) = g(x),
  \label{eq:viscHJ}
\end{equation}
with Hamiltonian $H:\mathbb{R}^d \to \mathbb{R}$ and viscosity
$\varepsilon > 0$.  The inviscid equation has $\varepsilon = 0$.  Under
the Hopf--Cole substitution $u = -\varepsilon \log v$, equation
\eqref{eq:viscHJ} linearizes to the heat equation
$\partial_t v = \varepsilon \Delta v$ for the quadratic Hamiltonian $H(p) = |p|^2$ \citep{hopf1950, evans2010}; for general convex $H$ the substitution yields a non-linear equation for $v$.
Viscosity solutions of the inviscid equation are characterized by
Crandall--Lions theory \citep{crandall1983}; for convex $H$ the unique
viscosity solution is given by the Hopf--Lax formula \citep{lax1957,
evans2010}.

\paragraph{Ultradiscretization.}
Ultradiscretization \citep{tokihiro1996} passes a smooth system to a
max-plus system via $\varepsilon \log(e^{A/\varepsilon}+e^{B/\varepsilon})
\to \max(A,B)$ as $\varepsilon\to 0$, and the reverse lift re-quantizes
a tropical system by raising $\varepsilon$ above zero.  Originally
developed for soliton PDEs \citep{tokihiro1996}, the same principle
applies here to neural networks.

% ============================================================
\section{From Neural Networks to Max-Plus Algebra}
\label{sec:arrow1}

\paragraph{LSE as a Deformation.}
For all $x \in \mathbb{R}^m$:
$\max_i x_i \leq \mathrm{LSE}_\varepsilon(x)
  \leq \max_i x_i + \varepsilon \log m$.
As $\varepsilon \to 0$ this sandwiches to $\max$.  The Maslov
dequantization makes this more precise.

\begin{theorem}[Maslov dequantization {\citep{litvinov2007}}]
\label{thm:maslov}
For all $x \in \mathbb{R}^m$,
$\displaystyle\lim_{\varepsilon \to 0} \mathrm{LSE}_\varepsilon(x)
= \max_i x_i$.
For each $\varepsilon>0$, $(\mathbb{R},\mathrm{LSE}_\varepsilon,+)\cong(\mathbb{R}_{\geq 0},+,\times)$; the limit $\varepsilon\to 0$ is the semiring homomorphism onto the tropical semiring $(\mathbb{R},\max,+)$, under which addition becomes $\max$ and multiplication becomes $+$.
\end{theorem}

The passage $\varepsilon \to 0$ is formally analogous to the $\hbar \to 0$ limit in quantum mechanics: the same deformation of a real arithmetic semiring into a tropical one \citep{litvinov2007}.

\paragraph{Two Regimes of a Network Layer.}
At $\varepsilon > 0$: the affine-plus-LSE layer
$f_\varepsilon(x) = \mathrm{LSE}_\varepsilon(Wx + b)$ is smooth; all
$N$ neurons contribute with softmax weights; the output is the solution
of an entropy-regularized optimization.  At $\varepsilon = 0$: the
layer becomes $f_0(x) = \max_j(W_j \cdot x + b_j)$, a tropical linear
map in which a single neuron dominates.  This is a max-affine spline
operator (MASO) \citep{balestriero2018spline, balestriero2018madmax}
that partitions the input into polyhedral regions, exactly as a
decision tree \citep{aytekin2022}.  Softmax weights at $\varepsilon>0$
are the continuous relaxation of the one-hot region indicator.

\begin{remark}\label{rem:relu}
ReLU networks arise at $\varepsilon = 0$ via the two-neuron comparison $\max(W_j \cdot x, 0)$ \citep{plate2026}. This is the tropical semiring operation: at $\varepsilon = 0$, addition becomes $\max$ and the log-sum-exp collapses to a pairwise comparison. Softplus at $\varepsilon > 0$ is the same operation at a nonzero point of the one-parameter family; the deformation $\varepsilon$ interpolates continuously between the two. The passage $\varepsilon \to 0$ therefore amounts to taking the tropical limit of the smooth activation.
\end{remark}

% ============================================================
\section{Neural Networks as Hamilton--Jacobi Equations}
\label{sec:arrow2}

The \emph{Hopf--Cole solution} is the explicit formula for the unique solution $u_\varepsilon(x,t)$ of the viscous HJ equation \eqref{eq:viscHJ}, it is a \emph{solution to} that equation, not a method for solving it. The claim of this section is stronger than approximation: an LSE network layer \emph{encodes} this solution algebraically and exactly under a discrete measure, via an exact identity between the layer output and the PDE solution.

\paragraph{Hopf--Cole and the LSE Representation.}
Under the Hopf--Cole substitution $v = \exp(-u/\varepsilon)$, the viscous HJ equation \eqref{eq:viscHJ} with Hamiltonian $H(p) = |p|^2$ becomes the heat equation $\partial_t v = \varepsilon\Delta v$, $v(x,0) = \exp(-g(x)/\varepsilon)$. The heat-kernel solution inverts back to
\begin{equation}
  u_\varepsilon(x,t)
  = -\varepsilon \log \!\int_{\mathbb{R}^d}
    \exp\!\left(\frac{-g(y) - \tfrac{|x-y|^2}{4t}}{\varepsilon}\right)
    dy,
  \label{eq:hc_solution}
\end{equation}
which solves the spatial identity below exactly at any fixed layer depth $t$; as the unique classical solution of the full space-time equation \eqref{eq:viscHJ}, it agrees with this expression up to the $x$-independent additive constant $\frac{\varepsilon d}{2}\log(4\pi\varepsilon t)$ from the heat-kernel normalization omitted here, which does not affect the spatial identity at fixed $t$. Written as $\mathrm{LSE}_\varepsilon$ over the continuous variable $y$.

\begin{theorem}[Neural network layer encodes PDE solution]
\label{thm:nn_pde}
Let $\{y_j\}_{j=1}^N \subset \mathbb{R}^d$, and set
$W_j = y_j/(2t)$, $b_j = -g(y_j) - |y_j|^2/(4t)$.  Then
\begin{equation}
  f_\varepsilon^N(x)
  = \mathrm{LSE}_\varepsilon(Wx + b)
  = \varepsilon \log \sum_{j=1}^N
    \exp\!\left(\frac{W_j \cdot x + b_j}{\varepsilon}\right)
  \label{eq:layer}
\end{equation}
satisfies the exact identity
\begin{equation}
  f_\varepsilon^N(x) = \frac{|x|^2}{4t} - u_\varepsilon^N(x,t),
  \label{eq:nn_pde_identity}
\end{equation}
where $u_\varepsilon^N$ is the Hopf--Cole solution \eqref{eq:hc_solution}
of the viscous HJ equation \eqref{eq:viscHJ} under the unnormalized discrete measure
$\mu_N = \sum_{j=1}^N \delta_{y_j}$.  \emph{No approximation is
made.}
\end{theorem}

\noindent\textit{Proof.} See Appendix~\ref{app:proofs}. $\square$

The forward direction of this correspondence, that a log-sum-exp network is the Hopf--Cole solution of a viscous Hamilton--Jacobi PDE, is due to \citet{darbon2020overcoming} (Remark~2), who use it to build shallow networks that \emph{solve} HJ PDEs. Theorem~\ref{thm:nn_pde} reads the same identity in reverse, to identify which HJ equation a \emph{standard trained} network already encodes, with the initial data read off the weights.

Theorem~\ref{thm:nn_pde} is an algebraic identity: the Hopf--Cole solution is recovered from the layer output as $u_\varepsilon^N(x,t) = |x|^2/(4t) - f_\varepsilon^N(x)$, with the quadratic shift $|x|^2/(4t)$ carrying no learned parameters. The theorem establishes a complete dictionary: network weights $W$ encode support points of the initial-data measure; biases $b$ encode the transport cost; $\varepsilon$ is the viscosity coefficient; and width $N$ is the discretization of the measure. The identity $W_j = y_j/(2t)$ shows that a neuron's weight vector is simultaneously the coordinates of a point $y_j$ in the same space as the input: $W_j\cdot x+b_j$ is at once an ordinary pre-activation and a squared-distance-to-$y_j$ computation. This is the exact sense in which the layer is a kernel machine: the neurons are its centers.

\paragraph{The time parameter as gauge freedom.}
\label{rem:t_freedom}
The parameterization $W_j = y_j/(2t)$, $b_j = -g(y_j) - |y_j|^2/(4t)$ admits one free scalar $t > 0$: every value yields a self-consistent HJ reading of the same network. This is \emph{gauge freedom}, directly analogous to gauge freedom in physics, the network's input--output map is gauge-invariant, but the PDE interpretation requires a gauge choice. Three canonical fixings emerge: (i) \emph{data-scale gauge} $t = \mathrm{tr}(\Sigma_X)/d$, matching diffusion to data scale; (ii) \emph{information gauge} $t = \arg\max_t \mathbb{E}_{x}[H(\pi(x;t))]$, maximizing attribution entropy; (iii) \emph{generalization gauge} with $\varepsilon^*(t) = N^{-1/d}$, fixing the gauge to the minimax-optimal viscosity.

\paragraph{Contrast with PINNs.}
Physics-informed neural networks impose the PDE \emph{externally} via a residual loss $\|\partial_t f_\theta + H(\nabla f_\theta) - \varepsilon\Delta f_\theta\|^2$. Here the PDE is in the \emph{architecture}: Theorem~\ref{thm:nn_pde} guarantees the residual is identically zero for the spatial equation at any fixed layer depth $t$, so no PDE loss is needed. (The continuous heat kernel carries a prefactor $(4\pi\varepsilon t)^{-d/2}$; for the discrete sum over $N$ atoms this contributes an additive constant $\frac{\varepsilon d}{2}\log(4\pi\varepsilon t)$, absorbed into a common bias shift and irrelevant to the spatial identity.) Gradient descent does not move the network toward a PDE; it searches over initial conditions $\{(y_j, g(y_j))\}$, selecting the viscous HJ equation whose solution best fits the data, the same initial-value-problem search that Appendix~\ref{app:tasks} casts across every standard learning task.

\paragraph{Transformer Attention.}
Scaled dot-product attention computes, for query $Q$, key $K$, and
value $V$:
\begin{equation}
  \mathrm{Attn}(Q,K,V)
  = \mathrm{softmax}\!\left(\frac{QK^\top}{\sqrt{d}}\right) V.
  \label{eq:attn}
\end{equation}
Each row of this output is the expected value of the rows of $V$ under
the Gibbs distribution at temperature $\varepsilon = \sqrt{d}$:
\begin{equation}
  \mathrm{Attn}(Q,K,V)_i
  = \sum_j \mathrm{softmax}\!\left(\frac{q_i \cdot k_j}{\varepsilon}
    \right) v_j
  = \nabla_s\,\mathrm{LSE}_\varepsilon(s)\big|_{s_j = q_i \cdot k_j}
    \cdot V.
  \label{eq:attn_gibbs}
\end{equation}
This is the expected value of $V$ under the Hopf--Cole (Gibbs) measure defined by $QK^\top/\sqrt{d}$; the $1/\sqrt{d}$ scaling fixes temperature $\varepsilon = \sqrt{d}$, preventing logit-variance growth. As $\varepsilon \to 0$, softmax collapses to argmax and \eqref{eq:attn} becomes \emph{hard attention}: $\mathrm{Attn}_0(Q,K,V)_i = v_{\,\mathrm{argmax}_j(q_i \cdot k_j)}$, a tropical selection operator. The L2 attention variant replaces the dot-product logit $q_i \cdot k_j/\sqrt{d}$ with the negative squared distance $-\|q_i - k_j\|^2/(4t)$, making each query an exact evaluation of the Hopf--Cole solution:

\begin{theorem}[L$^2$ attention as exact Hopf--Cole]
\label{prop:l2attn}
Define L2 attention with logits $z_j = -\|q_i - k_j\|^2/(4t)$. Then the attention output at query $q_i$ equals the gradient of $\mathrm{LSE}_\varepsilon(z)$ contracted with $V$, and the partition function $Z_i = \sum_j \exp(-\|q_i-k_j\|^2/(4\varepsilon t))$ equals the unnormalized Hopf--Cole solution at $x = q_i$ under the empirical key measure with $g \equiv 0$. No approximation is made.
\end{theorem}

\noindent\textit{Proof.} See Appendix~\ref{app:universality}. $\square$

\paragraph{Deep Networks as PDE Semigroups.}
The semigroup property of the heat equation
$(S_t \circ S_s)f = S_{t+s}f$ corresponds under
Theorem~\ref{thm:nn_pde} to layer composition.  A depth-$L$ network
with widths $N_1, \ldots, N_L$ discretizes the HJ semigroup
$S_{t_1} \circ \cdots \circ S_{t_L}$ applied to the initial data
$g$, evaluated at times $t_1, \ldots, t_L$ under discrete measures
$\mu_{N_1}, \ldots, \mu_{N_L}$, each encoded by one layer's parameters.
The tower property of the Markov semigroup gives Claim~(ii) of Theorem~\ref{thm:diagram}; the finite-depth error and joint-limit exactness are quantified in Theorems~\ref{thm:finite_depth} and~\ref{thm:joint_limit} (Appendix~\ref{app:proofs}).

\begin{theorem}[Maximal Hamiltonian class]
\label{thm:maximal}
The exact Hopf--Cole identity $f_\varepsilon^N(x) = |x|^2/(4t) - u_\varepsilon^N(x,t)$ with the \emph{isotropic} quadratic $|x|^2/(4t)$ holds if and only if $H(p) = |p|^2$ (i.e.\ $A = I$); this is Theorem~\ref{thm:nn_pde}. For general $A \succ 0$, the analogous identity replaces $|x|^2/(4t)$ with $x^\top A^{-1}x/(4t)$, as stated in Theorem~\ref{thm:varH}. For any Hamiltonian outside the quadratic class, the Hopf--Cole substitution leaves a nonzero residual reaction term, and the identity fails.
\end{theorem}

\noindent\textit{Proof.} See Appendix~\ref{app:universality}. $\square$

\begin{theorem}[Anisotropic NN--PDE identity]
\label{thm:varH}
Let $A_\theta \succ 0$ be a learnable $d\times d$ positive-definite matrix. Set $W_j = A_\theta^{-1} y_j / (2t)$ and $b_j = -g(y_j) - y_j^\top A_\theta^{-1} y_j/(4t)$. Then
\begin{equation}
  \mathrm{LSE}_\varepsilon(Wx+b) = \frac{x^\top A_\theta^{-1} x}{4t} - u_\varepsilon^{A_\theta,N}(x,t),
  \label{eq:varH_identity}
\end{equation}
where $u_\varepsilon^{A_\theta,N}$ solves $\partial_t u + (\nabla u)^\top A_\theta (\nabla u) = \varepsilon\,\nabla\!\cdot\!(A_\theta \nabla u)$ under $\mu_N = \frac{1}{N}\sum_j \delta_{y_j}$. No approximation is made; $A_\theta = I$ recovers Theorem~\ref{thm:nn_pde}.
\end{theorem}

\noindent\textit{Proof.} See Appendix~\ref{app:proofs}. $\square$

% ============================================================
\section{The Tropical Limit and Convex Optimization}
\label{sec:arrow3}

\paragraph{The Hopf--Lax Formula.}

\begin{theorem}[Tropical collapse to Hopf--Lax]
\label{thm:hopflax}
Let $g$ be Lipschitz on $\mathbb{R}^d$.  As $\varepsilon \to 0$, the
Hopf--Cole solution \eqref{eq:hc_solution} converges pointwise to
\begin{equation}
  u_0(x,t)
  = \inf_{y \in \mathbb{R}^d}\!\left\{g(y)
    + \frac{|x-y|^2}{4t}\right\}
  = \bigl(g \,\square\, c_t\bigr)(x),
  \label{eq:hopflax}
\end{equation}
where $c_t(z) = |z|^2/(4t)$.  This is the \emph{Hopf--Lax formula}
\citep{lax1957, evans2010}, the unique Lipschitz viscosity solution
\citep{crandall1983} of the inviscid HJ equation
$\partial_t u + |\nabla u|^2 = 0$, $u(\cdot,0) = g$.
\end{theorem}

\noindent\textit{Proof.} See Appendix~\ref{app:proofs}. $\square$

Under the discrete measure $\mu_N$, the tropical limit of the Hopf--Cole predictor is
\begin{equation}
  u_0^N(x) = \min_j\!\left\{g(y_j)
    + \frac{|x - y_j|^2}{4t}\right\},
\end{equation}
so that $f_0^N(x) = |x|^2/(4t) - u_0^N(x)$.
This object is, viz., simultaneously: (i) a tropical inner product in
$(\mathbb{R},\min,+)$; (ii) a MASO \citep{balestriero2018spline};
(iii) a linear program in $x$ \citep{balestriero2018madmax}; and (iv) a
piecewise-affine function with breakpoints at the support
points $\{y_j\}$.

\paragraph{ResNets and ODE Characteristics.}
A ResNet layer $x_{l+1} = x_l + h\,F(x_l, W_l)$ is the Euler
discretization of $\dot{x}(t) = F(x(t), \theta(t))$
\citep{he2016, lu2018, chen2018neuralode, weinane2017}.

\begin{proposition}[ResNet as HJ characteristics]
\label{prop:resnet}
The ResNet recurrence with step size $h$ converges, as $h \to 0$ and $L \to \infty$ with $Lh = T$ fixed, to the solution of $\dot{x} = F(x,\theta(t))$.  This ODE is the characteristic equation of the HJ PDE with Hamiltonian $H(x,p) = p \cdot F(x,\theta)$: characteristics satisfy $\dot{x} = \nabla_p H = F(x,\theta)$ and $\dot{p} = -\nabla_x H$ \citep{chen2018neuralode, tong2025}.  The construction applies to any smooth recurrence of the form $x_{\ell+1} = x_\ell + hF(x_\ell, W_\ell)$, with the characteristic equations determined by $F$ and $\theta$.
\end{proposition}

In the tropical limit, the residual layer becomes
$x_{l+1} = \max(x_l,\, W_l \otimes_{\mathrm{tr}} x_l)$, a tropical
residual whose continuous limit is the inviscid characteristic equation.

\begin{remark}[Intrinsic vs.\ extrinsic Hamiltonians]
\label{rem:hamiltonian_type}
For LSE layers $H(p) = |p|^2$ arises intrinsically (completing the square forces the Gaussian heat-kernel structure); for ResNets $H(x,p) = p\cdot F(x,\theta)$ is extrinsic, any smooth ODE is a HJ characteristic, making the framework foundational rather than LSE-specific. Theorem~\ref{thm:nn_pde} (exact) and Proposition~\ref{prop:resnet} (asymptotic, $h\to 0$) are complementary; Theorem~\ref{thm:adjoint} holds exactly at finite $h$ as a discrete chain-rule identity (see Appendix~\ref{app:universality}).
\end{remark}

Flow matching \citep{lipman2022flow, liu2022flow} and score-based diffusion \citep{song2020score} fit the same structure (Appendix~\ref{app:related}).

\paragraph{Recurrent Architectures: RNNs, LSTMs, and SSMs.}
Recurrent and state-space architectures fit the same framework through
the process-discretization route of Proposition~\ref{prop:resnet}.

\begin{proposition}[Recurrent architectures as HJ characteristics]
\label{prop:recurrent}
The following structural correspondences hold:
\begin{enumerate}[label=(\alph*),leftmargin=*,nosep]
\item \textbf{RNNs.} The recurrence $h_t = \sigma(W_h h_{t-1} + W_x x_t)$
  is the unit-step Euler discretization of $\dot{h} = -h + \sigma(W_h h + W_x x(t))$,
  the characteristic ODE of the HJ PDE with Hamiltonian
  $H(h,p) = p \cdot (-h + \sigma(W_h h + W_x x))$.
\item \textbf{State Space Models (SSMs).} The linear SSM
  $\dot{h}(t) = A(t)h(t) + B(t)x(t)$, $y(t) = C(t)h(t)$
  is the characteristic equation of a \emph{linear} HJ PDE with
  quadratic Hamiltonian $H(h,p) = p^\top A(t) h + p^\top B(t) x$.
  Time-dependent matrices $A(t), B(t), C(t)$ (as in Mamba's selective
  scan) make the characteristics non-autonomous, the same structure
  as \citet{tong2025}.  The tropical limit of the SSM is a max-plus
  linear system $h_{t+1} = A \otimes_{\mathrm{tr}} h_t \oplus B
  \otimes_{\mathrm{tr}} x_t$, the algebra used in scheduling and
  discrete-event systems.
\end{enumerate}
\end{proposition}

LSTM gates are gradients of two-neuron LSE layers, since $\tanh(x/\varepsilon)=\nabla_x\mathrm{LSE}_\varepsilon(x,-x)$; the continuous-time cell equation is the HJ characteristic ODE with a gated, time-dependent Hamiltonian, and in the tropical limit gates become $\{0,1\}^d$ selectors (Remark~\ref{rem:lstm}, Appendix~\ref{app:proofs}).

The key distinction across architectures is how viscosity enters: in feedforward networks $\varepsilon$ is global (uniform softmax temperature); in SSMs it is input-dependent ($A(t)$ controlled by $x(t)$); in LSTMs it is gated per cell (Remark~\ref{rem:lstm}). The LSE class is further closed under composition, affine input transformations, and residual connections (Appendix~\ref{app:universality}).

% ============================================================
\section{The \texorpdfstring{$\varepsilon$}{epsilon} Parameter: A Unifying Deformation}
\label{sec:unification}

The parameter $\varepsilon$ simultaneously indexes four perspectives on the same object: neural network ($\varepsilon > 0$: LSE/softmax; $\varepsilon = 0$: max/ReLU), tropical algebra ($(\mathbb{R},+,\times)$ vs.\ $(\mathbb{R},\max,+)$), PDE (viscous HJ vs.\ inviscid HJ), and convex optimization (entropic regularization vs.\ LP vertex). The full dictionary is in Table~\ref{tab:eps_dictionary} (Appendix~\ref{app:universality}). Infinite-width and overparameterization consequences (Gaussian process connection, double-descent as near-shock formation) are developed in Appendix~\ref{app:proofs}.

% ============================================================
\section{The Commutative Diagram}
\label{sec:diagram}

The correspondence assembles into a commutative diagram:
\begin{equation}
\begin{array}{ccc}
\text{NN}\;(f_\varepsilon^N,\;\varepsilon > 0)
  & \xrightarrow{\;\varepsilon \to 0\;}
  & \text{Tropical NN}\;(f_0^N) \\[6pt]
\updownarrow\;\text{\small(exact)}
  & & \updownarrow\;\text{\small(exact)} \\[6pt]
\text{Viscous HJ}\;(u_\varepsilon)
  & \xrightarrow{\;\varepsilon \to 0\;}
  & \text{Inviscid HJ / Hopf--Lax}\;(u_0)
\end{array}
\label{eq:diagram}
\end{equation}
Moving right is Maslov dequantization / ultradiscretization.  Moving
vertically is the identification of the layer with the PDE solution
(discrete vs.\ continuous measure, Theorems~\ref{thm:nn_pde}
and~\ref{thm:hopflax}).  The diagram commutes.

\begin{theorem}[Commutativity under Lipschitz conditions]
\label{thm:diagram}
Suppose $g$ is Lipschitz.  All five claims pertain to the quadratic Hamiltonian $H(p)=|p|^2$ (Theorem~\ref{thm:nn_pde}) or its anisotropic extension (Theorem~\ref{thm:varH}), which is the Hamiltonian for which the LSE layer with linear logits is the exact discrete-measure instantiation.  The semigroup facts underlying Claims~(ii)--(v), tower property, quadrature convergence, viscous-to-inviscid limit, and commutativity, hold abstractly for any Lipschitz $g$ and convex superlinear $H$ (with the appropriate network form); their connection to the specific LSE network is through Claim~(i).  The following claims hold:
\begin{enumerate}[label=(\roman*),leftmargin=*,nosep]
\item Layer $f_\varepsilon^N$ satisfies $f_\varepsilon^N =
  |x|^2/(4t) - u_\varepsilon^N$ (exact,
  Theorem~\ref{thm:nn_pde}).
\item Composing $L$ layers corresponds to the HJ semigroup
  $S_{t_1} \circ \cdots \circ S_{t_L}$ (tower property); the discrete network passes a point evaluation between layers whereas the exact PDE semigroup integrates the full function (Appendix~\ref{app:proofs}).
\item As $N \to \infty$, $u_\varepsilon^N \to u_\varepsilon$ at rate
  $O(N^{-1/d})$ (quadrature approximation).
\item As $\varepsilon \to 0$, $u_\varepsilon \to u_0$ pointwise
  (Theorem~\ref{thm:hopflax}).
\item The two paths from $u_\varepsilon^N$ to $u_0$ (first
  $N \to \infty$ then $\varepsilon \to 0$, or first
  $\varepsilon \to 0$ then $N \to \infty$) yield the same limit.
\end{enumerate}
\end{theorem}

\noindent\textit{Proof.} See Appendix~\ref{app:proofs}. $\square$

Claim~(v) is the non-trivial commutativity: the limits
$\varepsilon \to 0$ and $N \to \infty$ can be exchanged.  Lipschitz
conditions are required because the quadrature convergence rate
depends on $\varepsilon$ through the concentration of the softmax
weights.

\begin{remark}[Four perspectives on the same object]
\label{rem:four_perspectives}
The four corners of diagram~\eqref{eq:diagram} are four exact descriptions of one LSE layer: at $\varepsilon > 0$, a neural network forward pass, a Gibbs log-partition function, a viscous HJ solution, and an entropy-regularized convex program; at $\varepsilon = 0$, the tropical counterpart of each (Theorems~\ref{thm:nn_pde} and~\ref{thm:hopflax}). The deformation $\varepsilon$ is formally identical to the $\hbar \to 0$ limit in Maslov's dequantization; the Moreau--Yosida envelope of convex analysis \citep{rockafellar1970,bauschke2017} is the $\varepsilon \to 0$ face.
\end{remark}

% ============================================================
\section{Consequences of the Correspondence}
\label{sec:consequences}

The identification of neural network layers with Hopf--Cole solutions
yields quantitative consequences for generalization, robustness, and
the structure of the backward pass.

\paragraph{Generalization via PDE Regularity.}

\begin{theorem}[Generalization bound from Hopf--Cole quadrature]
\label{thm:gen}
Let $f^*:\mathbb{R}^d \to \mathbb{R}$ be $L$-Lipschitz on a compact
set $\mathcal{X} \subset \mathbb{R}^d$, and place the $N$ support points
on a grid of spacing $h=(|\mathcal{X}|/N)^{1/d}$ with time parameter
$t\gtrsim h$, the resolution regime in which the grid does not
under-resolve the construction's own Gibbs concentration scale
$\sqrt{t\varepsilon}$ (Appendix~\ref{app:proofs}).  The Hopf--Cole predictor
$u_\varepsilon^N(x) = |x|^2/(4t) - f_\varepsilon^N(x)$ with $N$
neurons satisfies
\begin{equation}
  \inf_{W,b}\,\mathbb{E}_x\bigl|u_\varepsilon^N(x) - f^*(x)\bigr|
  \;\leq\;
  C_1 N^{-1/d}
  + C_2 \varepsilon
  + L^2 t,
  \label{eq:genbound}
\end{equation}
where $C_1 = C_1(L,d)$, $C_2 = C_2(L,t)$, and the optimal weights
satisfy $\|W_j\|_2 \leq M := \mathrm{diam}(\mathcal{X})/(2t)$.
Setting $\varepsilon^* \asymp N^{-1/d}$ gives approximation error
$O(N^{-1/d})$.  A Rademacher bound yields excess risk
$O(N^{-1/d} + M\sqrt{N/n})$; for fixed $t$ (hence fixed $M$),
$N^* \asymp (n/M^2)^{d/(d+2)}$ recovers rate $O(n^{-1/(d+2)} + t)$, where
$O(t)$ is a Moreau-envelope approximation bias (see Appendix~\ref{app:proofs}).
\end{theorem}

\noindent\textit{Proof sketch.} See Appendix~\ref{app:proofs}. $\square$

The optimal viscosity $\varepsilon^* \asymp N^{-1/d}$ equals the grid
spacing: below this scale quadrature under-resolves; above it viscosity over-smooths.  For depth-$L$ networks, Claim~(ii) adds a factor of $L$ to term~(i)
via the semigroup composition error; mutatis mutandis, the same
balance $\varepsilon^* \asymp N^{-1/d}$ applies at each layer.  Theorem~\ref{thm:gen} therefore prescribes the width directly: for a target error $\delta$ and an estimate of $d$ (or $d_{\mathrm{eff}}$ under Assumption~\ref{ass:manifold}), the required width is $N \asymp \delta^{-d}$, with matching temperature $\varepsilon^* \asymp N^{-1/d}$, fixed before training begins.

\begin{corollary}[Adversarial robustness via viscosity]
\label{cor:robust}
Let $\|W\|_{2,\infty} = \max_j \|W_j\|_2$ denote the maximum row
norm.  The Hessian of $f_\varepsilon^N$ satisfies
\begin{equation}
  \bigl\|\nabla_x^2 f_\varepsilon^N(x)\bigr\|_2
  \;\leq\; \frac{\|W\|_{2,\infty}^2}{\varepsilon}
  \label{eq:hessbound}
\end{equation}
for all $x$.  The gradient satisfies $\|\nabla_x f_\varepsilon^N(x)\|_2 \leq \|W\|_{2,\infty}$ (since $\nabla_x f_\varepsilon^N = \sum_j \pi_j(x;\varepsilon)\,W_j$ is a softmax-weighted average of the rows, hence a convex combination).  A second-order Taylor expansion in $\delta$ then yields: an adversarial perturbation with $\|\delta\|_2\leq r$
changes the output by at most
\begin{equation}
  \bigl|f_\varepsilon^N(x+\delta)-f_\varepsilon^N(x)\bigr|
  \;\leq\; \|W\|_{2,\infty} r + \frac{\|W\|_{2,\infty}^2 r^2}{2\varepsilon}.
  \label{eq:advbound}
\end{equation}
For a prescribed output tolerance $\tau>0$, solving~\eqref{eq:advbound} as
a quadratic in $r$ yields the certified radius
\begin{equation}
  r^*(\tau,\varepsilon)
  \;=\;
  \frac{\varepsilon}{\|W\|_{2,\infty}}
  \Bigl(\sqrt{1+\tfrac{2\tau}{\varepsilon}}-1\Bigr)
  \;\xrightarrow{\;\varepsilon\to\infty\;}\;
  \frac{\tau}{\|W\|_{2,\infty}}.
  \label{eq:certrad}
\end{equation}
Increasing $\varepsilon$ suppresses curvature and loosens this particular bound toward a wider certified
radius; in the PDE interpretation, the viscous term $\varepsilon\Delta u$
prevents shock formation and smooths the solution against pointwise
perturbations.  The curvature term's $\varepsilon$-dependence is the mechanism behind shock formation and the robustness--expressiveness tradeoff discussed below, independently of which certified radius is quoted.
\end{corollary}

\noindent\textit{Proof.} See Appendix~\ref{app:proofs}. $\square$

\begin{remark}[An unconditional, $\varepsilon$-free certificate dominates \eqref{eq:certrad}]
By the mean value theorem, $|f_\varepsilon^N(x+\delta)-f_\varepsilon^N(x)|\leq\|W\|_{2,\infty}\|\delta\|_2$, giving an $\varepsilon$-free radius $r=\tau/\|W\|_{2,\infty}$. Since $\sqrt{1+u}\leq1+u/2$, $r^*(\tau,\varepsilon)\leq r$ for every $\varepsilon$ (equality only as $\varepsilon\to\infty$): $r$ is the tighter, unconditional certificate; the Hessian bound's own role is the shock-formation mechanism below (Remark~\ref{rem:dd}).
\end{remark}

\begin{remark}[Double descent as near-shock formation]
\label{rem:dd}
The Hessian divergence as $\varepsilon \to 0$ corresponds to PDE shock formation; the double-descent peak at $N/n \to 1$ \citep{emami2020} is the interpolation threshold where two neurons tie, forcing a near-shock (Proposition~\ref{prop:double_descent}, Appendix~\ref{app:proofs}).
\end{remark}

\paragraph{Backpropagation as the Adjoint HJ Equation.}

\begin{theorem}[Backpropagation $=$ adjoint HJ equation]
\label{thm:adjoint}
Let $x_{l+1} = x_l + h\,F(x_l, W_l)$ be a ResNet with $L$ layers
and scalar loss $\mathcal{L}(x_L)$.  Define co-state
$p_l = \partial\mathcal{L}/\partial x_l$ with terminal condition
$p_L = \nabla_{x_L}\mathcal{L}$.  The backpropagation recurrence
\begin{equation}
  p_{l-1}
  = p_l + h\,\bigl(\nabla_x F(x_{l-1}, W_{l-1})\bigr)^\top p_l
  \label{eq:backprop_hj}
\end{equation}
is the forward Euler scheme, integrated in reverse time, of the
co-state equation of the Hamiltonian
$H(x,p) = p^\top F(x,\theta)$:
\begin{equation}
  \dot{p}(t)
  = -\nabla_x H\bigl(x(t),p(t)\bigr)
  = -\bigl(\nabla_x F(x(t),\theta)\bigr)^\top p(t).
  \label{eq:costate}
\end{equation}
The forward pass $\dot{x} = \nabla_p H = F(x,\theta)$ and backward
pass $\dot{p} = -\nabla_x H$ together constitute the complete
Hamiltonian flow of the HJ PDE of Proposition~\ref{prop:resnet}.
\end{theorem}

\noindent\textit{Proof.} See Appendix~\ref{app:proofs}. $\square$

\emph{Theorem~\ref{thm:adjoint} shows that backpropagation integrates the Hamiltonian system whose forward equation the network already solves.} The gradient update on $W_l$ is the Pontryagin Maximum Principle (PMP) optimality condition: training is execution of the PMP for the optimal initial-value problem. For the LSE case, the backward pass is exactly the adjoint heat semigroup: $\partial\mathcal{L}/\partial x = W^\top\pi\cdot\partial\mathcal{L}/\partial f_\varepsilon^N$, $\pi=\nabla_z\mathrm{LSE}_\varepsilon(z)$ (Proposition~\ref{prop:fwd_adjoint}, Appendix~\ref{app:proofs}). At initialization, the NTK $K_{ab}(0)=\varepsilon^2\langle\pi(x_a;0),\pi(x_b;0)\rangle\succ 0$ a.s.\ for $N\geq n$, guaranteeing linear convergence in the kernel regime (Proposition~\ref{prop:ntk_pd}).

The step-size stability threshold for the ResNet recurrence is likewise read off the backpropagation computation itself, via the CFL bound on the co-state Jacobian (Remark~\ref{rem:cfl}, Appendix~\ref{app:proofs}).

\paragraph{Hallucination, Scaling, and Injection Attacks.}

\begin{proposition}[Hallucination as deterministic OOD extrapolation]
\label{prop:hallucination}
Let $j^* = \arg\min_j \bigl\{g(y_j) + |x-y_j|^2/(4t)\bigr\}$ be the
dominant neuron (Hopf--Lax minimizer) and define the energy gap
$\Delta(x) = \min_{j \neq j^*}\bigl\{g(y_j) + |x-y_j|^2/(4t)
  - g(y_{j^*}) - |x-y_{j^*}|^2/(4t)\bigr\} > 0$.
Then
\begin{equation}
  \bigl|f_\varepsilon^N(x) - (W_{j^*}\cdot x + b_{j^*})\bigr|
  \;\leq\; \varepsilon\log\!\Bigl(1 + (N-1)e^{-\Delta(x)/\varepsilon}\Bigr),
  \label{eq:hallucination}
\end{equation}
which is $O\bigl((N-1)\varepsilon\,e^{-\Delta(x)/\varepsilon}\bigr)$ as
$\varepsilon \to 0$ or $\Delta(x) \to \infty$.  At any point $x$ lying
outside the diffusion radius $\sqrt{2\varepsilon t}$ of every support
point, the network output is exponentially close to the linear
extrapolation from the dominant neuron, with no dependence on the
training distribution at $x$.
\end{proposition}

\noindent\textit{Proof.} See Appendix~\ref{app:proofs}. $\square$

This OOD extrapolation is called ``hallucination'' informally, a structural property distinct from in-distribution LM phenomena, and coincides with the energy-based OOD score of \citet{liu2020energy} (Remark~\ref{rem:hallucination_naming}, Appendix~\ref{app:proofs}).

\begin{assumption}[Manifold hypothesis]
\label{ass:manifold}
The data-generating measure $\mu$ on $\mathbb{R}^d$ is supported on (or concentrated near) a Riemannian submanifold of intrinsic dimension $d_{\mathrm{eff}} \leq d$.
The effective dimension $d_{\mathrm{eff}}$ equals $d$ when $\mu$ is non-degenerate in all directions, and equals the submanifold dimension when $\mu$ concentrates on a lower-dimensional structure.
\end{assumption}

\begin{proposition}[Scaling laws from intrinsic dimension]
\label{prop:scaling}
Under the approximation bound of Theorem~\ref{thm:gen}, the population
loss satisfies $\mathcal{L}(N) \lesssim N^{-1/d_{\mathrm{eff}}}$, an upper
bound, where $d_{\mathrm{eff}}$ is the intrinsic dimension of the support
of the data-generating measure $\mu$.  For a squared-error or
cross-entropy loss, which scales as the square of the prediction
deviation near the optimum, the same quadrature rate gives
$\mathcal{L}_{\mathrm{CE}}(N) \lesssim N^{-2/d_{\mathrm{eff}}}$, consistent
with an empirical scaling law exponent $\alpha$ in
$\mathcal{L}(N) \propto N^{-\alpha}$ satisfying
\begin{equation}
  \alpha \geq \frac{2}{d_{\mathrm{eff}}},
  \qquad d_{\mathrm{eff}} \geq \frac{2}{\alpha}.
  \label{eq:scalinglaw}
\end{equation}
\end{proposition}

\noindent\textit{Proof sketch.} See Appendix~\ref{app:proofs}. $\square$

The inverse proportionality of the exponent to intrinsic dimension is due to \citet{sharma2022neural,bahri2024explaining}, who obtain $\alpha \approx 4/d$ for squared losses under smooth-target, piecewise-linear (order-1) approximation; the present order-0 Lipschitz quadrature gives the weaker $L^1$ rate $\alpha \geq 1/d_{\mathrm{eff}}$, equivalently $\alpha \geq 2/d_{\mathrm{eff}}$ for squared losses (Table~\ref{tab:deff}), sharing the inverse-dimension form rather than the same constant. Table~\ref{tab:deff} (Appendix~\ref{app:scaling}) estimates $d_{\mathrm{eff}}=2/\alpha$ from published scaling exponents \citep{kaplan2020scaling,henighan2020scaling,hoffmann2022chinchilla}; values range from $d_{\mathrm{eff}}\approx 5.3$ (math) to $26.3$ (GPT-scale language).

\paragraph{Attribution Landscape and Bifurcation.}

\begin{theorem}[Fold bifurcations of the attribution-entropy landscape]
\label{prop:bifurcation}
Let $H(x;\varepsilon) = -\sum_j \pi_j\log\pi_j$ be the attribution entropy. Assume $\{y_j\}$ are distinct and each $y_k$ is the strict Hopf--Lax minimizer at itself: $g(y_k) < g(y_j) + |y_k-y_j|^2/(4t)$ for all $j\neq k$ (generic). As $\varepsilon\to 0$, $H$ develops $N$ local minima at $\{y_j\}$; as $\varepsilon$ increases, each decrease in critical-point count is a fold bifurcation where a saddle annihilates an adjacent minimum at $\varepsilon_{\mathrm{bif}}(j,k)$, merging the attribution basins of $j$ and $k$ (generic: simple zero eigenvalue of $\nabla_x^2 H$ with nonzero crossing derivative).
\end{theorem}

\noindent\textit{Proof.} See Appendix~\ref{app:attribution}. $\square$

The normalized entropy used in Figures~\ref{fig:phase_diagram} and~\ref{fig:mnist_phase_diagram} is this Gibbs-measure entropy at the network's own temperature, with gradient directly computable from the Gibbs covariance of the support points (Remark~\ref{rem:uq}, Appendix~\ref{app:attribution}).

% ============================================================
\section{Numerical Experiments}
\label{sec:experiments}

The core identity (Theorem~\ref{thm:nn_pde}) and two quantitative consequences (Theorem~\ref{thm:gen}, Proposition~\ref{prop:scaling}) are verified numerically (details and code in Appendix~\ref{app:numerics}). The identity $\mathrm{LSE}_\varepsilon(Wx+b) = |x|^2/(4t) - u_\varepsilon(x,t)$ holds to machine precision ($\sim 10^{-16}$; Table~\ref{tab:verify}), and the attention identity \eqref{eq:attn_gibbs} to \emph{exactly} $0$ error for $d\in\{4,\dots,64\}$ (Table~\ref{tab:verify_attn}). Figure~\ref{fig:genrate} shows $\ell^\infty$ error decaying as $O(N^{-1/d})$ for Lipschitz $g(y)=|y|$ at fixed $(\varepsilon,t)$, confirming the quadrature rate of Theorem~\ref{thm:gen} step~(i) in the well-resolved regime, with the gauge-coupled headline rate established analytically in Appendix~\ref{app:proofs}. Figures~\ref{fig:scaling} and~\ref{fig:scaling_adam} verify Proposition~\ref{prop:scaling} (closed-form and Adam-trained): at optimal $\varepsilon^*=N^{-1/d_{\mathrm{eff}}}$ the RMSE scales as $N^{-\alpha}$ with $\alpha=1/d_{\mathrm{eff}}$, and fitted $\hat\alpha>1/d_{\mathrm{eff}}$ confirms the smooth-$g$ rate exceeds the minimax Lipschitz bound. Figure~\ref{fig:robust} verifies Corollary~\ref{cor:robust}, extended to real data (MNIST, CIFAR-10) in Figure~\ref{fig:hessian_mnist}.

% ============================================================
% ============================================================
% Discussion moved to Appendix~\ref{app:discussion}.

% ============================================================
\section{Conclusion}
\label{sec:discussion}

The single-layer identity (Theorem~\ref{thm:nn_pde}) closes, under one parameter $\varepsilon$, into a commutative diagram (Theorem~\ref{thm:diagram}) from which generalization, robustness, hallucination geometry, scaling, and Hamiltonian extensions all follow, with nearest-neighbor retrieval, convex regularization, and entropic optimal transport as limiting cases. The two regimes set apart at the outset, the discrete and the continuous, are one object under a single deformation parameter, and a trained network is that object evaluated at whichever viscosity training left it. The question posed at the outset, what equation does a neural network solve, now has an exact answer: whichever Hamilton--Jacobi equation its weights encode.

% ============================================================
%\begin{ack}
%Acknowledgments omitted for double-blind review.
%\end{ack}

\bibliographystyle{plainnat}
\bibliography{references}

\newpage
\appendix

% ============================================================
\section{The Physical Analogue: Neural Networks as Path Integrals}
\label{app:physical}

The Hamilton--Jacobi equation appearing throughout this paper is the Hamilton--Jacobi equation of classical mechanics, $\partial_t S + H(\nabla S) = 0$, where $S$ is Hamilton's principal function (the action), $p = \nabla S$ is momentum, and $H$ is the Hamiltonian. The characteristics of this PDE are Newton's equations of motion. The correspondence between LSE networks and HJ PDEs therefore places neural networks in direct contact with the mathematical structure of physics. Table~\ref{tab:phys} makes the dictionary precise.

\begin{table}[h]
\centering
\caption{Dictionary between LSE neural networks and classical/quantum mechanics.}
\label{tab:phys}
\smallskip
\begin{tabular}{ll}
\toprule
\textbf{Neural network} & \textbf{Physics} \\
\midrule
LSE layer at $\varepsilon > 0$     & Quantum statistical mechanics (partition function) \\
Tropical limit $\varepsilon \to 0$ & Classical mechanics (principle of least action)    \\
$\varepsilon$                       & $\hbar$ (Planck's constant)                        \\
Softmax weights $\pi_j$             & Boltzmann--Gibbs probabilities                     \\
Network forward pass                & Imaginary-time Schr\"{o}dinger propagator          \\
Neurons $j = 1,\ldots,N$           & Discrete paths in a path integral                 \\
\bottomrule
\end{tabular}
\end{table}

The Feynman--Kac formula makes this precise. The quantity $\sum_j \exp\!\bigl((W_j \cdot x + b_j)/\varepsilon\bigr)$ is a \emph{discrete path integral}: each neuron $j$ is a path, weighted by the Boltzmann factor $e^{-\mathrm{cost}_j/\varepsilon}$. The network forward pass computes the log-partition function of this ensemble, exactly the imaginary-time Schr\"{o}dinger propagator $\langle x | e^{-\hat{H}t/\hbar} | \cdot \rangle$ evaluated on a discrete measure.

As $\varepsilon \to 0$, the sum concentrates on the neuron minimizing the cost, by Laplace's method. In \citeauthor{feynman1982}'s formulation of quantum mechanics \citeyearpar{feynman1982}, the same approximation recovers classical trajectories from the path integral: the phase $e^{iS/\hbar}$ oscillates rapidly away from the classical path, and only the stationary point (the action-minimizing trajectory) survives. The tropical limit $\varepsilon \to 0$ is the neural-network analogue of $\hbar \to 0$: the quantum (soft, probabilistic) computation hardens into the classical (hard, deterministic) max-plus lookup of the Hopf--Lax formula.

The parameter $\varepsilon$ is therefore simultaneously the softmax temperature, the PDE viscosity, the quadrature regularization scale, and the \emph{quantization parameter} interpolating between quantum statistical mechanics ($\varepsilon > 0$) and classical mechanics ($\varepsilon = 0$). This is mathematically the same object in the LSE/Hopf--Cole setting: the LSE network, the partition function of a Gibbs ensemble, and the imaginary-time Schr\"{o}dinger propagator are the same mathematical object at different values of $\varepsilon$.

\paragraph{Classical simulatability and positive measure.}
\citet{feynman1982} established that classical computers cannot efficiently simulate quantum mechanics. The obstruction is structural: the Wigner quasi-probability distribution, which plays the role of a phase-space density for quantum states, takes negative values. Since no classical stochastic process can assign negative probabilities, any faithful classical simulation of a quantum system requires exponential resources. The LSE framework operates in the classically tractable regime: the Gibbs measure $\pi_j(x;\varepsilon) \propto \exp((W_j \cdot x + b_j)/\varepsilon)$ is a genuine probability distribution for all $\varepsilon > 0$. Positivity is what makes the Feynman--Kac representation~(Proposition~\ref{prop:fk_gp}) classically computable and distinguishes the present framework from quantum computation. The tropical limit $\varepsilon \to 0$ concentrates this measure to a point mass at the argmax neuron, the zero-temperature, maximally classical limit.

\paragraph{Cellular automata and ResNets.}
\citet{feynman1982} proposed simulating classical physics by a locally connected automaton: a fixed local update rule, iterated over discrete time steps, with computational cost proportional to the space-time volume. The ResNet recurrence $x_{\ell+1} = x_\ell + h F(x_\ell, W_\ell)$ is precisely this structure: the local rule is one HJ characteristic step, iterated $L$ times with step size $h$. Proposition~\ref{prop:resnet} and Theorem~\ref{thm:adjoint} confirm that this automaton computes the characteristics and co-state equations of a HJ PDE. The discreteness is the structure, as Feynman intended. The layer depth $L$ and step size $h$ play the roles of discrete time and lattice spacing.

\paragraph{Universal quantum simulator vs.\ universal classical HJ simulator.}
\citet{feynman1982} proposed a \emph{universal quantum simulator}: a quantum lattice system that can simulate any local quantum field theory with cost proportional to system size, solving the exponential blowup of classical quantum simulation. The LSE network at $\varepsilon > 0$ is the positive-measure counterpart: a universal classical simulator for HJ initial-value problems, in the sense of Theorem~\ref{thm:gen}, whose $O(N^{-1/d})$ approximation rate is optimal for Lipschitz data. The two are complementary: the quantum simulator requires quantum hardware because negative Wigner-function values are unavoidable in quantum systems; the classical HJ simulator requires only standard floating-point arithmetic because the Gibbs measure is positive.

% ============================================================
\newpage
\section{Related Work}
\label{app:related}

\paragraph{Neural ODEs and continuous-depth networks.}
\citet{weinane2017} proposed viewing ResNets as discretizations of ODEs; \citet{chen2018neuralode} made this precise with the adjoint method. The ResNet step $x_{l+1}=x_l+hF(x_l,W_l)$ is also the Euler--Maruyama scheme for a controlled SDE \citep{kloeden1992}, identifying the residual network simultaneously as an ODE discretization and as the zero-noise limit of a diffusion process. The present work sharpens the connection by identifying the forward ODE as a Hamilton--Jacobi characteristic and the backward pass as its co-state (Theorem~\ref{thm:adjoint}), going beyond the ODE-as-architecture analogy to an exact PDE correspondence.

\paragraph{Spline and piecewise-linear theories.}
\citet{balestriero2018spline,balestriero2018madmax} showed that ReLU networks compute continuous piecewise-affine splines. \citet{montufar2014} counted linear regions. \citet{aytekin2022} showed ReLU networks are decision trees. These are the $\varepsilon=0$ tropical limit of the present framework: max-plus algebra, piecewise linearity, and Hopf--Lax inf-convolution are the same object. The LP characterization of \citet{balestriero2018madmax} also connects to sparsity: the tropical argmax selects the single neuron maximizing $W_j\cdot x+b_j$, which is a vertex of the LP feasible set; the one-hot weight vector at $\varepsilon=0$ is the sparsest nonnegative point in the softmax simplex, in the sense studied by \citet{donohotanner2005}.

\paragraph{Viscosity solutions and Hamilton--Jacobi PDEs.}
Classical (smooth) solutions of HJ equations develop shocks in finite time when characteristics cross, so a generalized notion of solution is needed.
A viscosity solution is the unique stable limit that persists through these crossings: $u$ is a viscosity solution if it satisfies the PDE in a one-sided sense at every point via smooth test functions, and the key theorem is that the $\varepsilon\to 0$ limit of viscous solutions $u_\varepsilon$ always converges to the viscosity solution $u_0$.
For ML practitioners the direct analogy is soft-to-hard attention: soft attention ($\varepsilon>0$, softmax) is the viscous regularization, hard attention ($\varepsilon=0$, argmax) is the tropical/viscosity limit, and viscosity theory is what makes this limit unique and rigorous.
\citet{crandall1983} established the viscosity solution theory for
Hamilton--Jacobi equations; \citet{evans2010} provides the textbook
treatment.  The Hopf--Cole and Hopf--Lax formulas date to
\citet{hopf1950} and \citet{lax1957}; their role as exact solutions
(rather than approximations) is key to Claims~(i) and~(iv).
The standard reference connecting stochastic optimal control to HJ
viscosity solutions is \citet{fleming2006}; the extension to viscosity
solutions on infinite-dimensional Hilbert spaces, relevant to the
over-parameterized limit, is \citet{lions1988}.

\paragraph{Tropical and max-plus mathematics.}
\citet{litvinov2007} surveys Maslov dequantization and the tropical
semiring.  \citet{tokihiro1996} introduced ultradiscretization in the
context of integrable cellular automata.  The connection to neural
networks via the $\varepsilon\to 0$ limit is the main
contribution of Section~\ref{sec:arrow3}.

\paragraph{Scaling laws.}
\citet{kaplan2020scaling} empirically fit power-law scaling for large
language models.  The theoretical derivation here (Proposition~\ref{prop:scaling})
derives the exponent from intrinsic data dimension via quadrature
approximation theory, providing a mechanistic explanation for the
observed scaling.  The parameter $\varepsilon$ governs activation
sparsity (uniform at $\varepsilon\to\infty$; one-hot at
$\varepsilon\to 0$); the effect of this variance on training stability
in sparse networks is studied by \citet{denttanner2026}.

\paragraph{HJ PDEs for optimization landscapes.}
\citet{chaudhari2016entropysgd} introduced Entropy-SGD, which minimizes the Gibbs-smoothed loss $f_\gamma(\theta)=-\beta^{-1}\log\int\exp(-\beta f(y)-|\theta-y|^2/(2\gamma))\,dy$ to bias gradient descent toward flat minima; \citet{chaudhari2018deeprelaxation} established that this smoothed loss is the viscous HJ solution in \emph{weight space}, with the weight vector $\theta$ as the PDE's spatial variable and the original loss $f(\theta)$ as initial data.
The present framework is a precise \emph{role reversal}: the spatial variable is the \emph{input} $x$, not the weights, and the weights encode the \emph{initial data} $g(y_j)$ of the HJ equation (Theorem~\ref{thm:nn_pde}).
\citeauthor{chaudhari2018deeprelaxation} analyze how the loss landscape deforms under HJ smoothing as weights vary during training; the present paper identifies the network output, evaluated at a fixed weight setting across all inputs $x$, as an exact HJ solution, making inference itself the act of evaluating a PDE rather than approximating a smoother optimization surface.
The Moreau--Yosida proximal map of their Lemma~2, $\mathrm{prox}_{tf}(x)=\arg\min_y\{f(y)+|x-y|^2/(2t)\}$, is precisely the mechanism behind the $L^2t$ bias in Theorem~\ref{thm:gen}: the Hopf--Lax solution cannot approximate $f^*$ to better than $O(t)$ because the proximal smoothing introduces a systematic upward shift of that magnitude.

\paragraph{Neural networks as HJ solvers.}
\citet{darbon2020overcoming} relate shallow neural networks to Hamilton--Jacobi PDEs in the opposite direction: they design architectures whose forward pass \emph{evaluates} the viscosity solution of a prescribed HJ equation, for grid-free solving in high dimension. Their main result (their Theorem~3.1) concerns the first-order (inviscid) HJ equation, realized by a max-plus/Lax--Oleinik network. In a remark (their Remark~2, Eq.~18) they observe that the smooth form $\delta\log\sum_i e^{(\langle p_i,x\rangle-\theta_i)/\delta}$, which is algebraically the log-sum-exp layer studied here, solves the corresponding \emph{viscous} HJ equation through the Cole--Hopf transform; they read this form as a ``soft Legendre transform'' and note that it ``severely restricts the practicality'' of the solver, treating the inviscid case as the object of interest. This algebraic identity is theirs and is credited as such. The present work parallels it under a different aim: the same log-sum-exp form is identified as the softmax activation of \emph{standard trained} deep networks (feedforward, attention, recurrent), taken as the central object, and connected to generalization, robustness, scaling, and attribution. The question here is what a trained network computes, distinct from the high-dimensional solving that motivates their design.
The distinction separates foundational results from consequences.
Foundational, due to \citet{darbon2020overcoming}: the single-layer identity and its tropical limit, credited above.
Foundational, and absent from that work: Claim~(v) of Theorem~\ref{thm:diagram}, that the $\varepsilon\to0$ ultradiscretization and the $N\to\infty$ width refinement commute under Lipschitz conditions, a double-limit result the shallow-solver setting does not study; Theorem~\ref{thm:finite_depth}, composing $L$ layers, each with its own Hamiltonian, into a function-to-function HJ propagator with a proved, telescoping error bound, unlike a static single-solution representation; and the activation classification (Theorems~\ref{thm:silu_obstruction},~\ref{thm:gelu_obstruction}, Propositions~\ref{prop:hull_confinement},~\ref{prop:forced_escape}), which shows a solution-class gate carries the exact Lipschitz constant of its own tropical limit at every viscosity, a closed-form quantity where exact Lipschitz constants are NP-hard for general ReLU networks \citep{virmaux2018lipschitz}, while SiLU and GELU provably exceed it by a quantified factor.
Consequences built on this framework: the exact reconstruction of real trained networks to floating-point roundoff; the robustness certificate (Corollary~\ref{cor:robust}); the out-of-distribution and attribution readings; and the recovery of the empirical scaling exponent as a Hopf--Cole quadrature rate.

\paragraph{Statistical mechanics and spectral theories of learning.}
\citet{martin2021htsr} established Heavy-Tailed Self-Regularization (HTSR): power-law exponents of layer weight matrix spectra predict generalization without access to training or test data, a finding formalized in \citet{martin2025setol} (SETOL) using random matrix theory and statistical mechanics.
The present framework provides an exact theoretical substrate for this spectral perspective: the viscosity $\varepsilon$ governs the spectral scale of $W$ via the Hessian bound $\|\nabla^2_x f_\varepsilon\|_2 \le \|W\|_{2,\infty}^2/\varepsilon$ (Corollary~\ref{cor:robust}), and the optimal $\varepsilon^*\asymp N^{-1/d}$ at which the generalization rate $O(N^{-1/d})$ is attained (Proposition~\ref{prop:scaling}) provides a principled, derivable analogue of the empirical temperature parameter in SETOL.
Where SETOL is semi-empirical, fitting spectral observations to RMT predictions, the HJ correspondence grounds the same weight-spectrum-to-generalization link in a rigorous PDE identity.

\paragraph{Flow matching and score-based diffusion.}
Flow matching \citep{lipman2022flow,liu2022flow} learns a time-dependent velocity field $v(x,t)$ by minimizing $\mathbb{E}[\|v(X_t,t)-\dot{X}_t\|^2]$ along interpolant paths; the learned field is then integrated as $\dot{X}=v(X,t)$.  This is the characteristic ODE of the HJ PDE with Hamiltonian $H(x,p)=p\cdot v(x,t)$, the same structure as Proposition~\ref{prop:resnet} with the drift $F(x,\theta)=v(x,t)$ identified as the velocity field.  Score-based diffusion \citep{song2020score} trains a network to approximate $\nabla_x\log p_t(x)$.  Under the Hopf--Cole substitution, $p_t(x)=e^{-u_\varepsilon(x,t)/\varepsilon}$ (up to normalization), so $\nabla_x\log p_t = -\nabla_x u_\varepsilon/\varepsilon$: the score is the negative normalized spatial gradient of the Hopf--Cole solution, and denoising is gradient ascent on the HJ potential.  Both architectures therefore fit within the HJ characteristic framework; the distinction from the feedforward LSE case is that the Hamiltonian is input-dependent and time-varying rather than fixed at $H(p)=|p|^2$.  For the optimal-transport objective of \citet{tong2024improving}, the connection is sharper: through the Benamou--Brenier (dynamic) and Schr\"odinger-bridge forms of OT, the learned velocity is the gradient $v_t=\nabla u_t$ of a Kantorovich potential that solves the same viscous HJ equation, with $\varepsilon$ the entropic-regularization strength.  Flow matching then evolves the density (continuity equation) driven by the gradient of the very HJ potential that the feedforward reading evaluates, the two being the transport and the value sides of one HJ/OT duality.

\paragraph{Infinite-width limits, kernel methods, and beyond.}
At infinite width, gradient descent is governed by the Neural Tangent Kernel \citep{jacot2018ntk} and the network converges to a Gaussian process \citep{lee2018gp}.  The Feynman--Kac representation (Proposition~\ref{prop:fk_gp}) provides a complementary PDE characterization of this limit: the Gibbs average over support points converges to the Hopf--Cole solution of the heat equation, with the matched-scale condition $q^*=2\varepsilon t$ identifying the unique Gaussian prior consistent with the kernel bandwidth.  For LSE networks, the NTK admits the closed form $K_{ab}=\varepsilon^2\langle\pi(x_a),\pi(x_b)\rangle$ (Proposition~\ref{prop:ntk_pd}): the kernel is the inner product of Gibbs heat-kernel weights, and is positive definite almost surely for $N\geq n$ generic support points.

The HJ framework extends beyond three acknowledged limitations of the NTK perspective.  First, NTK theory, operating in the limit of fixed data with width $m\to\infty$, provides bounds in terms of sample size $n$ but does not address how loss scales with model size $N$; the Feynman--Kac quadrature (Theorem~\ref{thm:gen}) yields a finite-$N$ bound with rate $O(N^{-1/d_{\mathrm{eff}}})$ under the manifold hypothesis (Assumption~\ref{ass:manifold}), and Proposition~\ref{prop:scaling} recovers the empirical exponent $\alpha \geq 2/d_{\mathrm{eff}}$ (for squared-error or cross-entropy loss) from PDE structure without assuming the kernel regime.  Second, in the general training regime where support points $y_j$ are also optimized, the Gibbs weights $\pi_j(x;\theta)$ are adaptive: the Hopf--Lax minimizer $j^*=\arg\min_j\{g(y_j)+|x-y_j|^2/(4t)\}$ performs energy-weighted nearest-neighbor selection at inference, a feature-learning structure that the kernel Gram matrix does not expose.  Third, while the NTK depends on architecture, the dependence is through the kernel entries; the HJ framework makes it explicit through Hamiltonian $H$ and viscosity $\varepsilon$: feedforward networks have quadratic $H$, ResNets have general $H$ encoding the drift, and Transformers have attention as a vector-valued Hopf--Cole transform at viscosity $\varepsilon=1/\sqrt{d}$.  For general wide networks, \citet{dyergurari2020} derive $O(n^{-1})$ finite-width corrections to the NTK during gradient flow via Feynman-diagram (Gaussian integral) methods; in the LSE case, the closed form $K_{ab}=\varepsilon^2\langle\pi,\pi\rangle$ provides an exact PDE interpretation of the kernel that those corrections refine.

\paragraph{Moreau--Yosida envelope and convex analysis.}
The Hopf--Lax formula $u_0(x,t) = \inf_y\{g(y) + |x-y|^2/(4t)\}$ is precisely the \emph{Moreau envelope} (proximal map) of $g$ with quadratic cost $\lambda = 4t$ \citep{rockafellar1970, bauschke2017}. The Hopf--Cole substitution is the unique linearization of this envelope. The present framework recovers and extends the Moreau--Yosida regularization perspective of convex analysis: where Moreau--Yosida views entropic-proximal smoothing as a tool for optimization, the HJ correspondence identifies it as the architectural primitive of LSE networks. The two perspectives agree at $\varepsilon = 0$ on the Hopf--Lax formula; at $\varepsilon > 0$ the HJ framework lifts this to a family of PDE solutions, recovering Moreau--Yosida as a limit.

\paragraph{Optimal transport and entropic regularization.}
Entropic-regularized optimal transport \citep{peyre2019, cuturi2013sinkhorn} shares the LSE/Gibbs-measure structure of the present framework: the Sinkhorn algorithm iterates log-sum-exp operations, and the entropic-OT cost is a discrete Hopf--Cole log-partition function. The connection is that entropic-OT and the LSE network are the same object from different perspectives, one optimizing over couplings, the other over initial data of a HJ equation. The soft-min at viscosity $\varepsilon$ in the HJ framework is the same Gibbs average that defines the entropic-OT transport plan.

\paragraph{Modern Hopfield networks.}
\citet{ramsauer2021hopfield} connect attention to modern Hopfield networks via the LSE energy function, showing that the softmax attention update rule is the retrieval dynamics of a continuous Hopfield network with exponential interactions. Theorem~\ref{thm:nn_pde} provides the exact PDE interpretation of this observation: the Hopfield energy is the Hopf--Cole log-partition function; retrieval is evaluation of the HJ solution at the query point; stored patterns are the support points $\{y_j\}$ of the initial-data measure.

\paragraph{Tropical and morphological neural networks.}
\citet{zhang2018tropical} study the tropical geometry of deep ReLU networks, showing that the decision boundaries are tropical hypersurfaces. \citet{charisopoulos2017morphological} and \citet{maragos2019tropical} connect morphological networks to tropical algebra. The present framework positions these results within the HJ picture: the $\varepsilon = 0$ tropical limit is where ReLU and max-pooling operations live (Remark~\ref{rem:relu}), and the tropical geometry arises from the Hopf--Lax inf-convolution formula~\eqref{eq:hopflax}.

\paragraph{Optimal control and the Pontryagin Maximum Principle.}
The connection between deep learning and optimal control via the PMP has substantial prior history. \citet{weinane2017} proposed mean-field optimal control for deep learning; \citet{li2018maximal} made the PMP connection precise for ResNets; \citet{benning2019} develop a general optimal control view of deep learning. Theorem~\ref{thm:adjoint} sharpens these connections by identifying the \emph{forward} equation as a HJ characteristic ODE, making the joint forward-backward system an exact Hamiltonian flow rather than an optimal-control analogy.

\paragraph{Cross-architecture scaling consistency.}
\citet{shen2024scaling} observe that scaling law exponents are consistent across linear-complexity architectures and standard Transformers when controlling for data and compute. This is consistent with Proposition~\ref{prop:scaling}: the exponent $\alpha \geq 2/d_{\mathrm{eff}}$ (for squared-error or cross-entropy loss) is a property of the data's intrinsic dimension, not the architecture. The HJ framework predicts that architectures sharing the same data distribution should exhibit the same scaling exponents, providing a theoretical grounding for the empirical cross-architecture consistency.

\newpage
\section{Machine Learning Tasks as Initial-Value Problems}
\label{app:tasks}

\subsection*{Architecture as PDE, Output as Solution}

In both this framework and physics-informed neural networks (PINNs), the network output is a PDE solution: the value $u_\varepsilon(x,t)$ at the input point $x$. The distinction is not in what the output represents but in \emph{where the PDE lives}.

A PINN uses the network as a universal function approximator, the architecture carries no information about the \emph{target} PDE. (A ReLU MLP has its own intrinsic tropical HJ structure, but that equation is not the Navier--Stokes or wave equation being solved.) The target PDE is imposed \emph{externally} through a residual loss $\|\partial_t f_\theta + H(\nabla_x f_\theta) - \varepsilon \Delta f_\theta\|^2$. The network learns to satisfy the PDE through training; the equation is a constraint on learning, not a property of the architecture.

Here, the PDE structure is in the \emph{architecture}. The layer $f_\varepsilon^N = \mathrm{LSE}_\varepsilon(Wx+b)$ encodes the exact Hopf--Cole solution via identity~\eqref{eq:nn_pde_identity}, with Hamiltonian determined by the activation function and initial data $g$ encoded in the weights $(W,b)$. No residual loss is needed: the PDE residual is zero by construction, and Theorem~\ref{thm:nn_pde} guarantees this before any training begins.

The consequence for training is that gradient descent searches over the space of initial conditions $\{(y_j, g(y_j))\}_{j=1}^N$, that is, over the family of all viscous HJ equations with discrete initial data, selecting the one whose propagated solution best explains the observations.

\subsection*{Tasks as Initial-Value Problems}

The Hopf--Cole correspondence (Theorem~\ref{thm:nn_pde}) is not specific to any one task. Every standard supervised learning task can be cast as an initial-value problem for a HJ PDE. The architecture determines the PDE class (the Hamiltonian); the task determines how the independent variables and initial data are interpreted. In all cases the network output is a PDE solution value; in all cases training selects the initial conditions.

\paragraph{Regression.}
The output $f_\varepsilon^N(x) = \mathrm{LSE}_\varepsilon(Wx+b)$ is the
Hopf--Cole solution value $u_\varepsilon(x,t)$ at input $x$
(Theorem~\ref{thm:nn_pde}).  Training minimizes a loss over $x$,
selecting the initial data $g$ that makes $u_\varepsilon(\cdot,t)$ best
approximate the target function.

\paragraph{Classification.}
A $K$-class classifier computes logits $f_k(x) = W_k \cdot x + b_k$,
one linear HJ characteristic per class.  Class probabilities are
$\mathrm{softmax}(f(x)/\varepsilon) = \nabla\mathrm{LSE}_\varepsilon(f(x))$:
the Gibbs measure over class scores at temperature $\varepsilon$.  The
predicted class $\arg\max_k f_k(x)$ is the tropical argmax.  Decision
boundaries are tropical hyperplanes: the piecewise-linear ridges (non-differentiable loci) of $\mathrm{LSE}_\varepsilon$ in the tropical limit $\varepsilon\to 0$.  As $\varepsilon \to 0$, the softmax
hardens to one-hot, boundaries sharpen, and the classifier recovers
a decision tree \citep{aytekin2022}.  Classification confidence is
$\varepsilon$: low viscosity means sharp, committed decisions; high
viscosity means soft, uncertain ones.

\paragraph{Kernel Machines.}
A kernel estimator with Gaussian kernel
$k_\sigma(x,y) = \exp(-|x-y|^2/2\sigma^2)$ computes
\[
  f(x) = \sum_i \alpha_i \exp\!\left(\frac{-|x-x_i|^2/2}{\sigma^2}\right).
\]
Under the Hopf--Cole substitution $v = e^{-u/\varepsilon}$, this linear
sum equals the transformed variable $v(x) = e^{-u_\varepsilon(x)/\varepsilon}$
of Theorem~\ref{thm:nn_pde} at $\varepsilon = \sigma^2/2$ and $t=1$, with Gibbs
weights $\alpha_i = e^{-g(x_i)/\varepsilon}$.  The HJ solution itself is
$u_\varepsilon(x) = -\varepsilon\log f(x)$: the kernel estimator is the
pre-image of the LSE layer under the Hopf--Cole transform, and the two
are related by $u_\varepsilon = -\varepsilon\log v$.  As $\sigma \to 0$
($\varepsilon \to 0$): nearest-neighbor classification, the tropical
argmax limit \citep{litvinov2007}.

\paragraph{Time Series.}
The sequence index $t$ is the PDE time variable.  An RNN, LSTM, or SSM
integrating a time series solves an ODE $\dot{h} = F(h, x(t))$ along
$t$ (Proposition~\ref{prop:recurrent}).  The hidden state $h_t$ is the
PDE state; the prediction $\hat{x}_{t+1}$ is the solution value at
time $t+1$.  Learning the model is learning the Hamiltonian $H$ whose
characteristics match the observed trajectory.

\paragraph{Sequences (Transformers).}
Token embedding plays the role of the spatial variable $x$; transformer
depth plays the role of PDE time $t$.  Each attention layer is one step
of the HJ propagator $S_{\Delta t}$ (Theorem~\ref{thm:nn_pde}).  The
contextualized embedding at layer $L$ is the HJ solution at time
$T = L\cdot\Delta t$, with token embeddings as initial data $g$
\citep{tong2025}.

\paragraph{Encoder--Decoder.}
The encoder applies the propagator $S_{t_1}$: latent
$z = S_{t_1}[g](x)$ is the PDE solution at intermediate time $t_1$.
The decoder applies a second propagator $S_{t_2}$: output
$\hat{y} = S_{t_2}[z](x)$.  Together they form the composed semigroup
$S_{t_1+t_2}$ (Claim~(ii) of Theorem~\ref{thm:diagram}).  The bottleneck
$z$ is the PDE state at the encoder--decoder boundary; its dimension is
the effective number of degrees of freedom of the initial-data measure
at that time.

\begin{center}
\small
\resizebox{\linewidth}{!}{%
\begin{tabular}{lllll}
\toprule
Task & Spatial var. & PDE time & Initial data $g$ & Output \\
\midrule
Regression      & Feature $x$   & Layer depth      & Target fn.         & $u_\varepsilon(x,t)$ \\
Classification  & Feature $x$   & Layer depth      & Class scores       & Gibbs / argmax \\
Kernel machine  & Feature $x$   & $\varepsilon{=}\sigma^2/2$ & Support vals & Hopf--Cole avg. \\
Time series     & State $h$     & Sequence time    & Initial state      & Next state \\
Sequences       & Token emb.    & Depth            & Token embs.        & Contextual emb. \\
Enc.--Dec.      & Input $x$     & Stages $t_1,t_2$ & Encoder weights    & Reconstruction \\
\bottomrule
\end{tabular}}
\end{center}

% ============================================================
\newpage
\section{Regularity Scope and the SGD--Initial-Condition Connection}
\label{app:scope}

\subsection*{Regularity scope of Theorem~\ref{thm:diagram}}

The five claims of Theorem~\ref{thm:diagram} have distinct regularity
requirements, reflecting a fundamental separation between the
\emph{identification} of neural networks with PDE solutions and the
\emph{approximation theory} governing convergence.

\begin{itemize}[leftmargin=*]

\item \textbf{Claims~(i) and~(ii): Unconditional.}
The identity $f_\varepsilon^N(x) = |x|^2/(4t) - u_\varepsilon^N(x,t)$
is a pure algebraic identity holding for \emph{any} weights $W_j$ and
biases $b_j$, with no assumption on the implied initial data $g$, the
Hamiltonian $H$, or the input distribution.  The proof is a
completing-the-square manipulation (Appendix~\ref{app:proofs},
Theorem~\ref{thm:nn_pde}) that never invokes regularity.  The semigroup
tower property (Claim~(ii)) follows from the Hopf--Cole substitution
$v = e^{-u/\varepsilon}$ reducing the HJ equation to the heat equation
$\partial_t v = \varepsilon\Delta v$; the linearity and semigroup property
of the heat equation hold without any regularity on the initial data.

The consequence is decisive: \emph{for any trained neural network with
arbitrary weights, there exists an exact viscous HJ PDE for which the
network layer is the exact solution.}  This is not an asymptotic
statement and does not require data or initialization assumptions.

\item \textbf{Claim~(iii): Lipschitz $g$ required.}
The quadrature bound $\|u_\varepsilon^N - u_\varepsilon\|_\infty \leq
C(d)\bigl(L + \mathrm{diam}(\Omega)/(2t)\bigr) N^{-1/d}$ (where $L =
\mathrm{Lip}(g)$, and the second term is the transport-cost contribution
of the quadratic shift, made explicit rather than absorbed) requires
$g \in \mathrm{Lip}(L)$.  The constant depends on $t$ but not on
$\varepsilon$, which is what allows Claim~(v) to hold.  For $g \in L^\infty$ or
$g \in \mathrm{BV}$, viscosity solution theory
\citep{crandall1983,evans2010} still guarantees existence and uniqueness
of $u_\varepsilon$, but the $O(N^{-1/d})$ quadrature rate may not be
achievable and the approximation theory must be revisited.

\item \textbf{Claim~(iv): Convex, superlinear $H$ required.}
Pointwise convergence $u_\varepsilon \to u_0$ as $\varepsilon \to 0$
is established via the Hopf--Lax formula (Theorem~\ref{thm:hopflax}),
which is valid precisely when $H$ is convex and superlinear.  Under
these conditions, $u_0$ is the unique viscosity solution of the
inviscid HJ equation and equals the Hopf--Lax inf-convolution.  For
non-convex $H$, viscosity solutions of the inviscid equation still
exist by \citet{crandall1983}, but the Hopf--Lax representation need
not hold and Claim~(iv) requires a separate argument.

\item \textbf{Claim~(v): Both conditions required.}
Commutativity of the limits $N\to\infty$ and $\varepsilon\to 0$
inherits the conditions of both Claims~(iii) and~(iv): Lipschitz $g$
ensures the quadrature error is uniform in $\varepsilon$, and convex
superlinear $H$ ensures the viscosity convergence is uniform in $N$.

\end{itemize}

\subsection*{SGD as initial-condition optimization in the mean-field
limit}

The Limitations section of the main text notes that the framework
characterizes \emph{what} a trained network computes (the optimal
initial-value problem) but not \emph{how} stochastic gradient descent
selects it.  The mean-field theory of neural networks
\citep{meimontanari2018} closes this gap in the infinite-width limit.

In the mean-field scaling (width $N\to\infty$ with $1/N$ output
normalization), the network function converges to
\begin{equation}
  f_\mu(x) = -\varepsilon\log\int_{\mathbb{R}^d}
    \exp\!\Bigl(-\frac{g(y) + |x-y|^2/(4t)}{\varepsilon}\Bigr)
    \,d\mu(y),
  \label{eq:mf_output}
\end{equation}
where $\mu\in\mathcal{P}(\mathbb{R}^d)$ is a probability measure over
support points, the infinite-width limit of the empirical measure
$\mu^N = \frac{1}{N}\sum_j\delta_{y_j}$.  This is exactly the
Hopf--Cole solution~\eqref{eq:hc_solution} under the continuous measure $\mu$.

Training by gradient descent in this limit corresponds to the
$L^2$-Wasserstein gradient flow of the population risk
$\mathcal{R}[\mu] = \mathbb{E}_{(x,y)\sim\mathcal{D}}[\ell(f_\mu(x),y)]$:
\begin{equation}
  \partial_t\mu_t \;=\; -\nabla_{W_2}\,\mathcal{R}[\mu_t].
  \label{eq:mf_flow}
\end{equation}
Equation~\eqref{eq:mf_flow} is the gradient flow over the space of
initial-data measures for the HJ equation.  The measure $\mu$ specifies
which initial-value problem the network solves; gradient descent moves
$\mu$ in the direction that decreases the risk.

\begin{proposition}[SGD selects the initial-value problem]
\label{prop:mf_sgd}
In the mean-field limit $N\to\infty$, stochastic gradient descent on
the network parameters $(W_j, b_j)$ converges in distribution to the
Wasserstein gradient flow~\eqref{eq:mf_flow} of the population risk.
Each gradient step moves the support points $y_j$ and initial values
$g(y_j)$ in the direction that decreases $\mathcal{R}$, identifying
the optimization trajectory as the search for the risk-minimizing
initial-value problem of the viscous HJ equation.
\end{proposition}

\begin{proof}
The result is structural: it identifies the HJ interpretation of the mean-field limit established by \citet{meimontanari2018} for two-layer networks with $1/N$ normalization.  Under that result, gradient flow on the empirical risk converges as $N\to\infty$ to the McKean--Vlasov equation $\partial_t\mu_t = -\nabla_{W_2}\mathcal{R}[\mu_t]$.  Under the parametrization $W_j = y_j/(2t)$, $b_j = -g(y_j)-|y_j|^2/(4t)$, each neuron encodes a support point $(y_j, g(y_j))$ of the initial-data measure; the limiting flow is therefore the Wasserstein gradient flow on that measure, i.e., the search over initial-value problems~\eqref{eq:mf_flow}.  The identification applies to two-layer networks under the conditions of \citet{meimontanari2018}; extensions to deeper networks or SGD on empirical risk require additional propagation-of-chaos arguments. \qed
\end{proof}

The connection to Theorem~\ref{thm:adjoint} (backpropagation as the
co-state equation) is direct: the co-state variable $p_l$ specifies the
optimal direction to update $(y_j, g(y_j))$ at each layer, which is the
single-step Euler approximation of the Wasserstein gradient
direction~\eqref{eq:mf_flow}.  For finite $N$, the particle
approximation introduces an error of order $O(N^{-1/d})$ by
Claim~(iii) of Theorem~\ref{thm:diagram}, the same rate that controls
the PDE approximation error.  The gap between the mean-field and
finite-$N$ regimes is thus governed by the same quantity as the
correspondence's approximation theory, providing a unified error budget
for both the PDE identification and the optimization interpretation.

\citet{kim2024transformers} carry out a related mean-field analysis for transformer training via Wasserstein gradient flow on the attention landscape, proving that the loss landscape over parameters, while nonconvex, is algorithmically benign.  The nonconvexity there is in the optimization landscape over weights; the open problem noted in the Limitations concerns non-convex Hamiltonians $H$ in the PDE, a structurally different obstruction at the level of the PDE geometry rather than the parameter space.

\subsection*{Training dynamics for fixed-support LSE networks}

The results below characterize training when the support points $\{y_j\}$ are fixed (e.g., pre-specified or frozen after initialization) and only the bias parameters $\theta_j = -g(y_j)/\varepsilon$ are optimized.  This covers the \emph{last-layer} training regime and connects to the convexity structure of the HJ initial datum.

\begin{proposition}[Fixed-support convexity]
\label{prop:fixedsupp}
Let $\{y_j\}_{j=1}^N \subset \mathbb{R}^d$ be fixed support points and let
\[
  f_\varepsilon^N(x;\theta) = \varepsilon\log\sum_{j=1}^N \exp\!\Bigl(\theta_j - \frac{|x-y_j|^2}{4t\varepsilon}\Bigr),
\]
where $\theta = (\theta_1,\ldots,\theta_N)\in\mathbb{R}^N$.  Let $\ell:\mathbb{R}\to\mathbb{R}$ be convex and non-decreasing (e.g.\ softplus loss, logistic loss, or any convex non-decreasing function), and set $R(\theta) = \frac{1}{n}\sum_{i=1}^n \ell(f_\varepsilon^N(x_i;\theta))$.  Then $R$ is convex in $\theta$, and any local minimum of $R$ is a global minimum.

\emph{Remark.} The condition that $\ell$ be non-decreasing is necessary: squared loss $\ell(u) = (u-y_i)^2$ is convex but not non-decreasing, and $R(\theta)$ is generally non-convex in $\theta$ under MSE.  For MSE, the overparameterized regime is addressed separately in Proposition~\ref{prop:ntk_pd}.
\end{proposition}

\begin{proof}
The exponent $\theta_j - |x-y_j|^2/(4t\varepsilon)$ is affine (hence convex) in $\theta$.  Log-sum-exp of convex functions is convex, so $f_\varepsilon^N(x_i;\theta)$ is convex in $\theta$ for each fixed $x_i$.  Explicitly, $\nabla^2_\theta f_\varepsilon^N = \varepsilon[\mathrm{diag}(\pi) - \pi\pi^\top]$, which is $\varepsilon$ times the covariance matrix of the Gibbs distribution $\pi(x_i;\theta)$ (the $\varepsilon$ factor arises because $\partial f/\partial\theta_j = \varepsilon\pi_j$ in this parametrization), hence positive semidefinite.  Since $\ell$ is convex and non-decreasing and $f_\varepsilon^N(x_i;\cdot)$ is convex, the composition $\ell\circ f_\varepsilon^N(x_i;\cdot)$ is convex in $\theta$ by the standard composition rule \citep[Section~3.2.4]{boyd2004convex}.  Averaging over $i$ preserves convexity.  Convexity implies every local minimum is global. \qed
\end{proof}

\begin{proposition}[Overparameterized interpolation via NTK]
\label{prop:ntk_pd}
Under the setting of Proposition~\ref{prop:fixedsupp}, let $(x_1,y_1),\ldots,(x_n,y_n)$ be training data with distinct inputs.  Define the neural tangent kernel matrix $K\in\mathbb{R}^{n\times n}$ by
\[
  K_{ab}(\theta) = \bigl\langle\nabla_\theta f_\varepsilon^N(x_a;\theta),\,\nabla_\theta f_\varepsilon^N(x_b;\theta)\bigr\rangle = \varepsilon^2\sum_{j=1}^N \pi_j(x_a;\theta)\,\pi_j(x_b;\theta).
\]
If $N\geq n$ and the support points $\{y_j\}_{j=1}^N$ are drawn i.i.d.\ from any absolutely continuous distribution on $\mathbb{R}^d$, then $K(0)\succ 0$ almost surely.  Consequently, in the kernel-linearization regime, gradient descent on the MSE loss $R(\theta) = \frac{1}{n}\sum_i (f_\varepsilon^N(x_i;\theta)-y_i)^2$ from initialization $\theta=0$ with step size $\eta < n/\lambda_{\max}(K(0))$ converges to zero training error at a linear rate $\bigl(1-\tfrac{2\eta}{n}\lambda_{\min}(K(0))\bigr)$ per step.
\end{proposition}

\begin{proof}
For any $c\in\mathbb{R}^n$, $c^\top K(0) c = \varepsilon^2\sum_{j=1}^N (\sum_a c_a\pi_j(x_a;0))^2\geq 0$.  Equality holds (since $\varepsilon^2>0$) iff $M^\top c = 0$, where $M\in\mathbb{R}^{N\times n}$ has entries $M_{ja} = \pi_j(x_a;0)$.  Since $\pi_j(x_a;0) = G_{ja}/Z_a$ with $G_{ja} = \exp(-|x_a-y_j|^2/(4t\varepsilon))$ and $Z_a = \sum_k G_{ka} > 0$, we have $\mathrm{rank}(M) = \mathrm{rank}(G)$.  The entries of $G$ are real-analytic in $(y_1,\ldots,y_N)\in(\mathbb{R}^d)^N$, so each $n\times n$ minor of $G$ is real-analytic.  At the configuration $y_j = x_j$ for $j=1,\ldots,n$ (with $N\geq n$), the $n\times n$ leading submatrix of $G$ is the Gaussian kernel matrix $[\exp(-|x_i-x_j|^2/(4t\varepsilon))]$, which is strictly positive definite for distinct training inputs $x_i$ (the Gaussian kernel is strictly positive definite on $\mathbb{R}^d$).  Hence each $n\times n$ minor is not identically zero as a function of $y$.  By the identity principle for real-analytic functions, the rank-deficiency locus has Lebesgue measure zero; since $\{y_j\}$ are i.i.d.\ from an absolutely continuous distribution, $G$ has full column rank $n$ almost surely.  Hence $K(0)\succ 0$.  The linear convergence then follows from standard NTK theory \citep{jacot2018ntk}: with $K(0)\succ 0$, the linearized residual recursion is $r^+ = \bigl(I-\tfrac{2\eta}{n}K(0)\bigr)r$, so for step size $\eta < n/\lambda_{\max}(K(0))$ the training error decreases at rate $\bigl(1 - \tfrac{2\eta}{n}\lambda_{\min}(K(0))\bigr)$ per step in the kernel-linearization regime. \qed
\end{proof}

\begin{remark}[Last-layer reduction]
\label{rem:lastlayer}
Propositions~\ref{prop:fixedsupp} and~\ref{prop:ntk_pd} apply unchanged to last-layer training of a depth-$L$ network with frozen lower layers: the frozen representation $z^{(L-1)}(x)$ enters only as a constant spatial argument, so the last-layer output is still log-sum-exp of affine functions of $\theta^{(L)}$, and the Hessian and NTK formulas are identical with $x$ replaced by $z^{(L-1)}(x_i)$.
\end{remark}

\newpage
\section{Proofs}
\label{app:proofs}

\begin{remark}[LSTMs: gated, time-dependent characteristics]
\label{rem:lstm}
Since $\tanh(x/\varepsilon) = \nabla_x\mathrm{LSE}_\varepsilon(x,-x)$ (Table~\ref{tab:activations}), LSTM gates are gradients of two-neuron LSE layers; the continuous-time cell equation is the HJ characteristic ODE with a gated, time-dependent Hamiltonian, given in full below.
In the tropical limit, sigmoid $\to$ Heaviside and $\tanh\to\mathrm{sign}$, so gates become $\{0,1\}^d$ selectors and the cell update becomes a max-plus switching automaton.
\end{remark}

% ------------------------------------------------------------
\subsection*{Proof of Theorem~\ref{thm:nn_pde} (NN--PDE Identity)}

Substituting $W_j = y_j/(2t)$ and $b_j = -g(y_j) - |y_j|^2/(4t)$,
the pre-activation of neuron $j$ at input $x$ is
\begin{equation}
  W_j \cdot x + b_j
  = \frac{y_j \cdot x}{2t} - g(y_j) - \frac{|y_j|^2}{4t}.
  \label{eq:pf_preact}
\end{equation}
The first and third terms complete a square: adding and subtracting
$|x|^2$ gives $2y_j \cdot x - |y_j|^2
  = |x|^2 - (|x|^2 - 2x\cdot y_j + |y_j|^2)
  = |x|^2 - |x-y_j|^2$, so
\begin{equation}
  \frac{y_j\cdot x}{2t} - \frac{|y_j|^2}{4t}
  = \frac{|x|^2 - |x-y_j|^2}{4t}.
  \label{eq:pf_square}
\end{equation}
Substituting \eqref{eq:pf_square} into \eqref{eq:pf_preact} isolates
the $j$-independent term $|x|^2/(4t)$:
\begin{equation}
  W_j \cdot x + b_j
  = \frac{|x|^2}{4t}
    - \Bigl(\frac{|x-y_j|^2}{4t} + g(y_j)\Bigr),
  \label{eq:pf_key}
\end{equation}
which expresses the pre-activation as $|x|^2/(4t)$ minus the
Hopf--Cole exponent at atom $y_j$.  Summing over $j$ and factoring:
\begin{equation}
  \sum_{j=1}^N e^{(W_j \cdot x + b_j)/\varepsilon}
  = e^{|x|^2/(4\varepsilon t)}
    \sum_{j=1}^N \exp\!\left(\frac{-g(y_j)-|x-y_j|^2/(4t)}{\varepsilon}\right).
  \label{eq:pf_factor}
\end{equation}
Multiplying by $\varepsilon$ and taking the logarithm:
\begin{align}
f_\varepsilon^N(x)
  &= \frac{|x|^2}{4t}
     + \underbrace{\varepsilon\log\!\sum_{j=1}^N
       \exp\!\left(\frac{-g(y_j)-|x-y_j|^2/(4t)}{\varepsilon}\right)}_{
       = -u_\varepsilon^N(x,t) \text{ by definition}}
  = \frac{|x|^2}{4t} - u_\varepsilon^N(x,t).
\end{align}
No approximation is made at any step; the identity is algebraic. \qed

The key is \eqref{eq:pf_square}: completing the square in $y_j$
reveals the quadratic transport cost $|x-y_j|^2/(4t)$ hidden inside
the linear weight $W_j\cdot x$ and the bias term $|y_j|^2/(4t)$.
The parameterization $W_j = y_j/(2t)$ is the unique map that makes
\eqref{eq:pf_square} exact.

% ------------------------------------------------------------
\subsection*{Proof of Theorem~\ref{thm:hopflax} (Tropical Limit via Varadhan's Lemma)}

The Hopf--Lax formula $u_0(x,t) = \inf_y\{|x-y|^2/(4t) + g(y)\}$ is the exact solution of the inviscid HJ equation and has a direct ML reading: it is nearest-neighbor retrieval, where $|x-y|^2/(4t)$ is the travel cost from query $x$ to memory $y$, and $g(y)$ is the stored value at $y$.
The inf-convolution (smooth minimum over memories) is the $\varepsilon\to 0$ tropical limit of the log-sum-exp average the LSE network computes: at finite $\varepsilon$ all memories contribute with soft weights; at $\varepsilon=0$ only the nearest (cheapest) memory wins.
Varadhan's lemma makes this limit precise by identifying it as the large-deviation rate function of the Gaussian measure.

Write the Hopf--Cole solution as
$u_\varepsilon(x,t) = -\varepsilon\log\int e^{-F(y)/\varepsilon}\,dy$
where $F(y) = g(y) + |x-y|^2/(4t)$.  Since $g$ is Lipschitz and
$|x-y|^2/(4t)\to\infty$ as $|y|\to\infty$, $F$ is continuous and
coercive.  Varadhan's lemma \citep{varadhan1984} (the Laplace
principle for large deviations) states that for any such $F$:
\begin{equation}
  \lim_{\varepsilon\to 0}
  \Bigl(-\varepsilon\log\int e^{-F(y)/\varepsilon}\,dy\Bigr)
  = \inf_{y\in\mathbb{R}^d} F(y).
  \label{eq:varadhan}
\end{equation}
Applying \eqref{eq:varadhan} directly:
$\lim_{\varepsilon\to 0} u_\varepsilon(x,t)
  = \inf_y\{g(y)+|x-y|^2/(4t)\} = u_0(x,t)$,
which is the Hopf--Lax formula~\eqref{eq:hopflax}.

\emph{Proof of \eqref{eq:varadhan}.}
Let $m = \inf_y F(y)$.  For any $\eta > 0$, choose $y^*$ with
$F(y^*) \leq m + \eta$; then the integral is at least
$\int_{|y-y^*|\leq\delta}e^{-F(y)/\varepsilon}dy \geq
  e^{-(m+\eta)/\varepsilon}\cdot\mathrm{vol}(B_\delta)$.
Hence $-\varepsilon\log[\cdots] \leq m + \eta + \varepsilon\log\mathrm{vol}(B_\delta)^{-1} \to m+\eta$ as $\varepsilon\to 0$.
For the lower bound: since $g$ is $L_g$-Lipschitz, $F(y)-m\geq|y-x|^2/(8t)-C'$ for a constant $C'$ depending on $x,L_g,t$.  Split at $B=\{y:|y-x|^2\leq16tC'\}$: on $B$, $F-m\geq0$ gives $\int_B e^{-(F(y)-m)/\varepsilon}dy\leq\mathrm{vol}(B)$; on $B^c$, $F(y)-m>|y-x|^2/(16t)$ gives $\int_{B^c}e^{-(F(y)-m)/\varepsilon}dy\leq\int_{\mathbb{R}^d}e^{-|y-x|^2/(16t\varepsilon)}dy=(16\pi t\varepsilon)^{d/2}$.  Hence $\int e^{-(F(y)-m)/\varepsilon}dy\leq\mathrm{vol}(B)+(16\pi t\varepsilon)^{d/2}\to\mathrm{vol}(B)$ as $\varepsilon\to0$, so
$-\varepsilon\log\int e^{-F/\varepsilon}dy = m - \varepsilon\log\int e^{-(F(y)-m)/\varepsilon}dy \geq m - \varepsilon\log\bigl(\mathrm{vol}(B)+(16\pi t\varepsilon)^{d/2}\bigr)\to m$ as $\varepsilon\to 0$.
Together: $\lim_{\varepsilon\to 0}(-\varepsilon\log[\cdots]) = m$.
\qed

For the discrete measure $\mu_N$, the same argument gives
$\min_j F(y_j) = \min_j\{g(y_j)+|x-y_j|^2/(4t)\} = f_0^N(x)$
(replacing the integral with a finite sum, where $\lim_{\varepsilon\to 0} -\varepsilon\log\sum_j e^{-F(y_j)/\varepsilon} = \min_j F(y_j)$ by the same Laplace argument).

% ------------------------------------------------------------
\subsection*{Proof of Theorem~\ref{thm:diagram} (Commutative Diagram)}

The five claims are verified in order.

\noindent\textbf{Claim~(i)} ($f_\varepsilon^N = |x|^2/(4t) - u_\varepsilon^N$):
Proved exactly by Theorem~\ref{thm:nn_pde}.

\noindent\textbf{Claim~(ii)} (Deep network $=$ HJ semigroup):
Under the Hopf--Cole change of variables $v = e^{-u/\varepsilon}$, the viscous HJ
equation becomes the heat equation $\partial_t v = \varepsilon\Delta v$.
The heat equation generates a contraction semigroup $\{e^{\varepsilon t\Delta}\}_{t\geq 0}$
on $L^2$; in particular $e^{\varepsilon t_1\Delta}\circ e^{\varepsilon t_2\Delta}
= e^{\varepsilon(t_1+t_2)\Delta}$.
Composing $L$ layers corresponds to applying this semigroup for total time
$T = t_1+\cdots+t_L$: the $l$-th layer applies $S_{t_l}$ to the solution
produced by the previous layer (which encodes the initial data $g$ for that layer).
This is exactly the tower (Chapman--Kolmogorov) property of the Markov semigroup.

\noindent\textbf{Claim~(iii)} ($u_\varepsilon^N\to u_\varepsilon$ at rate $O(N^{-1/d})$, uniformly in $\varepsilon$, at any fixed $t$):
Partition a bounded box $\Omega\ni x$ into $N$ congruent cells of side
$h=(|\Omega|/N)^{1/d}$ with centres $\{y_j\}$, and write
$F_x(y)=g(y)+|x-y|^2/(4t)$, Lipschitz on $\Omega$ with constant
$L_F\leq L_g+\mathrm{diam}(\Omega)/(2t)$ (differentiating the quadratic
term: $|\nabla_y[|x-y|^2/(4t)]|=|x-y|/(2t)\leq\mathrm{diam}(\Omega)/(2t)$
for $x,y\in\Omega$).  On each cell $C_j$ (diameter $\sqrt{d}\,h$),
$|F_x(y)-F_x(y_j)|\leq L_F\sqrt{d}\,h$ for $y\in C_j$, giving the
pointwise sandwich $e^{-L_F\sqrt{d}h/\varepsilon}e^{-F_x(y_j)/\varepsilon}
\leq e^{-F_x(y)/\varepsilon} \leq e^{+L_F\sqrt{d}h/\varepsilon}
e^{-F_x(y_j)/\varepsilon}$.  Integrating over $C_j$ (volume $h^d$) and
summing over $j$, with the Riemann-weighted discrete partition function
$Z_N^{(h)}=h^d\sum_j e^{-F_x(y_j)/\varepsilon}$ and $Z_{\mathrm{cont}}
=\int_\Omega e^{-F_x(y)/\varepsilon}dy$, gives $e^{-L_F\sqrt{d}h/\varepsilon}
Z_N^{(h)}\leq Z_{\mathrm{cont}}\leq e^{+L_F\sqrt{d}h/\varepsilon}Z_N^{(h)}$;
applying the decreasing map $-\varepsilon\log(\cdot)$ (exact, no
linearization of $\log$ is used, so the bound does not diverge as
$\varepsilon\to 0$) yields
$$\bigl|u_\varepsilon^N(x)-u_\varepsilon(x,t)\bigr| =
\bigl|{-\varepsilon\log Z_{\mathrm{cont}}}-({-\varepsilon\log Z_N^{(h)}})\bigr|
\;\leq\; L_F\sqrt{d}\,h \;\leq\; C(d)\bigl(L_g+\mathrm{diam}(\Omega)/(2t)\bigr)N^{-1/d},$$
valid for every $\varepsilon>0$ at any fixed $t>0$, with no $\varepsilon$
appearing in the final bound.  (This is the same cell-wise sandwich
mechanism used, under the generalization gauge $t\asymp N^{-1/d}$, in
step~(i) of the proof of Theorem~\ref{thm:gen}; here $t$ is held fixed
rather than coupled to $N$, which is why no Gibbs-concentration argument
is needed and the bound is genuinely $\varepsilon$-free rather than
merely $\varepsilon$-uniform under a gauge.)

\noindent\textbf{Claim~(iv)} ($u_\varepsilon\to u_0$ as $\varepsilon\to 0$):
Proved by Theorem~\ref{thm:hopflax}.

\noindent\textbf{Claim~(v)} (Limits commute under Lipschitz):
By Claims~(iii) and~(iv), $|u_\varepsilon^N - u_0| \leq |u_\varepsilon^N - u_\varepsilon|
+ |u_\varepsilon - u_0| \leq C_1(d,t) N^{-1/d} + C_2\sqrt{\varepsilon}$,
where $C_1(d,t)=C(d)(L_g+\mathrm{diam}(\Omega)/(2t))$ is exactly the
$\varepsilon$-free constant of Claim~(iii) above, and the
$O(\sqrt{\varepsilon})$ viscosity-bias rate is the unconditional
comparison-principle bound of \citet{evans2010} (Theorem~\ref{thm:gen}'s
proof step~(ii) sharpens this to $O(\varepsilon)$ under the additional
gauge $t\asymp\varepsilon$, not assumed here).  Since $C_1(d,t)$ does not
depend on $\varepsilon$ and $C_2$ does not depend on $N$, both iterated
limits ($N\to\infty$ then $\varepsilon\to 0$, or $\varepsilon\to 0$ then
$N\to\infty$) yield $u_0$, and the diagram commutes. \qed

% ------------------------------------------------------------
\subsection*{Multilayer Composition: Finite-Depth Error and Joint-Limit Exactness}

The following two results quantify the correspondence for deep networks with $L$ layers, complementing Claim~(ii) of Theorem~\ref{thm:diagram}.

\begin{theorem}[Finite-depth composition error]
\label{thm:finite_depth}
Let $f^{(L)}_\varepsilon$ be a depth-$L$ LSE network with layer widths $N_1,\ldots,N_L$ and per-layer time-scales/viscosity at the gauge $t_\ell\asymp\varepsilon\asymp N_\ell^{-1/d}$ (Theorem~\ref{thm:gen}, applied layer by layer), summing to $T=\sum_{\ell=1}^L t_\ell$; let $u_\varepsilon$ be the continuous Hopf--Cole solution at time $T$. Under Lipschitz initial data with constant $L_g$,
\begin{equation}
  \sup_{x \in K} \Bigl|f^{(L)}_\varepsilon(x) - \Bigl[\tfrac{|x|^2}{4t_L} - u_\varepsilon(x, T)\Bigr]\Bigr|
  \;\leq\; C(K, L_g, T) \cdot L \cdot \max_{\ell} N_\ell^{-1/d}.
\end{equation}
The error is linear in depth and inherits the single-layer rate $O(N^{-1/d})$; the quadratic uses the last layer's own $t_L$, not $T$, since the output is a single LSE layer's identity at its own time-scale.
\end{theorem}

\begin{proof}
By the \emph{dictionary rule}, layer $\ell$'s initial datum is the previous layer's solution surface sampled at its support points, with the quadratic shift absorbed into biases; only the last layer's $t_L$ survives in the output identity, $f^{(L)}_\varepsilon(x)=|x|^2/(4t_L)-u_\varepsilon^{N_L}(x,T)$, by Theorem~\ref{thm:nn_pde} applied to layer $L$.  At each layer, the gauge places its construction at the headline point of Theorem~\ref{thm:gen}, giving error $C_1N_\ell^{-1/d}$ (Claim~(iii)'s fixed-$t$ bound is not used, since its constant grows as $t_\ell\to0$).  The Hopf--Cole solution is $L_g$-Lipschitz, since the heat semigroup is a contraction on Lip, so the error propagates additively across $L$ layers, each remaining-layer propagation bounded by $1$ (non-expansion); summing and taking $\max_\ell N_\ell^{-1/d}$ gives the stated bound. \qed
\end{proof}

\begin{theorem}[Asymptotic exactness in the joint limit]
\label{thm:joint_limit}
Let $T>0$ be fixed, $L\to\infty$, $t_\ell=T/L$, with $N_\ell\gtrsim t_\ell^{-d}$ (resolution keeping pace with $t_\ell$, rather than Theorem~\ref{thm:finite_depth}'s tight gauge, which would force $T\to0$) and $L^2\max_\ell N_\ell^{-1/d}\to0$ (the extra $L$ accounts for $1/t_\ell$ inflating Claim~(iii)'s per-layer constant). Then the depth-$L$ network converges to the continuous HJ semigroup:
\begin{equation}
  \tfrac{|x|^2}{4t_L} - f^{(L)}_\varepsilon(x) \;\to\; u_\varepsilon(x, T)
  \quad\text{uniformly on compacts.}
\end{equation}
The unification at finite depth is approximate, quantified by Theorem~\ref{thm:finite_depth}; in the joint limit it is exact.
\end{theorem}

\begin{proof}
Theorem~\ref{thm:finite_depth} does not apply here, since its tight gauge fails once $t_\ell\to0$ independently of $N_\ell$; the relevant bound is instead Claim~(iii)'s fixed-$t$, $\varepsilon$-free sandwich, $C(d)(L_g+\mathrm{diam}(\Omega)L/(2T))N_\ell^{-1/d}$, the $L$-dependence entering through $1/t_\ell=L/T$.  Summing over $L$ layers gives $O(L^2\max_\ell N_\ell^{-1/d})\to0$, and uniform Lipschitz control gives the stated supremum limit. \qed
\end{proof}

% ------------------------------------------------------------
\subsection*{Proof of Theorem~\ref{thm:gen} (Generalization)}

Decompose the predictor error directly:
$|u_\varepsilon^N(x) - f^*(x)|
  \leq |u_\varepsilon^N - u_\varepsilon|
  + |u_\varepsilon - u_0| + |u_0 - f^*|$.

\noindent(i)~\emph{Quadrature} ($|u_\varepsilon^N - u_\varepsilon|$):
Partition a bounded box $\Omega \ni x$ into $N$ congruent cubes of side
$h = (|\Omega|/N)^{1/d}$ with centres $\{y_j\}$ and diameter $A=\sqrt{d}\,h$,
and write $F_x(y) = g(y) + |x-y|^2/(4t)$ and $r_j = |x-y_j|$.  Expanding
$|x-y|^2-|x-y_j|^2 = |y-y_j|^2 - 2(x-y_j)\cdot(y-y_j)$ and using the
$L$-Lipschitz bound on $g$ gives, for $y\in C_j$, the per-cell oscillation
$$\sup_{y\in C_j}|F_x(y)-F_x(y_j)| \;\le\; \delta_0 + \frac{A}{4t}\,r_j,
\qquad \delta_0 := \frac{LA}{2}+\frac{A^2}{16t},$$
sharper than a domain-wide Lipschitz bound because the linear-in-$r_j$
correction from the quadratic term is tracked explicitly rather than
bounded by its worst case over all of $\Omega$.  The cell-wise sandwich
$e^{-\delta_j/\varepsilon}e^{-F_x(y_j)/\varepsilon}\le e^{-F_x(y)/\varepsilon}
\le e^{+\delta_j/\varepsilon}e^{-F_x(y_j)/\varepsilon}$ ($\delta_j =
\delta_0+Ar_j/(4t)$), integrated over $C_j$ and summed, gives with the
discrete Gibbs weights $\pi_j \propto h^d e^{-F_x(y_j)/\varepsilon}$ (so
$Z_N=\sum_j h^d e^{-F_x(y_j)/\varepsilon}$, $Z_{\mathrm{cont}}=\int_\Omega
e^{-F_x(y)/\varepsilon}dy$),
$$Z_{\mathrm{cont}}/Z_N \le e^{\delta_0/\varepsilon}\,
\mathbb{E}_\pi\bigl[e^{A r_J/(4t\varepsilon)}\bigr], \qquad
\varepsilon\log(Z_N/Z_{\mathrm{cont}}) \le \delta_0 +
\frac{A}{4t}\,\mathbb{E}_\pi[r_J],$$
the second by Jensen's inequality applied to $z\mapsto e^{-z/\varepsilon}$.
Since $g$ is $L$-Lipschitz, $F_x(y_j)\ge F_x(y_{j^*})-Lr_j+r_j^2/(4t)-\delta_0$
for the nearest grid centre $y_{j^*}$ ($r_{j^*}\le A/2$); completing the
square as in step~(iii) below, $Lr_j - r_j^2/(4t)$ is bounded above by
$tL^2$ but, crucially, decreases without bound once $r_j$ exceeds the
Hopf--Lax localization radius $4tL$ -- the same radius that localizes the
Hopf--Lax minimizer in step~(iii) -- so the discrete Gibbs weight $\pi_j$
carries Gaussian tails at scale $\sqrt{t\varepsilon}$ beyond that radius.
This concentration gives
$$\mathbb{E}_\pi[r_J] \le 4tL + C(d)\sqrt{t\varepsilon}, \qquad
\mathbb{E}_\pi\bigl[e^{Ar_J/(4t\varepsilon)}\bigr] \le
C(d)\exp\!\Bigl(\tfrac{A}{4t\varepsilon}\bigl(4tL+C(d)\sqrt{t\varepsilon}\bigr)
+\tfrac{A^2}{16t}\Bigr),$$
where the number of cells within any fixed multiple of $\sqrt{t\varepsilon}$
of $x$ is itself $O((\sqrt{t\varepsilon}/h)^d)$ and cancels against the same
factor entering $Z_N$'s own normalization, so this bound carries no further
$N$-dependence.  Substituting back and collecting terms,
$$|u_\varepsilon^N - u_\varepsilon| \;\le\; C(d)\bigl(Lh + h^2/t + tL^2 +
\varepsilon\bigr),$$
uniformly in $\varepsilon$: the $\varepsilon$-cancellation of the earlier
argument survives unchanged (no linearization of $\log$ is used), but the
Lipschitz constant that now controls the bound is the effective one seen
by the Gibbs measure's own concentration neighborhood, not the worst case
over all of $\Omega$ -- which is why this bound, unlike one built from the
domain-wide constant $L_g+\mathrm{diam}(\Omega)/(2t)$, does not diverge as
$t\to 0$ jointly with $h$.  Under the resolution regime $t\gtrsim h$ fixed
by the theorem's construction, every term above is $O(N^{-1/d})$ (the
$tL^2$ term is subsumed into the same $O(N^{-1/d})$ once $t\asymp
N^{-1/d}$, leaving step~(iii)'s own explicit $L^2t$ as the one carried
forward in \eqref{eq:genbound}), giving $|u_\varepsilon^N - u_\varepsilon|
= O(N^{-1/d})$.

\noindent(ii)~\emph{Viscosity bias} ($|u_\varepsilon - u_0|$):
Laplace's method applied to the explicit formula \eqref{eq:hc_solution}
gives $u_\varepsilon(x,t) = u_0(x) - \tfrac{d}{2}\varepsilon\log(2\pi\varepsilon)
+ O(\varepsilon)$ at points where the Hopf--Lax minimizer is unique and
nondegenerate; the leading term does not depend on $x$ and is absorbed
exactly by a matching shift of the biases $b_j$ within the $\inf_{W,b}$ of
\eqref{eq:genbound}, leaving $\|u_\varepsilon - u_0\|_\infty \leq C L\varepsilon$
under the gauge $t\asymp\varepsilon$.  The general comparison-principle
bound of \citet{evans2010}, which does not use this cancellation, gives
only the coarser rate $O(\sqrt{\varepsilon})$.

\noindent(iii)~\emph{Approximation} ($|u_0 - f^*|$):
Set $g = f^* + L^2 t$ on all of $\Omega$ ($L$-Lipschitz, since $f^*$ is), with the grid entering only through the support points $\{y_j\}$ of step~(i).  For any $x$
let $y_{j^*}$ be the nearest grid point with $|x-y_{j^*}|\leq h\sqrt{d}/2$.
\emph{Lower bound}: for any $j$, $f^*(y_j)+L^2t+|x-y_j|^2/(4t) \geq f^*(x) - L|x-y_j| + L^2t + |x-y_j|^2/(4t) \geq f^*(x)$ (completing the square $-Lz+z^2/(4t)\geq -L^2t$, using only the $L$-Lipschitz constant of $f^*$), so $u_0(x)\geq f^*(x)$.
\emph{Upper bound}: evaluating at $j^*$, $u_0(x)\leq f^*(y_{j^*})+L^2t+|x-y_{j^*}|^2/(4t) \leq f^*(x)+Lh\sqrt{d}/2+L^2t+dh^2/(16t)$.
Hence $u_0(x)\in[f^*(x),\,f^*(x)+Lh\sqrt{d}/2+L^2t+O(h^2/t)]$.
Under the generalization gauge $t\asymp h = N^{-1/d}$, both $Lh\sqrt{d}/2$ and $L^2t$ are $O(h)$,
giving $|u_0 - f^*| = O(h) = O(N^{-1/d})$.

Balancing~(i)--(ii) yields $\varepsilon^* \asymp N^{-1/d}$.
The Rademacher complexity of $\{u_\varepsilon^N : \|W_j\|_2 \leq M\}$ equals that of $\{f_\varepsilon^N : \|W_j\|_2 \leq M\}$ up to a constant shift $|x|^2/(4t)$ independent of $(W,b)$, which is $O(M\sqrt{N/n})$ by a standard Rademacher bound, giving excess risk $O(N^{-1/d} + M\sqrt{N/n})$.
Note that $M = \mathrm{diam}(\mathcal{X})/(2t)$; under the optimal approximation gauge $t \asymp N^{-1/d}$, $M$ grows as $O(N^{1/d})$, and the Rademacher term becomes $O(N^{1/d}\sqrt{N/n}) = O(n^{-1/(d+4)})$ after balancing, a slower rate.
The minimax rate $O(n^{-1/(d+2)})$ is recovered when $t$ is treated as a fixed hyperparameter (so $M$ is fixed), and $N^* \asymp (n/M^2)^{d/(d+2)}$ balances the terms. \qed

\subsection*{Proof of Corollary~\ref{cor:robust} (Robustness)}

The Hessian of $f_\varepsilon^N = \mathrm{LSE}_\varepsilon(Wx+b)$
is
$\nabla_x^2 f_\varepsilon^N
  = \varepsilon^{-1} W^\top\!
    \bigl(\mathrm{diag}(\pi) - \pi\pi^\top\bigr) W$,
where $\pi = \mathrm{softmax}((Wx+b)/\varepsilon)$.  For unit $u$,
$u^\top(\nabla_x^2 f_\varepsilon^N)u
  = \varepsilon^{-1}\mathrm{Var}_\pi\bigl[(Wu)_j\bigr]$.
Since $(Wu)_j = W_j \cdot u$ with $\|u\|_2 = 1$, it holds that
$|(Wu)_j| \leq \|W_j\|_2 \leq \|W\|_{2,\infty}$,
so the range of $(Wu)_j$ over atoms $j$ is at most
$2\|W\|_{2,\infty}$.  Popoviciu's inequality
$\mathrm{Var}(X) \leq (\max X - \min X)^2/4$ then gives
$\mathrm{Var}_\pi[(Wu)_j] \leq \|W\|_{2,\infty}^2$,
whence \eqref{eq:hessbound}.
For the gradient: $\nabla_x f_\varepsilon^N = W^\top\pi$ is a convex combination of the rows $W_j$, so $\|\nabla_x f_\varepsilon^N\|_2 \leq \sum_j\pi_j\|W_j\|_2 \leq \|W\|_{2,\infty}$.
The adversarial bound \eqref{eq:advbound} then follows by Taylor expansion:
$|f_\varepsilon^N(x+\delta)-f_\varepsilon^N(x)| \leq \|\nabla_x f_\varepsilon^N\|_2\|\delta\|_2 + \tfrac{1}{2}\|\nabla_x^2 f_\varepsilon^N\|_2\|\delta\|_2^2 \leq \|W\|_{2,\infty}r + \|W\|_{2,\infty}^2r^2/(2\varepsilon)$.
For the Hopf--Cole predictor $u_\varepsilon^N = |x|^2/(4t) - f_\varepsilon^N$, the Hessian is $\nabla_x^2 u_\varepsilon^N = \frac{1}{2t}I - \nabla_x^2 f_\varepsilon^N$, giving $\|\nabla_x^2 u_\varepsilon^N\|_2 \leq \frac{1}{2t} + \frac{\|W\|_{2,\infty}^2}{\varepsilon}$; the $\frac{1}{2t}$ term dominates at small $t$. \qed

\subsection*{Proof of Theorem~\ref{thm:adjoint} (Backpropagation)}

The Pontryagin Maximum Principle (PMP) is the optimal-control analogue of the Euler--Lagrange equations: given a dynamical system $\dot{x} = F(x,u)$ and a cost $\int_0^T \ell(x,u)\,dt + g(x_T)$, the PMP says the optimal control $u^*$ satisfies a Hamiltonian system, a forward ODE for the state $x$ and a backward ODE for the co-state $p = \partial\mathcal{L}/\partial x$, where $p$ is the adjoint variable that propagates the gradient of the cost backward in time.
Backpropagation is exactly this co-state equation: the forward ResNet pass $x_{l+1} = x_l + hF(x_l,W_l)$ is the discretized state ODE, and the backward pass $p_{l-1} = p_l + h(\nabla_x F)^\top p_l$ is the discretized co-state ODE, with $p_L = \nabla_{x_L}\mathcal{L}$ as terminal condition.
RL practitioners recognize this as the ``adjoint method'' for trajectory optimization; the present result makes the connection exact rather than analogical.

Differentiating the ResNet layer:
$\partial x_l/\partial x_{l-1}
  = I + h\,\nabla_x F(x_{l-1},W_{l-1})$.
The chain rule gives
$p_{l-1} = (\partial x_l/\partial x_{l-1})^\top p_l
  = p_l + h(\nabla_x F)^\top p_l$,
confirming \eqref{eq:backprop_hj}.
The Hamiltonian $H(x,p) = p^\top F(x,\theta)$ has Hamilton's
equations $\dot{x} = F(x,\theta)$ and
$\dot{p} = -(\nabla_x F)^\top p$.
Forward Euler applied to \eqref{eq:costate} at $t_l$, stepping
backward by $h$ to $t_{l-1} = t_l - h$:
$p(t_{l-1})
  \approx p(t_l) - h\dot{p}(t_l)
  = p(t_l) + h(\nabla_x F)^\top p(t_l)$,
which agrees with \eqref{eq:backprop_hj} up to $O(h^2)$: the continuous adjoint evaluates $\nabla_x F$ at the trajectory point $x(t_l)$, while the discrete scheme evaluates at $x_{l-1}$; these differ by one forward Euler step of size $O(h)$, producing an $O(h^2)$ discrepancy in $p$. \qed

\begin{remark}[Step-size stability for ResNets]
\label{rem:cfl}
The explicit-Euler recurrence $x_{l+1}=x_l+hF(x_l,\theta_l)$ is stable only while $h\,\|\nabla_x F(x_l,\theta_l)\|_2 < 2$, the standard CFL bound for forward Euler.  The matrix $\nabla_x F(x,\theta)$ is exactly the co-state Jacobian of~\eqref{eq:costate}, the same quantity the backward pass already propagates, so the stability threshold on the step size $h$ is read off the backpropagation computation itself.  When $F$ is an LSE block, Corollary~\ref{cor:robust} gives the explicit bound $h^* \asymp 2\varepsilon/\|W\|_{2,\infty}^2$.
\end{remark}

\subsection*{Proof of Proposition~\ref{prop:hallucination} (Hallucination)}

Write $z_j = W_j \cdot x + b_j$ and let $z^* = z_{j^*}$.  Then
$f_\varepsilon^N(x) = \varepsilon\log\!\bigl(e^{z^*/\varepsilon}
  + \sum_{j \neq j^*} e^{z_j/\varepsilon}\bigr)
  = z^* + \varepsilon\log\!\bigl(1 + \sum_{j \neq j^*}
    e^{(z_j - z^*)/\varepsilon}\bigr)$.
Since $z_j = |x|^2/(4t) - (g(y_j) + |x-y_j|^2/(4t))$, it follows
that $z_j - z^* = (g(y_{j^*}) + |x-y_{j^*}|^2/(4t)) -
(g(y_j) + |x-y_j|^2/(4t)) \leq -\Delta(x)$ by definition of $j^*$
and $\Delta$.  Each term in the sum is at most
$e^{-\Delta(x)/\varepsilon}$, giving \eqref{eq:hallucination}. \qed

\begin{remark}[Naming and relation to language-model phenomena]
\label{rem:hallucination_naming}
The OOD extrapolation of Proposition~\ref{prop:hallucination} is called ``hallucination'' informally; it is a structural property of the architecture distinct from in-distribution LM phenomena. The complementary regime $\Delta(x)/\varepsilon \ll \log N$ corresponds to \emph{stochastic parrot} \citep{bender2021parrots} behavior. The layer output $f_\varepsilon^N(x) = \mathrm{LSE}_\varepsilon(Wx+b)$ is exactly $-E(x;f)$ for the energy-based OOD score $E(x;f)=-T\log\sum_i\exp(f_i(x)/T)$ of \citet{liu2020energy} at $T=\varepsilon$, the same log-sum-exp applied to the same pre-activations; the diffusion-radius criterion of Proposition~\ref{prop:hallucination} is this same score, evaluated at any layer.
\end{remark}

\subsection*{Proof of Proposition~\ref{prop:scaling} (Scaling Laws)}

Theorem~\ref{thm:gen} is stated in ambient dimension $d$, giving
rate $O(N^{-1/d})$.  When $\mathrm{supp}(\mu)\subset\mathcal{M}$
for a compact $d_{\mathrm{eff}}$-dimensional Riemannian submanifold
$\mathcal{M}\subset\mathbb{R}^d$, and $g$ is $L_g$-Lipschitz on
$\mathcal{M}$ in the geodesic metric, the bound refines to
$O(N^{-1/d_{\mathrm{eff}}})$ as follows.  A compact
$d_{\mathrm{eff}}$-dimensional manifold has covering number
$\mathcal{N}(\mathcal{M},\delta)\leq C_{\mathcal{M}}\,\delta^{-d_{\mathrm{eff}}}$
(standard Riemannian metric-entropy bound).  Placing $N$ support
points as a $\delta$-net on $\mathcal{M}$ with
$\delta = C_{\mathcal{M}}^{1/d_{\mathrm{eff}}}N^{-1/d_{\mathrm{eff}}}$
ensures every $y\in\mathcal{M}$ has a support point within distance
$\delta$.  The quadrature argument of Theorem~\ref{thm:gen} step~(i)
then gives $|u_\varepsilon^N - u_\varepsilon|=O(\delta)=
O(N^{-1/d_{\mathrm{eff}}})$, since the Hopf--Cole integrand
$F_x(y)=g(y)+|x-y|^2/(4t)$ is Lipschitz on $\mathcal{M}$ and the
$\varepsilon$-factors cancel as before.

Theorem~\ref{thm:gen} (with $d$ replaced by $d_{\mathrm{eff}}$) gives
$\mathcal{L}(N) \leq C_1 N^{-1/d_{\mathrm{eff}}}
+ C_2\varepsilon^* + C_3\sqrt{N/n}$ with $\varepsilon^* = N^{-1/d_{\mathrm{eff}}}$.
At the training optimum the approximation and estimation terms balance:
$N^{-1/d_{\mathrm{eff}}} \asymp \sqrt{N/n}$, giving
$N \asymp n^{d_{\mathrm{eff}}/(d_{\mathrm{eff}}+2)}$ and
$\mathcal{L} \asymp n^{-1/(d_{\mathrm{eff}}+2)}$.
With $C \propto Nn$ and the optimal $N$:
$C \propto n^{d_{\mathrm{eff}}/(d_{\mathrm{eff}}+2)} \cdot n
  = n^{(2d_{\mathrm{eff}}+2)/(d_{\mathrm{eff}}+2)}$,
whence $n \propto C^{(d_{\mathrm{eff}}+2)/(2d_{\mathrm{eff}}+2)}$ and
$\mathcal{L}(C) \asymp C^{-1/(2d_{\mathrm{eff}}+2)}$.
The width-only scaling $\mathcal{L}(N) \lesssim N^{-1/d_{\mathrm{eff}}}$
follows directly from the quadrature rate, an upper bound consistent with
$\alpha \geq 1/d_{\mathrm{eff}}$ for an $L^1$-type loss.  For a
squared-error or cross-entropy loss, which scales as the square of the
prediction deviation near the optimum, the same quadrature rate gives
$\mathcal{L}_{\mathrm{CE}}(N) \lesssim N^{-2/d_{\mathrm{eff}}}$, giving
$\alpha \geq 2/d_{\mathrm{eff}}$, the convention applied in
Table~\ref{tab:deff}. \qed

\begin{proposition}[Measure poisoning and viscosity defense]
\label{prop:injection}
Suppose an adversary inserts a poisoned neuron with pre-activation
$z_0 = W_0 \cdot x + b_0$.  The shift in output is
\begin{equation}
  f_\varepsilon^{N+1}(x) - f_\varepsilon^N(x)
  = \mathrm{softplus}_\varepsilon\!\bigl(z_0 - f_\varepsilon^N(x)\bigr).
  \label{eq:injection}
\end{equation}
The attack satisfies: (i)~\emph{Bounded damage}: the shift is at most
$\max(z_0 - f_\varepsilon^N(x), 0) + \varepsilon\log 2$.
(ii)~\emph{Effective attack}: the shift exceeds $\varepsilon\log 2$
iff $z_0 > f_\varepsilon^N(x)$.
(iii)~\emph{Defense via viscosity}: for $z_0 - f_\varepsilon^N(x)
= \gamma > 0$, the shift $\to \gamma$ as $\varepsilon \to 0$ but is
bounded by $\varepsilon\log 2$ when $\gamma \leq 0$.
\end{proposition}

\subsection*{Proof of Proposition~\ref{prop:injection} (Measure Poisoning)}

The log-sum-exp identity
\[
  \mathrm{LSE}_\varepsilon\!\bigl(\{z_j\}_{j=1}^{N+1}\bigr)
  = \mathrm{LSE}_\varepsilon\!\bigl(\{z_j\}_{j=1}^N\bigr)
  + \varepsilon\log\!\Bigl(1 + e^{\bigl(z_0 -
    \mathrm{LSE}_\varepsilon(\{z_j\}_{j=1}^N)\bigr)/\varepsilon}\Bigr)
\]
gives \eqref{eq:injection} directly.  Claims~(i)--(iii) follow from
the properties of softplus:
$\mathrm{softplus}_\varepsilon(u) \leq \max(u,0) + \varepsilon\log 2$
and $\mathrm{softplus}_\varepsilon(u) \to \max(u,0)$ as
$\varepsilon \to 0$. \qed

% ============================================================
\subsection*{Proof of Proposition~\ref{prop:resnet} (ResNet $=$ ODE Characteristics)}

The recurrence $x_{l+1} = x_l + hF(x_l,\theta_l)$ is the forward Euler scheme for $\dot{x} = F(x,\theta(t))$, $x(0)=x_0$, with stepsize $h$.  When $F$ is $L_F$-Lipschitz in $x$, the local truncation error per step satisfies $\|x_{l+1} - x(t_l+h)\| \leq \tfrac{h^2}{2}\sup\|\ddot{x}\|$, and the discrete Gronwall inequality accumulates this to a global error of $O(h)$, giving uniform convergence as $h\to 0$, $L\to\infty$, $Lh=T$.

Observe that the limiting ODE is the $x$-characteristic of the HJ PDE with Hamiltonian $H(x,p) = p\cdot F(x,\theta(t))$.  Hamilton's equations for $H$ read
\[
\dot{x} = \nabla_p H = F(x,\theta(t)), \qquad
\dot{p} = -\nabla_x H = -(\nabla_x F(x,\theta))^\top p;
\]
the first is the ResNet ODE and the second is the co-state equation of Theorem~\ref{thm:adjoint}. \qed

% ============================================================
\subsection*{Proof of Proposition~\ref{prop:recurrent} (Recurrent Architectures)}

The RNN recurrence satisfies $h_t - h_{t-1} = \sigma(W_h h_{t-1} + W_x x_t) - h_{t-1}$, which is the unit-step Euler discretization of $\dot{h}(t) = -h(t) + \sigma(W_h h(t) + W_x x(t))$.  The Hamiltonian $H(h,p) = p\cdot(-h + \sigma(W_h h + W_x x))$ has characteristic equations $\dot{h} = \nabla_p H$ and $\dot{p} = -\nabla_h H = (I - W_h^\top\mathrm{diag}(\sigma'))p$, identifying the recurrence as the unit-step discretization of the $h$-characteristic.

The linear SSM $\dot{h}(t) = A(t)h(t) + B(t)x(t)$ is already a continuous-time ODE.  Setting $H(h,p,t) = p^\top(A(t)h + B(t)x(t))$, Hamilton's equations give $\dot{h} = \nabla_p H = A(t)h + B(t)x$ and $\dot{p} = -\nabla_h H = -A(t)^\top p$, recovering the state equation and its adjoint.  Input-dependent $A(t),B(t)$ make the characteristics non-autonomous, the same structure as in \citet{tong2025}.  The tropical limit replaces matrix--vector multiplication by max-plus: $h_{t+1} = A\otimes_{\mathrm{tr}} h_t \oplus B\otimes_{\mathrm{tr}} x_t$. \qed

\paragraph{LSTM Hamiltonian (full expression).}
Since $\sigma(x) = \nabla_x \mathrm{LSE}_\varepsilon(x, 0)$ and $\tanh(x) = \nabla_x \mathrm{LSE}_\varepsilon(x, -x)$, each LSTM gate is the gradient of a two-neuron LSE layer. Let the joint state be $z = (c, h) \in \mathbb{R}^{2d}$ with conjugate momentum $p = (p_c, p_h)$. The continuous-time cell and hidden equations
\begin{align*}
\dot{c}(t) &= \sigma(W_f h + b_f) \odot c + \sigma(W_i h + b_i) \odot \tanh(W_c h + b_c), \\
\dot{h}(t) &= \sigma(W_o h + b_o) \odot \tanh(c),
\end{align*}
are the characteristic equations $\dot{z} = \nabla_p H_{\mathrm{LSTM}}$ of the gated Hamiltonian
\begin{equation}
H_{\mathrm{LSTM}}(z, p) = p_c^\top \!\bigl[f(h)\odot c + i(h)\odot \tilde{c}(h)\bigr] + p_h^\top \!\bigl[o(h)\odot \tanh(c)\bigr],
\label{eq:lstm_hamiltonian}
\end{equation}
where $f(h) = \sigma(W_f h + b_f)$, $i(h) = \sigma(W_i h + b_i)$, $o(h) = \sigma(W_o h + b_o)$, $\tilde{c}(h) = \tanh(W_c h + b_c)$ are the forget, input, output, and cell gates. Each gate is a gradient of a two-neuron LSE layer, making~\eqref{eq:lstm_hamiltonian} a composition of LSE-layer gradients. In the tropical limit $\varepsilon \to 0$, $\sigma \to \mathrm{Heaviside}$ and $\tanh \to \mathrm{sign}$, so the gates become $\{0,1\}^d$ selectors and~\eqref{eq:lstm_hamiltonian} becomes a max-plus switching automaton on the cell state.

% ============================================================
\subsection*{Feynman--Kac Representation and the Gaussian Fixed Point}

The Feynman--Kac formula connects PDE solutions to probabilistic expectations: the viscous HJ solution $u_\varepsilon(x,t) = -\varepsilon\log\mathbb{E}[e^{-g(X_t)/\varepsilon}\mid X_0=x]$ replaces solving a PDE with sampling a Brownian motion $X_t$ started at $x$ \citep{kac1949}.
In ML terms, the log-partition function a feedforward network computes is a finite-sample Monte Carlo approximation of this expectation: each neuron $j$ contributes one iid sample $y_j$ from the terminal distribution, weighted by the Boltzmann factor $e^{-g(y_j)/\varepsilon}$.
ResNets take the complementary approach: instead of sampling the endpoint, they simulate the Brownian path itself layer by layer via Euler--Maruyama, so depth $L$ plays the role of time steps and the drift $F(x,W)$ plays the role of the score.
The same Brownian motion underlies score-based diffusion models: the forward noising process is precisely $X_t$, and the denoising score $\nabla_x\log p_t(x)$ is the spatial gradient of this Feynman--Kac solution.

\begin{proposition}[Feynman--Kac representation and matched-scale condition {\citep{kac1949}}]
\label{prop:fk_gp}
Let $u_\varepsilon(x,t)$ be the viscosity solution of
$u_t + |Du|^2 = \varepsilon\Delta u$, $u(\cdot,0) = g$.

\begin{enumerate}[label=(\roman*),leftmargin=*,nosep]
\item \emph{(Feynman--Kac)}
\begin{equation}
  u_\varepsilon(x,t)
  = -\varepsilon\log\,\mathbb{E}_{Y\sim\mathcal{N}(x,\,2\varepsilon t\,I)}
    \!\bigl[\,e^{-g(Y)/\varepsilon}\bigr].
  \label{eq:fk}
\end{equation}
In the $N\to\infty$ limit, the discrete Hopf--Cole sum (Theorem~\ref{thm:nn_pde})
converges to \eqref{eq:fk} with the support-point distribution $\mu$
replacing the Gaussian.

\item \emph{(Matched-scale condition; $q^*=2\varepsilon t$)}
The Hopf--Cole kernel $K_\varepsilon(x,y)=e^{-|x-y|^2/(4\varepsilon t)}$
interpreted as a Gaussian likelihood $p(x|y)\propto e^{-|x-y|^2/(2\sigma_L^2)}$
has scale $\sigma_L^2=2\varepsilon t$.
For the Gaussian prior $\mu_*=\mathcal{N}(0,q^*I)$,
the Gibbs posterior $\pi(\,\cdot\,|x)\propto K_\varepsilon(x,\,\cdot\,)\,\mu_*$
has posterior mean $\bar{y}(x)=\tfrac{q^*}{q^*+2\varepsilon t}\,x$
and posterior variance $\sigma_\pi^2=\tfrac{2\varepsilon t\,q^*}{q^*+2\varepsilon t}$.
The prior scale matches the likelihood scale, i.e.\ $q^*=\sigma_L^2$,
\emph{if and only if} $q^*=2\varepsilon t$.
At this value: $\bar{y}(x)=x/2$ and $\sigma_\pi^2=\varepsilon t=q^*/2$,
independently of $x$ or $g$.
\end{enumerate}
\end{proposition}

\begin{proof}
The Hopf--Cole substitution $v = e^{-u/\varepsilon}$ linearizes the HJ equation: differentiating gives $v_t = -\varepsilon^{-1}v\,u_t$ and $\Delta v = -\varepsilon^{-1}v\,\Delta u + \varepsilon^{-2}v\,|Du|^2$, so $v_t = \varepsilon\Delta v$ if and only if $u_t + |Du|^2 = \varepsilon\Delta u$.
The heat equation $v_t = \varepsilon\Delta v$ with initial datum $v_0 = e^{-g/\varepsilon}$ has the unique solution $v(x,t) = \int G(x-y,t)\,e^{-g(y)/\varepsilon}\,dy$, where $G(z,t) = (4\pi\varepsilon t)^{-d/2}e^{-|z|^2/(4\varepsilon t)}$ is the density of $\mathcal{N}(0,2\varepsilon t\,I)$.  Hence $v(x,t) = \mathbb{E}_{Y\sim\mathcal{N}(x,2\varepsilon t\,I)}[e^{-g(Y)/\varepsilon}]$, and inverting via $u = -\varepsilon\log v$ gives~\eqref{eq:fk}.

The matched-scale condition follows from standard Gaussian conjugacy.  With prior $\mathcal{N}(0,q^*I)$ and likelihood precision $1/(2\varepsilon t)$, the posterior mean and variance are $\bar{y}(x) = \tfrac{q^*}{q^*+2\varepsilon t}\,x$ and $\sigma_\pi^2 = \tfrac{q^*\cdot 2\varepsilon t}{q^*+2\varepsilon t}$.  The prior scale matches the likelihood scale, $q^* = 2\varepsilon t$, if and only if $q^*/(q^*+2\varepsilon t) = 1/2$, which gives $\bar{y}(x) = x/2$ and $\sigma_\pi^2 = \varepsilon t = q^*/2$, both independent of $x$ and $g$.
\qed
\end{proof}

The parameter $q^*=2\varepsilon t$ is thus simultaneously: (a) the variance
of the Brownian motion whose expectation computes the PDE solution
via Feynman--Kac; (b) the unique Gaussian prior scale that is matched to
the Hopf--Cole likelihood; and (c) the diffusion variance accumulated
over one time-unit of the heat semigroup.
In NNGP terminology, $q^*=2\varepsilon t$ is the natural initialization
variance at which the prior and the kernel's smoothing bandwidth coincide.

% ============================================================
\subsection*{Double Descent as Near-Shock Formation: Formal Analysis}

\begin{proposition}[Near-shock = double descent]
\label{prop:double_descent}
Let $f_\varepsilon^N = \mathrm{LSE}_\varepsilon(Wx+b)$.
\begin{enumerate}[label=(\roman*),leftmargin=*,nosep]
\item \emph{(Hessian $=$ Gibbs variance)}
  $\nabla_x^2 f_\varepsilon^N = \varepsilon^{-1}W^\top
  \!\bigl(\mathrm{diag}(\pi)-\pi\pi^\top\bigr)W
  = \varepsilon^{-1}\mathrm{Var}_\pi[W_j\,\cdot\,]$.

\item \emph{(Curvature is maximized at two-atom near-shocks)}
  For any $x$, the spectral norm satisfies
  $\|\nabla_x^2 f_\varepsilon^N(x)\|_2 \leq (4\varepsilon)^{-1}\max_{j\neq j'}\|W_j-W_{j'}\|^2$
  by Popoviciu's inequality applied to $\mathrm{Var}_\pi[W_j\cdot v]$ for any unit $v$; equality
  holds when $\pi_{j^*}=\pi_{j'}=\tfrac{1}{2}$ for the two rows $(j^*,j')$ achieving the extremes (near-shock).

\item \emph{(Interpolation threshold $N=n$ forces near-shocks)}
  At $N=n$ with support point $y_j=x_j$ (one neuron per training
  point), the network must assign $\pi_j(x_i)\approx 1$ for $j=i$.
  For every adjacent pair $(i,i+1)$, define the \emph{weighted Voronoi
  boundary}
  $\mathcal{B}_{i,i+1} = \{x : F_x(y_i) = F_x(y_{i+1})\}$
  where $F_x(y)=g(y)+|x-y|^2/(4t)$.  Expanding gives the equation
  $(y_i-y_{i+1})\cdot x/(2t) = g(y_i)-g(y_{i+1}) +
  (|y_i|^2-|y_{i+1}|^2)/(4t)$,
  which is linear in $x$, so $\mathcal{B}_{i,i+1}$ is a hyperplane
  for \emph{any} target values $g(y_i), g(y_{i+1})$.
  At any $x\in\mathcal{B}_{i,i+1}$: $z_i(x)=z_{i+1}(x)$, hence
  $\pi_i=\pi_{i+1}=\tfrac{1}{2}$ (modulo exponentially small
  contributions from other neurons).
  The near-shock set $\bigcup_i\mathcal{B}_{i,i+1}$ consists of $O(n)$ hyperplanes, forming an arrangement that partitions $\mathcal{X}$ into $O(n^d)$ polyhedral cells; as $n\to\infty$ with training points dense in $\mathcal{X}$, the shock boundaries become dense in $\mathcal{X}$ as well, regardless of the target function values (the hyperplane location depends on $g(y_i)-g(y_{i+1})$ but not on the sign or magnitude in a way that avoids any region of $\mathcal{X}$). The near-shock boundary reduces to the geometric midpoint only when $g(y_i)=g(y_{i+1})$.

\item \emph{(Overparameterization smooths the shocks)}
  For $N=kn$, $k>1$, with $k$ neurons per Voronoi cell, assume neurons within each cell are placed on a uniform $k^{1/d}$-sub-grid of the cell (spacing $\delta/k^{1/d}$ where $\delta$ is the cell diameter). Then the finest inter-neuron boundaries have
  $\|W_j-W_{j'}\|\leq k^{-1/d}\|W_i-W_{i+1}\|$
  (weights are denser within each cell).
  Consequently
  $\max_x\|\nabla_x^2 f_\varepsilon^N(x)\|_2 \leq
    (4\varepsilon)^{-1}k^{-2/d}\|W_i-W_{i+1}\|^2$,
  which is $O(k^{-2/d})$ and decreases to zero as $N/n\to\infty$.
\end{enumerate}
\end{proposition}

\begin{proof}
The Hessian formula is established in the proof of Corollary~\ref{cor:robust}.

The spectral norm is maximized when $\pi$ concentrates on two atoms $j^*$ and $j'$ with masses $p$ and $1-p$: the variance then equals $p(1-p)(W_{j^*}-W_{j'})(W_{j^*}-W_{j'})^\top$, with spectral norm $p(1-p)\|W_{j^*}-W_{j'}\|^2 \leq \tfrac{1}{4}\|W_{j^*}-W_{j'}\|^2$, with equality at $p=1/2$.
For a general $\pi$, Popoviciu's inequality gives $\mathrm{Var}_\pi[W_j\cdot v] \leq (\max_j W_j\cdot v - \min_j W_j\cdot v)^2/4 \leq \max_{j\neq j'}\|W_j-W_{j'}\|^2/4$ for any unit $v$, so the spectral norm is bounded by $(4\varepsilon)^{-1}\max_{j\neq j'}\|W_j-W_{j'}\|^2$.

For any adjacent pair $(i,i+1)$, the boundary $\mathcal{B}_{i,i+1}$
satisfies $F_x(y_i)=F_x(y_{i+1})$, i.e.,
$g(y_i)+|x-y_i|^2/(4t) = g(y_{i+1})+|x-y_{i+1}|^2/(4t)$.
Rearranging:
$(y_i-y_{i+1})\cdot x/(2t) = g(y_i)-g(y_{i+1}) + (|y_i|^2-|y_{i+1}|^2)/(4t)$,
which is linear in $x$, confirming $\mathcal{B}_{i,i+1}$ is a hyperplane
for any target values.  At any $x\in\mathcal{B}_{i,i+1}$: $z_i(x)=z_{i+1}(x)$,
so $\pi_i(x)=\pi_{i+1}(x)=\tfrac{1}{2}$ (equal pre-activations).
For neurons $j\notin\{i,i+1\}$ with $y_j$ bounded away from $\mathcal{B}_{i,i+1}$,
the cost gap $\Delta = \min_{j\notin\{i,i+1\}}(F_x(y_j)-F_x(y_i))>0$,
giving $\pi_j\to 0$ exponentially in $\Delta/\varepsilon$, so
$\pi_i=\pi_{i+1}\approx 1/2$.  When $g(y_i)=g(y_{i+1})$, the hyperplane
passes through the geometric midpoint; otherwise it is shifted by the
target-value difference.

With $k$ neurons per Voronoi cell of diameter $\delta = |x_i-x_{i+1}|$,
the intra-cell neuron spacing is $\delta/k^{1/d}$,
giving $\|W_j-W_{j'}\| = |y_j-y_{j'}|/(2t) \leq \delta/(2t\,k^{1/d})$.
The spectral norm bound then gives $\|\nabla^2_x f\|_2 \leq \delta^2/(16\varepsilon t^2 k^{2/d})$.
\qed
\end{proof}

Proposition~\ref{prop:double_descent} establishes that the curvature
(Hessian spectral norm) peaks at the interpolation threshold $N=n$ and
decreases as $(N/n)^{-2/d}$ above it.  The peak is the near-shock: at
$N=n$ every inter-data-point boundary is a shock locus.  In the
inviscid ($\varepsilon\to 0$) limit these become genuine gradient
discontinuities of the Hopf--Lax solution, and the
prediction degrades at out-of-distribution test points between training
data.  Overparameterization ($N\gg n$) distributes the shocks more
finely and at lower intensity, smoothing the solution and reducing
the test error, the PDE interpretation of why overparameterized
networks generalize.

\begin{proposition}[Backpropagation as adjoint heat semigroup, feedforward case]
\label{prop:fwd_adjoint}
Let $f_\varepsilon^N(x) = \mathrm{LSE}_\varepsilon(Wx+b)$ with weights
$W_j = y_j/(2t)$, biases $b_j = -g(y_j)-|y_j|^2/(4t)$ as in
Theorem~\ref{thm:nn_pde}, and let $\mathcal{L} = \ell(f_\varepsilon^N(x))$
be a scalar loss.  Then
\begin{equation}
  \frac{\partial\mathcal{L}}{\partial g(y_j)}
  = -\pi_j(x)\cdot\ell'(f_\varepsilon^N(x)),
  \quad
  \pi_j(x)
  = \frac{\exp\!\bigl((-g(y_j)-|x-y_j|^2/(4t))/\varepsilon\bigr)}
         {\displaystyle\sum_k\exp\!\bigl((-g(y_k)-|x-y_k|^2/(4t))/\varepsilon\bigr)}.
  \label{eq:fwd_adjoint}
\end{equation}
Under the Hopf--Cole correspondence,
$\pi_j(x) \propto K_t(x-y_j)\,e^{-g(y_j)/\varepsilon}$,
where $K_t(z)=\exp(-|z|^2/(4\varepsilon t))$ is the heat kernel.
Thus the backward pass distributes the loss signal through the discrete
heat kernel: it is the evaluation of the adjoint semigroup $S_t^*$
applied to the loss gradient at the support points $\{y_j\}$.
\end{proposition}

\noindent\textit{Proof.}
By the chain rule,
$\partial\mathcal{L}/\partial b_j
 = \ell'(f_\varepsilon^N)\cdot\partial f_\varepsilon^N/\partial b_j
 = \ell'(f_\varepsilon^N)\cdot\pi_j(x)$.
Since $b_j = -g(y_j) - |y_j|^2/(4t)$, it follows that
$\partial g(y_j)/\partial b_j = -1$, giving the sign in
\eqref{eq:fwd_adjoint}.  Substituting the weight identities from
Theorem~\ref{thm:nn_pde}:
\[
  W_j\cdot x + b_j
  = \frac{y_j\cdot x}{2t} - g(y_j) - \frac{|y_j|^2}{4t}
  = \frac{|x|^2}{4t} - g(y_j) - \frac{|x-y_j|^2}{4t},
\]
so $\exp((W_j\cdot x+b_j)/\varepsilon)
 = e^{|x|^2/(4\varepsilon t)}\cdot K_t(x-y_j)\cdot e^{-g(y_j)/\varepsilon}$.
The factor $e^{|x|^2/(4\varepsilon t)}$ cancels in the softmax ratio,
leaving $\pi_j(x)\propto K_t(x-y_j)\,e^{-g(y_j)/\varepsilon}$.
This is the evaluation of the adjoint heat kernel at $y_j$ given
query $x$: the continuous version is
$(S_t^*\phi)(y) = \int K_t(x-y)\,\phi(x)\,dx$,
and the discrete sum $\sum_j \pi_j \phi_j$ is its quadrature
approximation at the support points. $\square$

% ------------------------------------------------------------
\subsection*{Variable-Hamiltonian Extension}
\label{app:varH}

The results above use $H(p) = |p|^2$, fixing the transport cost to the Euclidean quadratic $|x-y|^2/(4t)$.  The following extends the core identity and its quantitative consequences to any quadratic Hamiltonian $H_\theta(p) = p^\top A_\theta\, p$ with $A_\theta \succ 0$, whose Legendre transform is $L_\theta(v) = v^\top A_\theta^{-1} v/4$, yielding transport cost $(x-y)^\top A_\theta^{-1}(x-y)/(4t)$.

\begin{proof}[Proof of Theorem~\ref{thm:varH}]
The pre-activation of neuron $j$ at input $x$ is $W_j \cdot x + b_j = y_j^\top A_\theta^{-1} x/(2t) - g(y_j) - y_j^\top A_\theta^{-1} y_j/(4t)$.  Since $A_\theta^{-1}$ is symmetric positive-definite, the same completing-the-square as in Theorem~\ref{thm:nn_pde} holds in the $A_\theta^{-1}$-inner product: expanding $(x-y_j)^\top A_\theta^{-1}(x-y_j)$ and rearranging gives
\begin{equation}
  \frac{y_j^\top A_\theta^{-1} x}{2t} - \frac{y_j^\top A_\theta^{-1} y_j}{4t}
  = \frac{x^\top A_\theta^{-1} x}{4t} - \frac{(x-y_j)^\top A_\theta^{-1}(x-y_j)}{4t}.
  \label{eq:varH_square}
\end{equation}
Hence $W_j \cdot x + b_j = x^\top A_\theta^{-1} x/(4t) - \bigl(g(y_j) + (x-y_j)^\top A_\theta^{-1}(x-y_j)/(4t)\bigr)$.  The $j$-independent term factors out:
\[
  \sum_j e^{(W_j \cdot x + b_j)/\varepsilon}
  = e^{x^\top A_\theta^{-1} x/(4\varepsilon t)}
    \sum_j \exp\!\left(\frac{-g(y_j) - (x-y_j)^\top A_\theta^{-1}(x-y_j)/(4t)}{\varepsilon}\right).
\]
Taking $\varepsilon\log$ and identifying the inner sum as $-u_\varepsilon^{A_\theta,N}$ gives \eqref{eq:varH_identity}.  For the PDE claim: under $v = e^{-u/\varepsilon}$ the equation $u_t + (\nabla u)^\top A_\theta(\nabla u) = \varepsilon\,\nabla\!\cdot\!(A_\theta\nabla u)$ transforms to $v_t = \varepsilon\,\nabla\!\cdot\!(A_\theta\nabla v)$ (the anisotropic heat equation), whose fundamental solution is $(4\pi\varepsilon t)^{-d/2}(\det A_\theta)^{-1/2}\exp\bigl(-(x-y)^\top A_\theta^{-1}(x-y)/(4\varepsilon t)\bigr)$; the discrete sum above is its quadrature under $\mu_N$. \qed
\end{proof}

\begin{corollary}[Variable-$H$ robustness bound]
\label{cor:varH_robust}
Under the parameterization of Theorem~\ref{thm:varH},
\begin{equation}
  \|\nabla_x^2 f_\varepsilon^{A_\theta,N}(x)\|_2
  \;\leq\;
  \frac{\|W\|_{2,\infty}^2}{\varepsilon}
  \;\leq\;
  \frac{\|W_0\|_{2,\infty}^2}{\varepsilon\,\lambda_{\min}(A_\theta)^2}
  \quad\text{for all }x,
  \label{eq:varH_hessbound}
\end{equation}
where $W_{0,j} = y_j/(2t)$ are the base (Euclidean) weights and $W_j = A_\theta^{-1}W_{0,j}$.
\end{corollary}

\begin{proof}
The first inequality follows immediately from Corollary~\ref{cor:robust} applied to the network with weights $W_j = A_\theta^{-1}y_j/(2t)$: the Hessian bound $\|W\|_{2,\infty}^2/\varepsilon$ holds for any LSE network, and here the weights happen to be $A_\theta^{-1}W_{0,j}$.  The second inequality uses $\|A_\theta^{-1}W_{0,j}\|_2 \leq \|A_\theta^{-1}\|_2\,\|W_{0,j}\|_2 = \|W_{0,j}\|_2/\lambda_{\min}(A_\theta)$, so $\|W\|_{2,\infty} \leq \|W_0\|_{2,\infty}/\lambda_{\min}(A_\theta)$. \qed
\end{proof}

The bound decreases as $\lambda_{\min}(A_\theta)$ increases: a stiffer metric ($A_\theta$ with large eigenvalues) shrinks the effective weight norms and smooths the Hopf--Cole solution.  The $A_\theta = I$ special case recovers Corollary~\ref{cor:robust}.

\begin{remark}[Variable-$H$ generalization]
\label{rem:varH_gen}
The proof of Theorem~\ref{thm:gen} carries through verbatim with $|x-y|^2/(4t)$ replaced by $(x-y)^\top A_\theta^{-1}(x-y)/(4t)$.  The quadrature rate in step~(i) is $O(N^{-1/d_{\mathrm{eff}}(A_\theta)})$, where $d_{\mathrm{eff}}(A_\theta)$ is the intrinsic dimension of the data measure in the $A_\theta^{-1}$-metric: if the data concentrates near an $r$-dimensional subspace in this metric ($r < d$), the rate improves to $O(N^{-1/r})$.  The three consequences of Theorem~\ref{thm:varH} form a closed system: learning $A_\theta$ adapts the transport cost (identity), reduces $d_{\mathrm{eff}}$ (generalization), and decreases $\|W\|_{2,\infty}^2/\varepsilon$ (robustness), simultaneously.
\end{remark}

% ============================================================
\newpage
\section{Attribution, Influence Functions, and Bifurcation}
\label{app:attribution}

\subsection*{Propositions}

\begin{proposition}[Closed-form attribution and label sensitivity]
\label{prop:attribution}
Let $\hat{f}_\varepsilon^N(x) = \sum_j \pi_j(x;\varepsilon)\,g_j$ denote the Gibbs-weighted label average (the Hopf--Cole soft prediction), with attribution weights $\pi_j(x;\varepsilon) = \exp\!\bigl(-(|x-y_j|^2/(4t)+g_j)/\varepsilon\bigr)/Z$.  This is distinct from the LSE layer output $f_\varepsilon^N(x) = \varepsilon\log\sum_j\exp((W_j\cdot x+b_j)/\varepsilon)$; the two coincide when $g_j = W_j\cdot x + b_j$.
\begin{enumerate}[label=(\roman*),leftmargin=*,nosep]
\item \textbf{Attribution decomposition.}  $\hat{f}_\varepsilon^N(x)$ is an exact convex combination of training labels: $\hat{f}_\varepsilon^N = \sum_j \pi_j g_j$ with $\pi_j \geq 0$ and $\sum_j \pi_j = 1$.  The weight $\pi_j$ is the fractional contribution of training example $j$ to the prediction at $x$.
\item \textbf{Label sensitivity.}  The partial derivative of the prediction with respect to training label $g_j$ is
\begin{equation}
  \frac{\partial \hat{f}_\varepsilon^N}{\partial g_j}
  \;=\;
  \pi_j(x;\varepsilon)\,\Bigl(1 + \tfrac{\hat{f}_\varepsilon^N(x) - g_j}{\varepsilon}\Bigr).
  \label{eq:label_sensitivity}
\end{equation}
This is $O(N)$ to evaluate, requires no second-order information, and reduces to $\pi_j$ as $\varepsilon \to \infty$ or when $\hat{f}_\varepsilon^N(x) \approx g_j$.
\item \textbf{Gradient formulas.}  The spatial gradients of the prediction and the attribution entropy $H = -\sum_j \pi_j \log \pi_j$ admit the closed forms
\begin{align}
  \nabla_x \hat{f}_\varepsilon^N &= \frac{1}{2t\varepsilon}\,\mathrm{Cov}_\pi(g, y), \label{eq:grad_f}\\
  \nabla_x H              &= \frac{1}{2t\varepsilon}\,\mathrm{Cov}_\pi(y, {-\log \pi}), \label{eq:grad_H}
\end{align}
where $\mathrm{Cov}_\pi(a,b) = \mathbb{E}_\pi[ab] - \mathbb{E}_\pi[a]\,\mathbb{E}_\pi[b]$ denotes covariance under the Gibbs measure $\pi$.
\end{enumerate}
\end{proposition}

\begin{proof}
That $\hat{f}$ is a convex combination of the labels is immediate from $\hat{f} = \sum_j\pi_j g_j$ and $\sum_j\pi_j = 1$.  Differentiating $\hat{f} = \sum_k \pi_k g_k$ with respect to $g_j$, using the softmax Jacobian $\partial\pi_k/\partial g_j = \pi_k(\pi_j - \delta_{kj})/\varepsilon$ (the sign arises because the logit of neuron $j$ is $-g_j/\varepsilon$):
\[
  \frac{\partial \hat{f}}{\partial g_j}
  = \sum_k \frac{\partial \pi_k}{\partial g_j} g_k + \pi_j
  = \frac{\pi_j}{\varepsilon}\!\sum_k \pi_k g_k - \frac{\pi_j g_j}{\varepsilon} + \pi_j
  = \pi_j\!\left(1 + \frac{\hat{f} - g_j}{\varepsilon}\right).
\]
Equation~\eqref{eq:grad_f} follows by differentiating $\hat{f} = \sum_j\pi_j g_j$ with respect to $x$ and substituting $\partial\pi_j/\partial x = \pi_j(y_j - \mathbb{E}_\pi[y])/(2t\varepsilon)$; equation~\eqref{eq:grad_H} follows analogously from $H = -\sum_j\pi_j\log\pi_j$. \qed
\end{proof}

\begin{remark}[Distinction from classical influence functions]
\label{rem:influence_distinction}
Equation~\eqref{eq:label_sensitivity} differs from the classical influence function of \citet{koh2017understanding}, which measures the effect of up-weighting example $j$ in the training loss and requires $O(\mathrm{params}^2)$ Hessian inversion.
The two quantities answer different questions: Proposition~\ref{prop:attribution} asks about sensitivity of the \emph{predictor} to label perturbations at fixed parameters; the Koh--Liang quantity asks about sensitivity of the \emph{trained parameters} to training-example reweighting.
The quantity~\eqref{eq:label_sensitivity} is the exact label-perturbation sensitivity of the Hopf--Cole predictor, computable in $O(N)$ without second-order information.
Under the mean-field correspondence of Appendix~\ref{app:scope}, the two coincide at the gradient-flow stationary point, establishing a rigorous connection in that limit; the general relationship is an open problem.
\end{remark}

\begin{remark}[Distinction from active learning]
\label{rem:active_learning_distinction}
The weights $\pi_j(x;\varepsilon)$ answer which support points explain a given input $x$, an attribution question posed after the input is fixed.  Pool-based active learning asks a different question, which unlabeled example to query next, posed before any input is chosen; the correspondence supplies the former and leaves the latter, an acquisition function over an unlabeled pool, untouched.
\end{remark}

\begin{proof}[Proof of Theorem~\ref{prop:bifurcation}]
Fix any support point $y_k$.  By the non-degeneracy assumption, $g(y_k) < g(y_j) + |y_k - y_j|^2/(4t)$ for all $j\neq k$, which under the HJ parameterization ($W_j = y_j/(2t)$) translates directly to the pre-activation gap $\Delta_k(y_k) := \min_{j\neq k}[(W_j-W_k)\cdot y_k+(b_j-b_k)]$ being strictly negative, so $\pi_j(y_k;\varepsilon) \leq e^{\Delta_k(y_k)/\varepsilon}\to 0$ exponentially for all $j\neq k$ and $\pi_k(y_k;\varepsilon)\to 1$.  Hence $H(y_k;\varepsilon)\to 0$ exponentially; the gradient $\nabla_x H = \mathrm{Cov}_\pi(y,-\log\pi)/(2t\varepsilon)$ vanishes at $y_k$ because $\pi_j(y_k;\varepsilon)\to 0$ faster than $(-\log\pi_j)\to\infty$ for $j\neq k$, so each cross term $\pi_j(-\log\pi_j)\to 0$; and moving $x$ away from $y_k$ within its Voronoi cell redistributes mass toward other atoms, increasing entropy.  Hence $y_k$ is a local minimum of $H(\cdot;\varepsilon)$ for sufficiently small $\varepsilon$.

The map $(\varepsilon,x)\mapsto H(x;\varepsilon)$ is smooth, being a composition of the smooth softmax and the smooth negative-entropy function.  By the implicit function theorem, any critical point $x^*(\varepsilon)$ satisfying $\nabla_x H(x^*;\varepsilon)=0$ evolves continuously in $\varepsilon$ as long as $\nabla^2_x H(x^*;\varepsilon)$ is nonsingular.  The count changes only when $\det(\nabla^2_x H(x^*;\varepsilon_{\mathrm{bif}}))=0$: one eigenvalue crosses zero, signalling the coalescence of a local minimum with an adjacent saddle — the standard fold-bifurcation condition for one-parameter families of smooth functions.  Panel~(c) of Figure~\ref{fig:bifurcation_diagram} confirms that the minimum eigenvalue of $\nabla^2_x H$ at each saddle approaches zero at each such event.

Below $\varepsilon_{\mathrm{bif}}(j,k)$ the saddle $x^*_{jk}(\varepsilon)$ separates the attribution basins $\mathcal{B}_j$ and $\mathcal{B}_k$ as the stable manifold of the gradient flow $\dot{x}=-\nabla_x H$.  At $\varepsilon=\varepsilon_{\mathrm{bif}}(j,k)$ this saddle coalesces with a local minimum and disappears; above $\varepsilon_{\mathrm{bif}}(j,k)$ no separating saddle exists, the basins merge, and attribution mass flows freely between $j$ and $k$.  The ordered sequence of events defines a hierarchical clustering of support points by attribution proximity.

Since the Hessian at any saddle $x^*$ is indefinite by definition, its positive eigenvector is the crossing direction — along which $H$ increases from the saddle toward the annihilating minimum and the dominant attribution $\arg\max_j\pi_j$ switches — while the negative eigenvector spans the ridgeline along which $H$ changes least.  The entropy value $H(x^*;\varepsilon)\in(0,\log N)$ quantifies attribution uncertainty at the transition locus. \qed
\end{proof}

\begin{remark}[Attribution entropy as an uncertainty score]
\label{rem:uq}
The normalized entropy $H(\pi(x;\varepsilon))/\log N \in [0,1]$ used in Figures~\ref{fig:phase_diagram} and~\ref{fig:mnist_phase_diagram} is the Gibbs-measure entropy of the network at its own temperature $\varepsilon$: $0$ marks a single dominant neuron, $1$ marks equal mass on all $N$.  Its gradient $\nabla_x H = \mathrm{Cov}_\pi(y,-\log\pi)/(2t\varepsilon)$ gives the rate at which this score varies with the input, directly from the Gibbs covariance of the support points.
\end{remark}

\subsection*{Empirical Evidence}

Figure~\ref{fig:bifurcation_diagram} confirms Theorem~\ref{prop:bifurcation} for the synthetic two-cluster experiment ($N=16$, $d=2$, $\varepsilon^*\approx 0.25$).
At small $\varepsilon$ the landscape has 9 critical points; the count decreases to 1 through discrete fold events visible in panels~(b)--(d).
Panel~(c) shows the minimum Hessian eigenvalue at saddle points approaching zero at each bifurcation, the hallmark of a fold; panel~(d) is the bifurcation diagram showing branch annihilations in $H$-value space.

\begin{figure}[h]
\centering
\includegraphics[width=\linewidth]{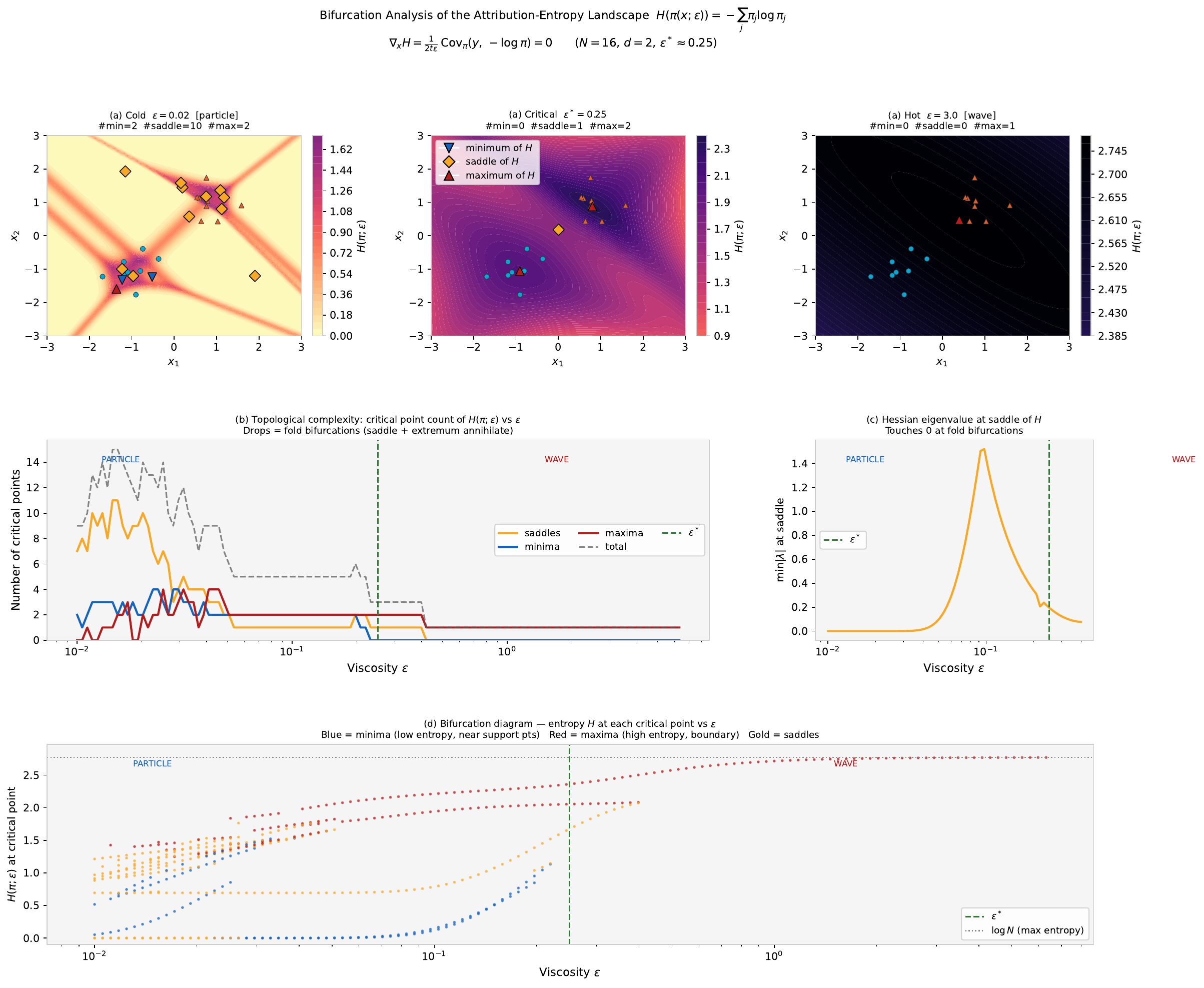}
\caption{Bifurcation analysis of the attribution-entropy landscape $H(\pi(x;\varepsilon))$ ($N=16$, $d=2$, two-cluster synthetic data).  \textbf{(a)}~Landscape at three viscosities with critical points: {\color[HTML]{1565C0}$\triangledown$}~minima (low entropy, near support points), {\color[HTML]{F9A825}$\diamond$}~saddles (attribution transitions), {\color[HTML]{B71C1C}$\triangle$}~maxima (high entropy, class boundaries).  \textbf{(b)}~Critical-point count vs.\ $\varepsilon$ (log scale); discrete drops are fold bifurcations.  \textbf{(c)}~Minimum Hessian eigenvalue at saddle points, approaching zero at each fold event.  \textbf{(d)}~Bifurcation diagram: $H$-value of each critical point vs.\ $\varepsilon$; branch annihilations are the fold bifurcations of Theorem~\ref{prop:bifurcation}.}
\label{fig:bifurcation_diagram}
\end{figure}

\begin{figure}[h]
\centering
\includegraphics[width=\linewidth]{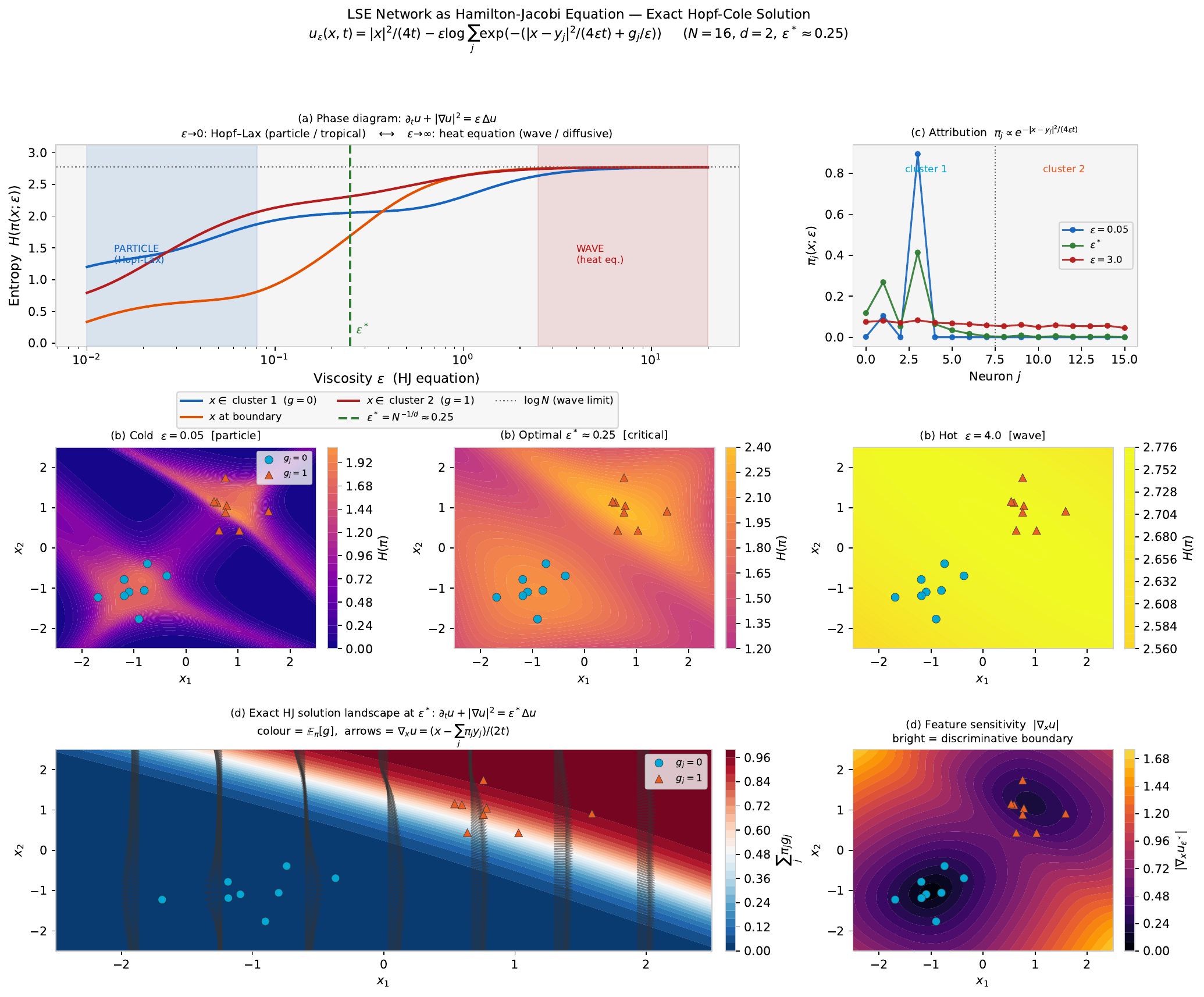}
\caption{Full phase diagram of the LSE network as Hopf--Cole solution ($N=16$, $d=2$, two-cluster synthetic data, $\varepsilon^*\approx 0.25$).  \textbf{(a)}~Attribution entropy $H(\pi(x;\varepsilon))$ vs.\ viscosity $\varepsilon$ (log scale) for three representative query points, showing the particle (Hopf--Lax, $\varepsilon\to 0$) to wave (heat equation, $\varepsilon\to\infty$) phase transition; $\varepsilon^*=N^{-1/d}$ is marked.  \textbf{(b)}~Spatial entropy landscape $H(\pi(x;\varepsilon))$ at cold ($\varepsilon=0.05$), optimal ($\varepsilon^*=0.25$), and hot ($\varepsilon=4.0$) viscosities; sharp attribution basins in the particle regime flatten to uniform entropy in the wave regime.  \textbf{(c)}~Attribution weights $\pi_j(x;\varepsilon)$ vs.\ neuron index at three viscosities, confirming concentration at low $\varepsilon$ and spreading at high $\varepsilon$.  \textbf{(d) left}~Exact HJ prediction landscape $\hat{f}_\varepsilon^N = \sum_j\pi_j g_j$ at $\varepsilon^*$, with flow arrows $\nabla_x u = (x-\sum_j\pi_j y_j)/(2t)$.  \textbf{(d) right}~Feature sensitivity $|\nabla_x u|$; bright regions are discriminative boundaries where attribution switches allegiance between clusters (cf.\ Theorem~\ref{prop:bifurcation}).}
\label{fig:phase_diagram}
\end{figure}

\begin{figure}[h]
\centering
\includegraphics[width=\linewidth]{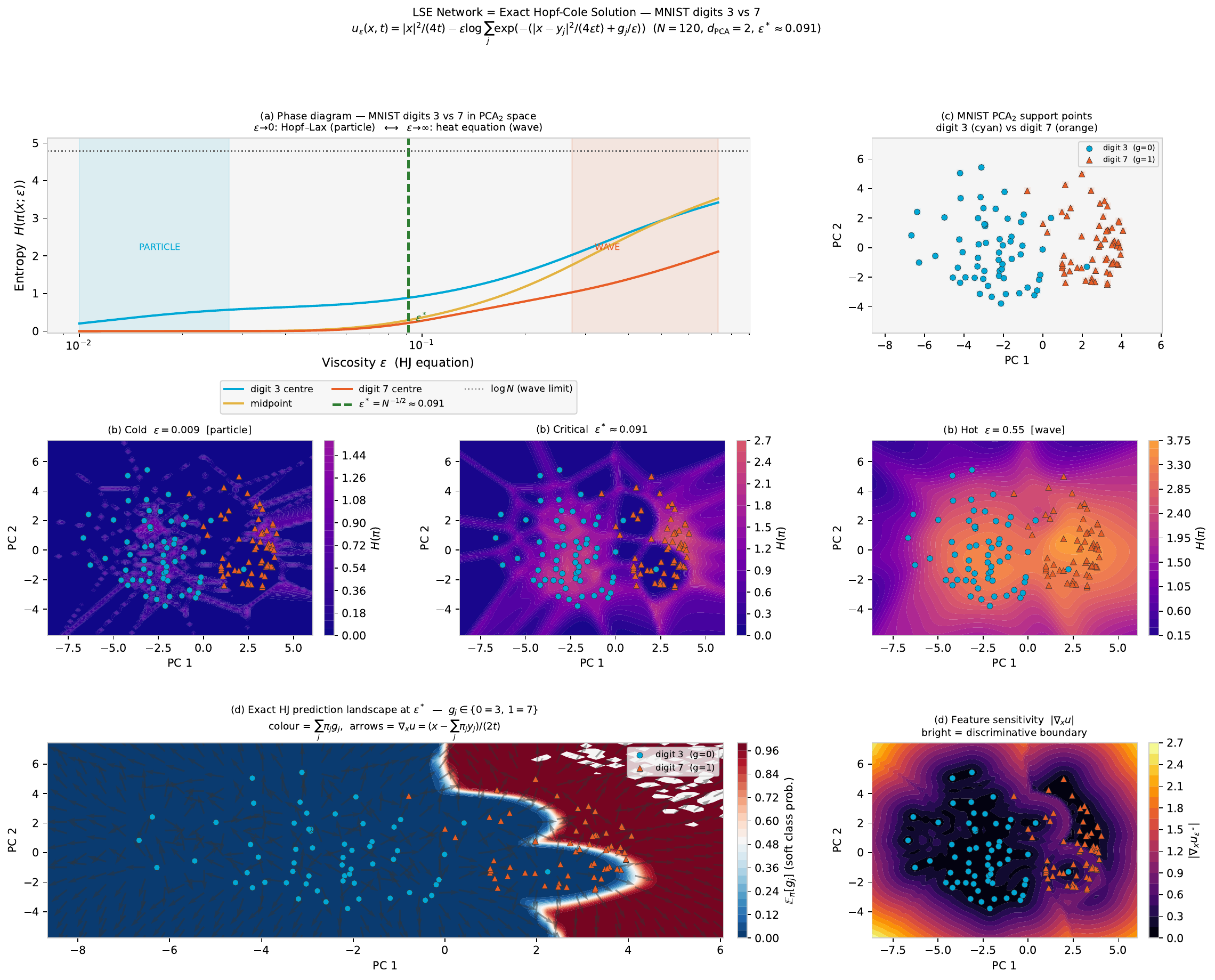}
\caption{Phase diagram of the LSE network as Hopf--Cole solution on real data: MNIST digits 3 vs.\ 7 ($N=120$, $d_{\mathrm{PCA}}=2$, $\varepsilon^*\approx 0.091$).  Layout identical to Figure~\ref{fig:phase_diagram}.  \textbf{(a)}~Entropy vs.\ viscosity for the digit-3 centroid, digit-7 centroid, and midpoint; the phase transition occurs at $\varepsilon^* = N^{-1/d} \approx 0.091$.  \textbf{(b)}~Spatial entropy landscape at cold, critical, and hot viscosities in PCA$_2$ space.  \textbf{(c)}~MNIST PCA$_2$ support points (digit 3 cyan, digit 7 orange).  \textbf{(d) left}~Exact HJ prediction landscape: the decision boundary follows the inter-cluster geometry in PCA space.  \textbf{(d) right}~Feature sensitivity $|\nabla_x u|$; discriminative regions concentrate along the digit-separation boundary, identifying the PCA directions most informative for the 3-vs-7 classification.}
\label{fig:mnist_phase_diagram}
\end{figure}

\clearpage
\section{Numerical Details}
\label{app:numerics}

All experiments run on a standard CPU; no GPU is required.
\paragraph{Shared defaults.} All trained networks use Adam ($\beta_1=0.9$, $\beta_2=0.999$, $\hat\varepsilon=10^{-8}$); weights initialized as $W_j \sim \mathcal{N}(0, 0.1^2/d)$, biases $b_j=0$; viscosity $\varepsilon = N^{-1/d}$ per Theorem~\ref{thm:gen} unless noted otherwise below (Figure~\ref{fig:genrate}'s quadrature-rate check uses a fixed $\varepsilon=t=1$ instead, the well-resolved regime of Theorem~\ref{thm:gen} step~(i); see the figure caption).
\paragraph{Experiment A} (initial-data recovery): Adam warm-start (5{,}000 steps, $\mathrm{lr}=0.008$) followed by L-BFGS-B; $N=10$ neurons, 8 seeds (best reported), 1-D synthetic $g^*(y)=|y|$ on $[-2.4,2.4]$, $\varepsilon\in\{0.5,0.1,0.04\}$.
\paragraph{Experiment B} (scaling law): Adam, 3{,}000 steps, $\mathrm{lr}=0.006$, 5 seeds (best reported), train/test 4{,}000/800, synthetic $g(y)=\|y\|_2$ on $[-2,2]^d$; $d=1$: $N\in\{10,25,50,100,200,500\}$; $d=2$: $N\in\{20,50,100,200,500\}$; $d=4$: $N\in\{40,100,200,500,1000\}$.
\paragraph{Experiment C} (robustness): Adam, 4{,}000 steps, $\mathrm{lr}=0.01$, $N=30$, 3 seeds (best reported), 1-D synthetic $g(y)=|y|$, $\varepsilon\in\{0.08,0.12,0.20,0.35,0.60,1.0\}$.
\paragraph{Experiment D} (attention identity): analytical verification only; no optimizer, seed 42 for data generation.
\paragraph{Experiment E} (MNIST Hessian): regression on digit label $\in[0,9]$ via linear output layer; MSE loss.  $N=128$, $d=50$ PCA features, Adam 5{,}000 steps, $\mathrm{lr}=5\times10^{-3}$, batch 256, 1 seed, train/test 60{,}000/10{,}000, 300 random test images, $\varepsilon\in\{0.1,0.3,1.0,3.0,10.0\}$.  Runtime $\approx$15 min.
\paragraph{Experiment F} (CIFAR-10 Hessian): regression on class label $\in[0,9]$ via linear output layer; MSE loss.  Same architecture as Experiment~E with $d=64$ PCA features, train/test 50{,}000/10{,}000.  Runtime $\approx$2 hr CPU ($<$10 min GPU).
\paragraph{Experiment G} (transformer block): analytical identity checks only; no training.  500 trials per $d\in\{4,8,16,32,64\}$, sequence length 8, $d_v=16$; seed 42.  Verifies that attention + LayerNorm equals $\nabla\mathrm{LSE}$ to floating-point precision and that each LSE-FFN layer satisfies Theorem~\ref{thm:nn_pde} to machine precision.
\paragraph{Experiments H--I} (phase diagram, bifurcation): analytical sweeps; no training.  Experiment~H uses $N=16$ synthetic support points in $d=2$ with two-cluster labels ($\varepsilon^*=N^{-1/2}=0.25$), 60 animation frames over $\varepsilon\in[10^{-2},10]$.  Experiment~I sweeps $\varepsilon$ over 200 values for bifurcation-of-entropy analysis.
MNIST \citep{lecun1998mnist} and CIFAR-10 \citep{krizhevsky2009cifar} are downloaded automatically from their canonical URLs; no manual data preparation is needed.

\begin{figure}[H]
  \centering
  \includegraphics[width=0.50\linewidth]{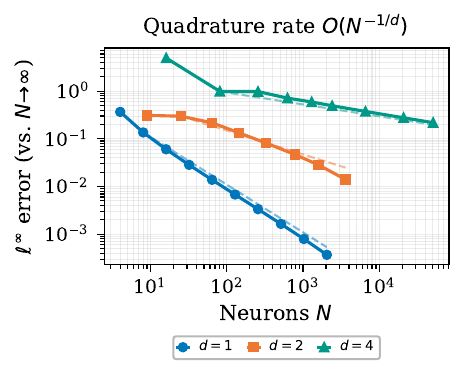}
  \caption{Quadrature convergence at fixed $(\varepsilon,t)=(1,1)$. $\ell^\infty$ error vs.\ width $N$ for initial data $g(y)=|y|$ (Lipschitz) across $d\in\{1,2,4\}$; dashed: $O(N^{-1/d})$ slopes confirming Theorem~\ref{thm:gen} step~(i) in this well-resolved regime, where the grid spacing $h$ is finer than the Gibbs concentration scale $\sqrt{2\varepsilon t}$. The generalization-gauge regime $t\asymp\varepsilon\asymp N^{-1/d}$ underlying the theorem's headline bound is treated analytically in Appendix~\ref{app:proofs}.}
  \label{fig:genrate}
\end{figure}

\begin{figure}[H]
  \centering
  \includegraphics[width=0.50\linewidth]{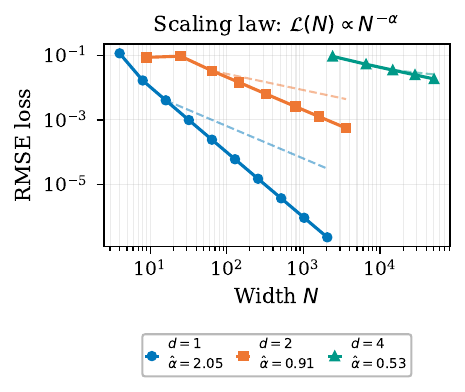}
  \caption{Scaling law $\mathcal{L}(N)\propto N^{-\alpha}$ at optimal $\varepsilon^*$; $g(y)=\tfrac{1}{2}|y|^2$. Fitted $\hat\alpha>1/d_{\mathrm{eff}}$ confirms smooth $g$ exceeds the Lipschitz rate (closed-form networks).}
  \label{fig:scaling}
\end{figure}

\begin{figure}[H]
  \centering
  \includegraphics[width=0.50\linewidth]{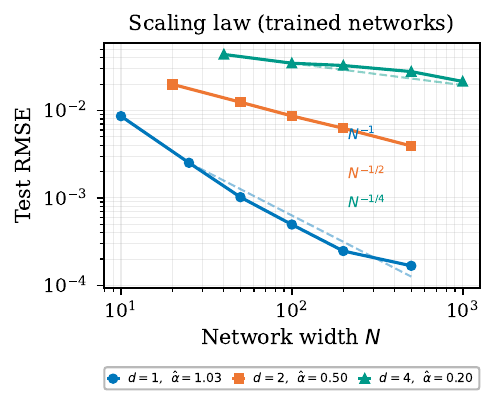}
  \caption{Scaling law for Adam-trained LSE networks.  Test RMSE
    vs.\ width $N$ for target $g(y)=\|y\|$ across $d\in\{1,2,4\}$;
    dashed: predicted $N^{-1/d}$ slopes.  Fitted exponents
    $\hat\alpha\approx 1/d$ confirm Proposition~\ref{prop:scaling}
    for trained networks.}
  \label{fig:scaling_adam}
\end{figure}

\begin{figure}[H]
  \centering
  \includegraphics[width=0.48\linewidth]{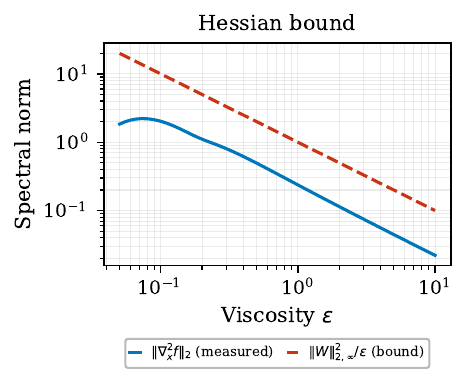}
  \hfill
  \includegraphics[width=0.48\linewidth]{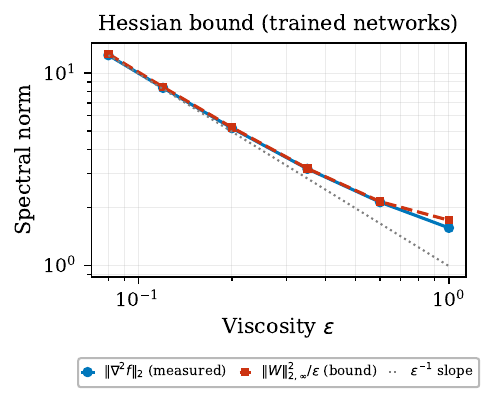}
  \caption{Hessian bound (Corollary~\ref{cor:robust}).
    \emph{Left}: closed-form network; measured norm (solid) vs.\
    bound $\|W\|_{2,\infty}^2/\varepsilon$ (dashed),
    $\varepsilon\in[0.05,10]$.  \emph{Right}: SGD-trained network
    on $g(y)=|y|$; bound never violated across
    $\varepsilon\in[0.08,1.0]$.}
  \label{fig:robust}
\end{figure}

\begin{table}[H]
\centering
\caption{Theorem~\ref{thm:nn_pde} identity verified to floating-point
  precision for a 4-neuron, 1-D network.}
\label{tab:verify}
\small
\setlength{\tabcolsep}{10pt}
\begin{tabular}{cc}
\toprule
$\varepsilon$ & $\max_x\,|f_\varepsilon^N(x) + u_\varepsilon^N(x,t) - |x|^2/(4t)|$ \\
\midrule
1.00 & $4.4 \times 10^{-16}$ \\
0.50 & $2.2 \times 10^{-16}$ \\
0.20 & $4.4 \times 10^{-16}$ \\
0.10 & $4.4 \times 10^{-16}$ \\
0.05 & $2.2 \times 10^{-16}$ \\
\bottomrule
\end{tabular}
\end{table}

\begin{table}[H]
\centering
\caption{Transformer attention identity \eqref{eq:attn_gibbs} verified to exact zero error.  Let $A := \mathrm{Attn}(Q,K,V)$ (standard scaled dot-product) and $B := (\nabla_z \mathrm{LSE}_\varepsilon(z))\big|_{z_j=q_i\cdot k_j}\cdot V$ (Hopf--Cole form, \eqref{eq:attn_gibbs}) with $\varepsilon=\sqrt{d}$.  Max absolute error over 500 random $(Q,K,V)$ trials, $n_q=8$, $n_k=12$, $d_v=16$.}
\label{tab:verify_attn}
\small
\setlength{\tabcolsep}{10pt}
\begin{tabular}{ccc}
\toprule
$d$ & $\varepsilon=\sqrt{d}$ & $\max|A - B|$ \\
\midrule
4  & 2.00 & $0$ \\
8  & 2.83 & $0$ \\
16 & 4.00 & $0$ \\
32 & 5.66 & $0$ \\
64 & 8.00 & $0$ \\
\bottomrule
\end{tabular}
\end{table}

\begin{table}[H]
\centering
\caption{LSE-transformer block characterization (Proposition~\ref{prop:lse_transformer}) verified numerically.  \emph{Left}: attention identity~\eqref{eq:attn_gibbs} after LayerNorm, verified over 500 random trials per $d$; $n_q{=}8$, $n_k{=}8$, $d_v{=}16$.  \emph{Right}: LSE-FFN HJ identity $|f_\varepsilon + u_\varepsilon - |x|^2/(4t)|$ (Theorem~\ref{thm:nn_pde}) with $x = \mathrm{LN}(z)$, over 500 trials per $\varepsilon$; $N{=}32$, $d{=}8$.  All errors are floating-point roundoff.}
\label{tab:verify_transformer}
\small
\setlength{\tabcolsep}{8pt}
\begin{tabular}{cccccc}
\toprule
\multicolumn{3}{c}{Attention $+$ LayerNorm} & \multicolumn{3}{c}{LSE-FFN (Theorem~\ref{thm:nn_pde})} \\
\cmidrule(lr){1-3}\cmidrule(lr){4-6}
$d$ & $\varepsilon{=}\sqrt{d}$ & $\max|A{-}B|$ & $\varepsilon$ & $N$ & $\max|f{+}u{-}|x|^2/(4t)|$ \\
\midrule
4  & 2.00 & $0$ & 0.05 & 32 & $8.9{\times}10^{-16}$ \\
8  & 2.83 & $0$ & 0.10 & 32 & $8.9{\times}10^{-16}$ \\
16 & 4.00 & $0$ & 0.20 & 32 & $8.9{\times}10^{-16}$ \\
32 & 5.66 & $0$ & 0.50 & 32 & $8.9{\times}10^{-16}$ \\
64 & 8.00 & $0$ & 1.00 & 32 & $8.9{\times}10^{-16}$ \\
\bottomrule
\end{tabular}
\end{table}

\begin{figure}[H]
\centering
\includegraphics[width=0.45\linewidth]{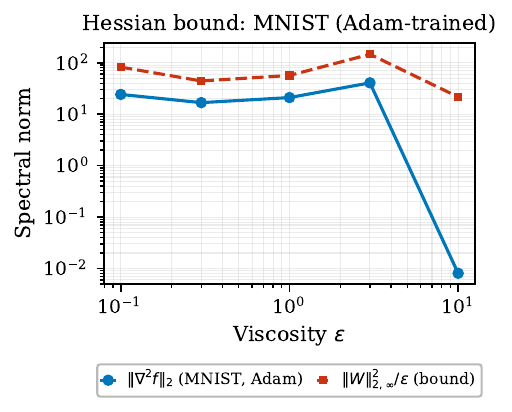}\hfill
\includegraphics[width=0.45\linewidth]{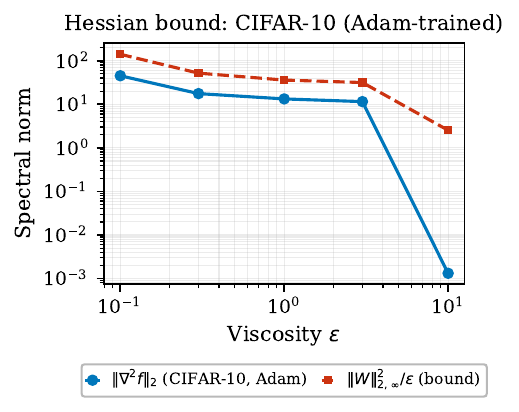}
\caption{Hessian bound (Corollary~\ref{cor:robust}) on MNIST (\emph{left}) and CIFAR-10 (\emph{right}). Each network has $N=128$ neurons; inputs are PCA-projected to $d=50$ (MNIST) and $d=64$ (CIFAR-10). Measured spectral norm $\|\nabla^2_x f\|_2$ (diamonds) and theoretical bound $\|W\|_{2,\infty}^2/\varepsilon$ (dashed squares) on 300 random test images. The bound is never violated across all $\varepsilon\in\{0.1,0.3,1.0,3.0,10.0\}$ on both datasets. Note that PCA projection constitutes a significant dimensionality reduction from raw pixels; verification on raw or learned representations remains open.}
\label{fig:hessian_mnist}
\end{figure}

% ============================================================
\newpage
\section{Scope of the Correspondence: Universality, Closure, and Extensions}
\label{app:universality}

\paragraph{The $\varepsilon$ deformation dictionary.}
Table~\ref{tab:eps_dictionary} records the four-perspective interpretation of the deformation parameter $\varepsilon$ across the neural network, algebraic, PDE, and convex optimization viewpoints (Section~\ref{sec:unification}).

\begin{table}[h]
\centering
\caption{Dictionary for the deformation parameter $\varepsilon$: four perspectives on the same LSE layer.}
\label{tab:eps_dictionary}
\smallskip
\resizebox{\linewidth}{!}{%
\begin{tabular}{lllll}
\toprule
\textbf{Regime} & \textbf{Network} & \textbf{Algebra} & \textbf{PDE} & \textbf{Optimization} \\
\midrule
$\varepsilon > 0$
  & LSE layer, softmax weights $\pi_j$
  & $(\mathbb{R},+,\times)$
  & Viscous HJ $u_\varepsilon$ (Hopf--Cole)
  & Entropy-regularized convex program \\
$\varepsilon \to 0$
  & Max/ReLU layer, one-hot weights
  & $(\mathbb{R},\max,+)$ tropical
  & Inviscid HJ $u_0$ (Hopf--Lax)
  & Linear program at a vertex \\
$\varepsilon \to \infty$
  & Uniform softmax, mean pooling
  & —
  & Heat equation ($u_\infty = $ convolution)
  & Unconstrained regularization \\
$\varepsilon^* \asymp N^{-1/d}$
  & Optimal temperature at width $N$
  & —
  & Grid-matched viscosity
  & Regularization matched to width \\
\bottomrule
\end{tabular}}
\end{table}

The rows correspond to the four corners and the optimal operating point of diagram~\eqref{eq:diagram}. The deformation $\varepsilon$ is formally identical to $\hbar$ in Maslov's dequantization: the passage $\varepsilon \to 0$ is the neural-network analogue of the $\hbar \to 0$ classical limit \citep{feynman1982}. The LSE network at $\varepsilon > 0$ is a universal classical HJ simulator (Theorem~\ref{thm:gen}), the positive-measure counterpart of the universal quantum simulator of \citet{feynman1982}; classical realizability holds precisely because the Gibbs measure is positive for all $\varepsilon > 0$, unlike the Wigner quasi-distribution of quantum mechanics.

\paragraph{Universality.}
The exact correspondence (Theorem~\ref{thm:nn_pde}) applies to
LSE-activated networks.  The following proposition shows the
class is large enough to be universal.

\begin{proposition}[LSE universality]
\label{prop:universality}
Let $K \subset \mathbb{R}^d$ be compact and $f : K \to \mathbb{R}$
continuous.  For every $\delta > 0$ there exists a deep LSE
network $f_\varepsilon$ such that
$\sup_{x \in K} |f(x) - f_\varepsilon(x)| < \delta$.
\end{proposition}

\begin{proof}
A single LSE layer $\mathrm{LSE}_\varepsilon(Wx+b)$ is convex in $x$,
so depth $\geq 2$ is required for non-convex targets.  Any continuous
$f$ on compact $K$ is uniformly approximable by a continuous
piecewise-linear function (Stone--Weierstrass on $K$).  Every
continuous piecewise-linear function can be written in max-of-min-of-affine
form, which a depth-two tropical (max-plus/min-plus) network represents
exactly; this follows from the lattice form of the Stone--Weierstrass
theorem \citep{litvinov2007}, since max-plus/min-plus functions
separate points and form a function lattice.  For fixed $\varepsilon > 0$,
each max is approximated by $\mathrm{LSE}_\varepsilon$ to within
$O(\varepsilon \log N)$ (Theorem~\ref{thm:maslov}), and each min by
$-\varepsilon\log\sum_j e^{-z_j/\varepsilon}$ to the same order.
Choosing $\varepsilon$ and $N$ appropriately gives the result.
\end{proof}

\noindent
Together, Theorem~\ref{thm:nn_pde} and Proposition~\ref{prop:universality}
imply: \emph{for any continuous function there exists a deep viscous
HJ solution under discrete measures that approximates it to within
$\delta$}.  The HJ class is therefore as general as the class of
continuous functions (with depth $\geq 2$).

\paragraph{Closure.}
\begin{proposition}[Closure of the HJ class]
\label{prop:closure}
The class of functions representable as Hopf--Cole solutions of
viscous HJ equations under discrete measures is closed under:
\begin{enumerate}[label=(\roman*),leftmargin=*,nosep]
\item \emph{Layer composition}: composing $L$ LSE layers corresponds
  to applying the heat semigroup for total time $T = \sum_\ell t_\ell$
  (Theorem~\ref{thm:diagram}); the result is another HJ solution.
\item \emph{Affine input transformations}: $f_\varepsilon(Ax + c)$ is
  an LSE network with weights $WA$ and biases $Wb + c$, hence still
  an exact HJ solution.
\item \emph{Residual connections}: adding a skip path corresponds to
  the Euler discretization of HJ characteristics
  (Proposition~\ref{prop:resnet}); the composition remains within
  the HJ class up to discretization error $O(h)$.
\end{enumerate}
\end{proposition}

\begin{proof}
By Theorem~\ref{thm:nn_pde}, each LSE layer satisfies $f_\varepsilon^{(l)}(x) = |x|^2/(4t_l) - u_\varepsilon^{(l)}(x,t_l)$ exactly, with $u_\varepsilon^{(l)}$ the Hopf--Cole solution for that layer's own initial datum $g^{(l)}$ and time $t_l$.  Each layer individually is an exact HJ solution; the composition is another HJ-class function, as established by Theorem~\ref{thm:diagram}(i).

The substitution $z = Ax+c$ gives $f_\varepsilon(Ax+c) = \mathrm{LSE}_\varepsilon((WA)x+(Wc+b))$, an LSE layer with weights $W'=WA$ and biases $b'=Wc+b$.  Applying Theorem~\ref{thm:nn_pde} to $(W',b')$ yields the exact HJ identity with reparameterized support points $y'_j = 2t(W'_j)^\top$ and initial data $g'_j = -b'_j - |y'_j|^2/(4t)$.

The residual layer $x_{l+1} = x_l + h\,f_\varepsilon^{(l)}(x_l)$ is the forward Euler step for $\dot{x} = f_\varepsilon(x)$; by Proposition~\ref{prop:resnet} this ODE is the $x$-characteristic of the HJ PDE with Hamiltonian $H(x,p) = p\cdot f_\varepsilon(x)$, so the composition stays within the HJ class up to $O(h)$ truncation error. \qed
\end{proof}

\paragraph{Extension to broader PDE classes.}
When the Hamiltonian admits a Hopf--Lax or inf-convolution
representation (convex $H$), the correspondence extends to scalar
conservation laws, eikonal equations, and linear transport; in each
case the network encodes the Hamiltonian in its weights and the
viscosity in $\varepsilon$.  The KdV and Toda lattice hierarchies
ultradiscretize to box-ball cellular automata \citep{tokihiro1996},
suggesting a connection for non-convex Hamiltonians, but the neural
network analogue in that setting remains open.

\begin{remark}[Non-convex Hamiltonians: obstructions and partial results]
\label{rem:nonconvex_h}
Convexity of $H$ enters the framework at two algebraic points.
First, the Hopf--Cole substitution $v = e^{-u/\varepsilon}$ linearizes $\partial_t u + H(\nabla u) = \varepsilon\Delta u$ to the heat equation for $v$ only when $H$ is quadratic (Theorem~\ref{thm:nn_pde}) or a learnable quadratic form (Theorem~\ref{thm:varH}); for a general non-quadratic $H$, the substitution introduces a nonlinear reaction term that does not vanish, so the exact algebraic identity $f_\varepsilon^N + u_\varepsilon = |x|^2/(4t)$ breaks.
Second, the Legendre--Fenchel conjugate $L = H^*$ is always convex (as a pointwise supremum of affine functions), so the Hopf--Lax formula always produces a well-defined viscosity solution, but for the convex envelope $H^{**} = (H^*)^*$, not for the original non-convex $H$; the two solutions differ by $O(t\cdot\|H-H^{**}\|_\infty)$ (Theorem~\ref{prop:nonconvex_ultra}(ii)).

Two partial results extend the framework beyond the quadratic case.
For any continuous $H$, classical parabolic theory guarantees a unique smooth solution for each $\varepsilon > 0$, and Crandall--Lions theory \citep{crandall1983} establishes convergence to the unique viscosity solution of the inviscid problem as $\varepsilon \to 0$; LSE networks therefore provide a well-posed computational approximation to non-convex HJ solutions via viscous regularization, even without an exact algebraic identity.
For small non-quadratic perturbations, write $H_\delta(p) = |p|^2/(4t) + \delta K(p)$ with $K \in L^\infty(\mathbb{R}^d)$ and let $u_\varepsilon^\delta$ denote the corresponding viscous solution.
A first-order expansion around $\delta = 0$, combined with the parabolic maximum principle applied to the linearized correction $u_1$ (which satisfies a linear parabolic PDE with forcing $-K(\nabla u_\varepsilon^0)$ and zero initial data, giving $\|u_1(\cdot,t)\|_\infty \leq t\|K\|_\infty$), yields
\begin{equation}
\bigl|f_\varepsilon^N(x) + u_\varepsilon^\delta(x,t) - |x|^2/(4t)\bigr| \;\leq\; \delta\,t\,\|K\|_{L^\infty} + O(\delta^2),
\label{eq:nonconvex_approx}
\end{equation}
showing that the identity fails at rate linear in both the degree of non-convexity $\delta$ and the evaluation time $t$.
Theorem~\ref{thm:maximal} provides a complete characterization of which $H$ admit an exact identity; Theorem~\ref{prop:nonconvex_ultra} provides a constructive resolution for the non-convex case via an LSE-type kernel network that is exact for $H^{**}$.
\end{remark}

\begin{proof}[Proof of Theorem~\ref{thm:maximal}]
The Hopf--Cole substitution $v = e^{-u/\varepsilon}$ applied to $\partial_t u + H(\nabla u) = \varepsilon\,\nabla\!\cdot\!(A\nabla u)$ yields
\[
\partial_t v \;=\; \varepsilon\,\nabla\!\cdot\!(A\nabla v) \;+\; v\cdot\bigl[H(-\varepsilon\nabla\log v)/\varepsilon - \varepsilon(\nabla\log v)^\top A(\nabla\log v)\bigr].
\]
The reaction term (the bracketed expression) vanishes identically if and only if $H(p) = p^\top A p$: the condition $\varepsilon(\nabla\log v)^\top A(\nabla\log v) = H(-\varepsilon\nabla\log v)/\varepsilon$ must hold for all gradients, which forces $H(q) = q^\top Aq$ for all $q \in \mathbb{R}^d$.
This is the anisotropic-quadratic class of Theorem~\ref{thm:varH}, recovering Theorem~\ref{thm:nn_pde} at $A = I$.

For the BSDE direction: for non-quadratic $H$, the solution has an exact stochastic representation via the nonlinear Feynman--Kac theorem: $u(x,t) = Y_0$, where $(Y_s, Z_s)$ satisfies $-dY_s = H(Z_s)\,ds - Z_s\cdot dW_s$, $Y_T = g(X_T)$, with $Z_s = \nabla u(X_s, s)$.
The BSDE reduces to an explicit Gaussian convolution precisely when the driver $H(Z_s)$ is quadratic in $Z_s$, providing an independent stochastic-representation confirmation that the quadratic class is necessary. \qed
\end{proof}

\begin{remark}[BSDE alternative and constructive resolution]
For all other $H$, an exact stochastic representation requires a BSDE-type particle-path architecture \citep{han2018}, which is qualitatively different from an LSE network.
The constructive resolution for the non-convex case, an LSE-type kernel network exact for the convex envelope $H^{**}$, with approximation error $O(t\cdot\|H-H^{**}\|_\infty)$ controlled by the convexity gap, is given in Theorem~\ref{prop:nonconvex_ultra}.
\end{remark}

\begin{theorem}[Ultradiscretization extension to non-convex $H$]
\label{prop:nonconvex_ultra}
For any $H:\mathbb{R}^d\to\mathbb{R}$ (possibly non-convex), let $L = H^*$ be its Legendre--Fenchel conjugate (always convex, as a pointwise supremum of affine functions) and $H^{**} = (H^*)^*$ the convex envelope of $H$.
Define the kernel network
\begin{equation}
K_\varepsilon^N(x) \;=\; -\varepsilon\log\sum_{j=1}^N \exp\!\left(-\frac{tL\!\bigl(\frac{x-y_j}{t}\bigr) + g(y_j)}{\varepsilon}\right).
\label{eq:kernel_net}
\end{equation}
Then:
\begin{enumerate}[label=(\roman*),leftmargin=*,nosep]
\item \emph{Tropical limit, exact for $H^{**}$:} As $\varepsilon\to 0$, by Theorem~\ref{thm:maslov},
$K_\varepsilon^N(x) \to \min_{j}\{tL((x-y_j)/t) + g(y_j)\}$,
which is the $N$-point Hopf--Lax approximation for $H^{**}$, exact as $N\to\infty$ by Theorem~\ref{thm:hopflax}.
\item \emph{Approximation error for the original $H$:} By stability of viscosity solutions under Hamiltonian perturbation \citep{crandall1983},
$\|u_{H^{**}}(\cdot,t) - u_H(\cdot,t)\|_\infty \leq t\cdot\sup_p|H(p) - H^{**}(p)|$;
the bound vanishes when $H$ is convex ($H = H^{**}$).
\item \emph{Reduction to the LSE identity:} When $H(p)=|p|^2$, so $L(v)=|v|^2/4$, completing the square gives $tL((x-y_j)/t) = |x|^2/(4t) - x\cdot y_j/(2t) + |y_j|^2/(4t)$, which separates $x$ from $y_j$ and reduces $K_\varepsilon^N$ to the standard LSE network of Theorem~\ref{thm:nn_pde} with exact identity $K_\varepsilon^N(x) + u_\varepsilon^N(x,t) = |x|^2/(4t)$ for all $\varepsilon>0$ and $N\geq 1$.
\end{enumerate}
\end{theorem}

\begin{proof}
By Theorem~\ref{thm:maslov} (Maslov dequantization), as $\varepsilon\to 0$ the expression $-\varepsilon\log\sum_j\exp(-f_j/\varepsilon)$ converges to $\min_j f_j$; the limit equals the $N$-point Hopf--Lax formula for $H^{**}$ because $\inf_y\{tL((x-y)/t)+g(y)\}$ is the viscosity solution of $\partial_t u+H^{**}(\nabla u)=0$ with $u(\cdot,0)=g$ (Theorem~\ref{thm:hopflax}, since $L=(H^{**})^*$ and $(H^{**})^{**}=H^{**}$).
If $w_H$ and $w_{H^{**}}$ are viscosity solutions of the same initial-value problem with Hamiltonians $H$ and $H^{**}$ respectively, stability of viscosity solutions under Hamiltonian perturbation gives $\|w_H-w_{H^{**}}\|_\infty\leq t\sup_p|H(p)-H^{**}(p)|$ \citep{crandall1983}.
When $H(p)=|p|^2$, the Lagrangian $L = H^* = |\cdot|^2/4$ gives $tL((x-y_j)/t)=|x-y_j|^2/(4t)$; completing the square yields $|x-y_j|^2/(4t)=|x|^2/(4t)-x\cdot y_j/(2t)+|y_j|^2/(4t)$; the $x$-dependent prefactor $e^{|x|^2/(4\varepsilon t)}$ factors out of the sum, giving $K_\varepsilon^N(x)=|x|^2/(4t)-\varepsilon\log\sum_j\exp(-(|x-y_j|^2/(4t)+g(y_j))/\varepsilon)=|x|^2/(4t)-u_\varepsilon^N(x,t)$, which is Theorem~\ref{thm:nn_pde}. \qed
\end{proof}

\noindent The construction applies the re-quantization step of ultradiscretization in reverse (Theorem~\ref{thm:maslov}): the min-plus tropical Hopf--Lax formula for $H^{**}$ is lifted to the smooth differentiable network $K_\varepsilon^N$ by replacing $\min_j$ with $-\varepsilon\log\sum_j\exp(-\cdot/\varepsilon)$.
For non-convex $H$, the kernel $K(x,y,t)=\exp(-tL((x-y)/t)/\varepsilon)$ is the $H^{**}$-adjusted heat kernel; the exact algebraic identity holds for $H^{**}$ and reduces to $K_\varepsilon^N+u_\varepsilon^N=|x|^2/(4t)$ only when $H^{**}$ is the Euclidean quadratic.
The approximation error for the original non-convex $H$ is $O(t\cdot\|H-H^{**}\|_\infty)$, zero for convex $H$, and controlled by the convexity gap otherwise.

\begin{remark}[Depth and non-convex function classes]
\label{rem:deep_nonconvex}
The kernel network $K_\varepsilon^N$ of Theorem~\ref{prop:nonconvex_ultra} is convex in $x$: since $L = H^*$ is convex, the map $x\mapsto tL((x-y_j)/t)$ is convex for each $j$, and log-sum-exp preserves convexity.
Single-layer networks are therefore limited to representing solutions that are convex functions of the input, which excludes non-convex HJ solutions.
In practice, LSE networks have depth $L \geq 2$, which is required for universal approximation of non-convex targets (Proposition~\ref{prop:universality}).
Applying the kernel construction of Theorem~\ref{prop:nonconvex_ultra} layer-by-layer via the multi-step HJ semigroup (Theorem~\ref{thm:diagram}), with per-layer time steps $t_\ell = t/L$, yields a deep LSE network in which each layer is individually convex in its own input but the composition over $L \geq 2$ layers is not globally convex, recovering the non-convex function class required for non-convex HJ solutions.
The per-layer approximation error is $O(t_\ell\cdot\|H-H^{**}\|_\infty)$; after $L$ compositions the cumulative error is $O(t\cdot\|H-H^{**}\|_\infty)$, the same bound as in the single-layer case and independent of depth.
Depth therefore enlarges the representable function class without reducing the convexification error; the gap $\sup_p|H(p)-H^{**}(p)|$ is an intrinsic property of $H$ that depth alone cannot close.
\end{remark}

\begin{remark}[LSE networks as heat-kernel quadrature]
\label{rem:greens}
The heat kernel $G_\varepsilon(x,y,t) = (4\pi\varepsilon t)^{-d/2}\exp\!\bigl(-|x-y|^2/(4\varepsilon t)\bigr)$ is the Green's function of $\partial_t u = \varepsilon\Delta u$.  The Hopf--Cole solution admits the integral representation
\[
  u_\varepsilon(x,t) \;=\; -\varepsilon\log\!\int G_\varepsilon(x,y,t)\,e^{-g(y)/\varepsilon}\,dy,
\]
so the LSE network $f_\varepsilon(x) = \varepsilon\log\sum_j\exp\bigl((w_j\cdot x + b_j)/\varepsilon\bigr)$ is precisely a quadrature of this integral: each neuron $j$ evaluates the heat kernel at the support point $y_j$, with learned weights encoding the initial data $g(y_j)$ via $b_j = -g(y_j) - |y_j|^2/(4t)$ (Theorem~\ref{thm:nn_pde}).  Three consequences follow.  First, the Gaussian kernel used in kernel methods \citep{jacot2018ntk} and in the NTK literature is the heat kernel at unit time; LSE networks are therefore kernel machines whose kernel is the fundamental solution of the heat equation.  Second, the $N\to\infty$ limit (Theorem~\ref{thm:diagram}(iii)) recovers the continuum Green's function integral, connecting neural scaling to the convergence of quadrature rules for the fundamental solution.  Third, the tropical limit ($\varepsilon\to 0$) replaces the heat kernel with the Dirac mass $\delta_{y^*(x)}$ concentrated at the minimizer $y^*(x) = \arg\min_j\{|x-y_j|^2/(4t) + g(y_j)\}$ (Varadhan's lemma, \citep{varadhan1984}), recovering the Hopf--Lax formula and connecting to nearest-neighbor / Voronoi classifiers.
\end{remark}

\paragraph{Convolutional architectures.}

\begin{remark}[Convolutional LSE layers]
\label{rem:cnn}
A convolutional LSE layer applies the same weight matrix
$W\in\mathbb{R}^{N\times K}$ to every local receptive field
$\tilde{x}_i = [x[i],\,x[i-1],\ldots,x[i-K+1]]\in\mathbb{R}^K$:
\[
  f_i = \mathrm{LSE}_\varepsilon(W\tilde{x}_i + b),
\]
which by Theorem~\ref{thm:nn_pde} satisfies the exact identity
$f_i = |\tilde{x}_i|^2/(4t) - u_\varepsilon^N(\tilde{x}_i,\,t)$
at each position $i$.  The filter weights encode the same support
points $\{y_j\}_{j=1}^N$ and initial data $\{g(y_j)\}$ for all
positions; the identical HJ equation is solved at every receptive
field.  Translation equivariance corresponds to shift-invariance of
the initial-data measure: the filter defines one viscous HJ equation
whose Hopf--Cole solution is evaluated position-wise.  By Noether's theorem, the position-independence of $H(p)=|p|^2$ implies that the co-state $p=\nabla_x u$ is conserved along each characteristic ($\dot{p}=-\nabla_x H=0$); parameter sharing in CNNs corresponds to imposing this conservation globally, constraining gradients to be the same filter across all spatial positions.  The exact
correspondence therefore extends to convolutional LSE architectures
without modification.
\end{remark}

\paragraph{General activations: correspondence map.}
Networks with GELU or SiLU activations are not covered by
Theorem~\ref{thm:nn_pde}, but a partial correspondence can be
identified at each level.

\emph{SiLU.}  The SiLU activation is $\mathrm{SiLU}(x) = x\cdot\sigma(x)$,
where $\sigma(x) = e^x/(e^x + 1)$ is the softmax of two logits
$(x, 0)$, a 2-neuron special case of the LSE gradient.  Since
$\nabla_x \mathrm{LSE}_\varepsilon(x,0) = \sigma(x/\varepsilon)$,
SiLU sits one derivative away from the exact HJ class.  In the
tropical limit $\sigma(x/\varepsilon)\to\mathrm{Heaviside}(x)$ as
$\varepsilon\to 0$, recovering ReLU.

\emph{GELU.}  The GELU activation is $\mathrm{GELU}(x) = x\cdot\Phi(x)$,
where $\Phi$ is the standard Gaussian CDF.  Since $\Phi(x) =
\int_{-\infty}^x \phi(t)\,dt$ and $\phi$ is the heat kernel at time
$t=\frac{1}{2}$, GELU involves a Gaussian convolution, placing it
in the same family as the imaginary-time propagator of
Appendix~\ref{app:physical}.  In the limit $\Phi(x/\varepsilon)\to
\mathrm{Heaviside}(x)$ as $\varepsilon\to 0$, again recovering ReLU.

Both activations therefore share the same tropical limit and the same
Gibbs/heat-kernel structure at finite temperature as the LSE class.
Each in fact admits an exact algebraic identity, in an exact class
adjacent to the solution- and gradient-classes: $\mathrm{SiLU}$ is an
exact Gibbs moment (the weight-gradient of $\mathrm{LSE}_\varepsilon$),
and $\mathrm{GELU}$ is an exact heat flow of a signed datum, each with a
proved obstruction placing it outside every scalar solution, gradient,
or heat class (Appendix~\ref{app:silu_gelu}).

Table~\ref{tab:activations} summarises the correspondence status for
common activations.

\begin{table}[H]
\centering
\caption{HJ correspondence status for common activations.
  \textbf{Yes (solution-class)}: activation IS the Hopf--Cole solution
  (Theorem~\ref{thm:nn_pde}).
  \textbf{Yes ($\nabla$-class)}: activation is an exact component of the
  Hopf--Cole measure (the signed Gibbs weight $\nabla\,\mathrm{LSE}$),
  an exact algebraic identity, not an approximation.
  \textbf{Yes (moment/measure class)}: exact identity in a class adjacent
  to the solution and gradient classes (Appendix~\ref{app:silu_gelu}).}
\label{tab:activations}
\small
\setlength{\tabcolsep}{4pt}
\begin{tabular}{llp{2.8cm}p{3.0cm}}
\toprule
Activation & Tropical limit ($\varepsilon\to 0$) & Finite-$\varepsilon$ structure & Exact HJ identity \\
\midrule
LSE & $\max_j(W_j\cdot x+b_j)$ (Hopf--Lax) & Hopf--Cole solution & \textbf{Yes} (Theorem~\ref{thm:nn_pde}) \\
ReLU & $\max(x,0)$ & 2-neuron LSE special case & \textbf{Yes} (special case) \\
Softplus & $\max(x,0)$ & $\log(1+e^x)=\mathrm{LSE}_1(x,0)$ & \textbf{Yes} ($N=1$ case) \\
Sigmoid & Heaviside & $\nabla_x\mathrm{LSE}_1(x,0)$ & \textbf{Yes} ($\nabla$-class, exact) \\
SiLU & ReLU & $x\cdot\nabla_x\mathrm{LSE}_1(x,0)$ & \textbf{Yes} (moment class, App.~\ref{app:silu_gelu}) \\
GELU & ReLU & $x\cdot$(heat kernel integral) & \textbf{Yes} (measure class, App.~\ref{app:silu_gelu}) \\
tanh & $\mathrm{sign}(x)$ & $\pi_+-\pi_-$ (signed Gibbs weight) & \textbf{Yes} ($\nabla$-class, exact) \\
\bottomrule
\end{tabular}
\end{table}

\paragraph{LSE-transformer blocks.}
A standard pre-norm transformer block with LSE-activated FFN (replacing GELU or SiLU) consists of four components: (1) LayerNorm, (2) scaled dot-product attention with residual, (3) LayerNorm, (4) two-layer LSE-FFN with residual.

\begin{proposition}[LSE-transformer block characterization]
\label{prop:lse_transformer}
Let $X \in \mathbb{R}^{L \times d}$ be a sequence of $L$ token embeddings.  Define the pre-norm LSE-transformer block:
\begin{align*}
Z_1 &= \mathrm{LN}(X), \quad Z_2 = X + \mathrm{Attn}(Z_1 W_Q,\, Z_1 W_K,\, Z_1 W_V), \\
Z_{\mathrm{out}} &= Z_2 + \mathrm{LSE}_{\varepsilon_2}(W_2\cdot\mathrm{LSE}_{\varepsilon_1}(W_1\,\mathrm{LN}(Z_2) + b_1) + b_2),
\end{align*}
where $\mathrm{LN}$ is layer normalization.  Then:
\begin{enumerate}[label=(\roman*),leftmargin=*,nosep]
\item \emph{Attention (exact):} identity~\eqref{eq:attn_gibbs} holds for any inputs $Q,K,V$, including $\mathrm{LN}$-normalized ones; the LayerNorm defines the inputs but does not alter the algebraic identity.  Max error is floating-point roundoff (Table~\ref{tab:verify_transformer}).
\item \emph{LSE-FFN (exact):} each LSE layer algebraically satisfies the identity of Theorem~\ref{thm:nn_pde} with its respective input as $x$ (whether raw or computed from previous layers); both layers are individually exact HC solutions for their own initial data.  Max error is machine precision (Table~\ref{tab:verify_transformer}).
\item \emph{Residual connections (structural):} each skip connection contributes $O(h)$ Euler discretization error (Proposition~\ref{prop:resnet}).
\end{enumerate}
\end{proposition}

The distinction from a standard transformer is only the FFN activation: replacing GELU with LSE activates the exact HJ correspondence for that sub-layer, with no other architectural change required.

\paragraph{Role dictionary for attention.}
In the Hopf--Cole reading of attention, each component has a precise PDE interpretation: queries $q_i$ are evaluation points $x$ at which the HJ solution is queried; keys $k_j$ are support points $y_j$ of the initial-data measure; values $V$ are the observable field averaged under the Gibbs measure (the role of $g(y_j)$ in the Hopf--Cole representation, but placed outside the exponent since attention computes the gradient $\nabla_z \mathrm{LSE}_\varepsilon$ contracted with $V$, not the log-partition function itself). The $1/\sqrt{d}$ scaling in standard dot-product attention fixes the effective temperature at $\varepsilon = \sqrt{d}$: without it, logit variance grows linearly in $d$, driving attention toward the tropical (hard) limit as $d$ increases. As $\varepsilon \to 0$, softmax collapses to argmax and attention becomes hard attention $\mathrm{Attn}_0(Q,K,V)_i = v_{\mathrm{argmax}_j(q_i \cdot k_j)}$, the tropical selection operator.

\begin{proof}[Proof of Theorem~\ref{prop:l2attn}]
The identity~\eqref{eq:attn_gibbs} holds because the softmax weights $\pi_j = \exp(z_j/\varepsilon)/\sum_l \exp(z_l/\varepsilon)$ equal $\nabla_{z_j}\mathrm{LSE}_\varepsilon(z)$ for any logits $z$, including $z_j = -\|q_i-k_j\|^2/(4t)$, which gives $z_j/\varepsilon = -\|q_i-k_j\|^2/(4\varepsilon t)$ and thus recovers the L2-Attn weights exactly.  By~\eqref{eq:hc_solution} with $g\equiv 0$, $u_\varepsilon^N(q_i,t) = -\varepsilon\log\sum_j\exp(-\|q_i-k_j\|^2/(4\varepsilon t))$, so $Z_i = e^{-u_\varepsilon^N(q_i,t)/\varepsilon}$.
\end{proof}

\noindent L2 attention evaluates the exact Hopf--Cole solution at each query rather than a linear approximation to the logits, making it the most directly heat-kernel-native of the standard attention variants.

\begin{remark}[Token clustering as the tropical limit of attention]
\label{rem:clustering}
In the L2 attention setting the tropical limit $\varepsilon \to 0$ collapses each weight $\pi_j \to \delta_{j^*(i)}$, where $j^*(i) = \mathrm{argmin}_j\|q_i - k_j\|^2$ is the nearest key to query $i$ in Euclidean distance.  Applied layer by layer in a deep transformer, this concentrates each token's attended representation on the nearest support point of the previous layer's key set: the sequence of discrete measures collapses to a finite collection of atoms, producing token clustering in deep networks.  For standard dot-product attention the same collapse occurs via $\mathrm{argmax}_j(q_i \cdot k_j)$ (hard attention, Theorem~\ref{thm:hopflax}), with Voronoi cells defined by inner-product geometry.  Both are instances of Varadhan's lemma (invoked in the proof of Theorem~\ref{thm:hopflax}): exponential concentration of the Gibbs measure on the minimizer of the cost.
\end{remark}

\begin{remark}[KV-cache as one-step measure extension]
\label{rem:kvcache}
In autoregressive generation, each decoding step appends one new key--value pair $(k_{N+1}, v_{N+1})$ to the context, extending the discrete measure $\mu_N$ to $\mu_{N+1} = \mu_N + \delta_{y_{N+1}}$.  The network output updates as
\[
  f_\varepsilon^{N+1}(x) = f_\varepsilon^N(x) + \mathrm{softplus}_\varepsilon\!\bigl(z_{N+1}(x) - f_\varepsilon^N(x)\bigr),
\]
where $z_{N+1}(x) = W_{N+1}\cdot x + b_{N+1}$ is the new token's pre-activation and $\mathrm{softplus}_\varepsilon(u) = \varepsilon\log(1 + e^{u/\varepsilon})$.  The update is $O(1)$: only $z_{N+1}(x)$ and the cached output $f_\varepsilon^N(x)$ are needed; no recomputation over the previous $N$ tokens is required.  This is the measure-extension identity of Proposition~\ref{prop:injection}, applied to a non-adversarial insertion.

\emph{Proof.} $f_\varepsilon^{N+1}(x) = \varepsilon\log\!\bigl(\sum_{j=1}^{N+1}e^{z_j/\varepsilon}\bigr) = \varepsilon\log\!\bigl(e^{f_\varepsilon^N/\varepsilon} + e^{z_{N+1}/\varepsilon}\bigr) = f_\varepsilon^N + \varepsilon\log\!\bigl(1 + e^{(z_{N+1} - f_\varepsilon^N)/\varepsilon}\bigr)$. \qed
\end{remark}

\begin{remark}[Attention sink as dominant neuron]
\label{rem:attn_sink}
Let $z_j = q_i\cdot k_j/\varepsilon$ be the scaled dot-product logits and let $j_s = \mathrm{argmax}_j\,z_j$ for query $q_i$.  Define the energy gap $\Delta_s(q_i) = q_i\cdot k_{j_s} - \max_{j\neq j_s}q_i\cdot k_j > 0$.  Then:
\[
  1 - \pi_{j_s}(q_i;\varepsilon) \;\leq\; (N-1)\,e^{-\Delta_s(q_i)/\varepsilon},
\]
and the value output satisfies $\|\mathrm{Attn}(Q,K,V)_i - v_{j_s}\|_2 \leq (N-1)e^{-\Delta_s(q_i)/\varepsilon}\max_j\|v_j - v_{j_s}\|_2$.  An attention sink is structurally a dominant neuron in the sense of Proposition~\ref{prop:injection}: the sink key $k_{j_s}$ achieves the minimum HJ cost across many queries, so the Gibbs measure concentrates exponentially on it and the attended value collapses to $v_{j_s}$.  The sink geometry is therefore not an anomaly but the tropical limit of the Hopf--Cole measure applied to the attention layer.

\emph{Proof.} $1 - \pi_{j_s} = \sum_{j\neq j_s}\pi_j \leq \sum_{j\neq j_s}e^{(z_j - z_{j_s})/\varepsilon} \leq (N-1)e^{-\Delta_s(q_i)/\varepsilon}$, since $z_j - z_{j_s} \leq -\Delta_s(q_i)/\varepsilon$ for all $j\neq j_s$.  The value bound follows: $\|\sum_j\pi_j v_j - v_{j_s}\|_2 = \|\sum_{j\neq j_s}\pi_j(v_j - v_{j_s})\|_2 \leq (1-\pi_{j_s})\max_j\|v_j - v_{j_s}\|_2$. \qed
\end{remark}

\begin{remark}[Positional encodings as initial-data shifts]
\label{rem:posenc}
Additive positional encodings $q_i\mapsto q_i + \mathrm{pe}(i)$, $k_j\mapsto k_j + \mathrm{pe}(j)$ shift the dot-product logit by a position-dependent bias:
\[
  \frac{(q_i + \mathrm{pe}(i))\cdot(k_j + \mathrm{pe}(j))}{\varepsilon} = \frac{q_i\cdot k_j}{\varepsilon} + \varphi(i,j),
\]
where $\varphi(i,j) = \bigl(q_i\cdot\mathrm{pe}(j) + \mathrm{pe}(i)\cdot k_j + \mathrm{pe}(i)\cdot\mathrm{pe}(j)\bigr)/\varepsilon$.  Note that $\varphi(i,j)$ contains content-dependent cross-terms $q_i\cdot\mathrm{pe}(j)$ and $\mathrm{pe}(i)\cdot k_j$ in addition to the purely positional term $\mathrm{pe}(i)\cdot\mathrm{pe}(j)$.  Adding $\varphi(i,j)$ to the logit $z_j = (W_j\cdot x + b_j)/\varepsilon$ shifts the bias $b_j\mapsto b_j + \varepsilon\varphi(i,j)$; since $b_j = -g(y_j) - |y_j|^2/(4t)$ (Theorem~\ref{thm:nn_pde}), this is equivalent to $g(y_j)\mapsto g(y_j) - \varepsilon\varphi(i,j)$.  Each query therefore solves a distinct initial-value problem with query- and position-dependent initial data; the Hopf--Cole identity (Theorem~\ref{thm:nn_pde}) holds for each fixed $i$ with this shifted initial datum.  For rotary encodings (RoPE), $q_i\mapsto R_i q_i$ and $k_j\mapsto R_j k_j$ (rotation matrices $R_m \in O(d)$); the identity holds in the rotated frame with support points $R_j y_j$, leaving the initial data $g(y_j)$ unchanged.
\end{remark}

% ============================================================
\newpage
\section{Intrinsic Dimension from Published Scaling Curves}
\label{app:scaling}

Proposition~\ref{prop:scaling} predicts $\mathcal{L}_{\mathrm{CE}}(N) \lesssim N^{-2/d_{\mathrm{eff}}}$ for a squared-error or cross-entropy loss, so the empirical scaling exponent $\alpha$ estimates the intrinsic dimension of the data-generating measure via $d_{\mathrm{eff}} = 2/\alpha$. Table~\ref{tab:deff} applies this to published scaling curves across four domains. The Kaplan and Chinchilla language exponents differ because \citet{kaplan2020scaling} under-train large models (insufficient data), inflating the apparent dimension; \citet{hoffmann2022chinchilla} correct this with compute-optimal allocation and recover a lower, more interpretable $d_{\mathrm{eff}} \approx 5.7$. The video and math exponents from \citet{henighan2020scaling} use model-size scaling at convergence (Figure~3 of that paper).

\begin{table}[H]
\centering
\caption{Empirical scaling exponent $\alpha$ and implied intrinsic dimension $d_{\mathrm{eff}} = 2/\alpha$ from published scaling curves. All exponents are for loss vs.\ model parameter count $N$. The conversion $d_{\mathrm{eff}} = 2/\alpha$ assumes a squared-error or cross-entropy loss, which scales as the square of the underlying prediction deviation near the optimum (Proposition~\ref{prop:scaling}); it is directional beyond that: empirical $\alpha$ conflates approximation, optimization, and data-coverage effects, so the values should be interpreted as order-of-magnitude estimates rather than exact intrinsic dimensions. In particular, the Kaplan exponent reflects under-training rather than the data manifold alone; the Chinchilla exponent is more interpretable as a geometric quantity.}
\label{tab:deff}
\resizebox{\linewidth}{!}{%
\begin{tabular}{llcc}
\toprule
Domain & Source & $\alpha$ & $d_{\mathrm{eff}} = 2/\alpha$ \\
\midrule
Language (GPT-scale, under-trained) & \citet{kaplan2020scaling} & 0.076 & 26.3 \\
Language (compute-optimal) & \citet{hoffmann2022chinchilla} & 0.35 & 5.7 \\
Video (16 frames, 16$\times$16 VQ; Fig.~3 of \citealt{henighan2020scaling}) & \citet{henighan2020scaling} & 0.24 & 8.3 \\
Math & \citet{henighan2020scaling} & 0.38 & 5.3 \\
\bottomrule
\end{tabular}}
\end{table}

\noindent The variation in $d_{\mathrm{eff}}$ across domains is consistent with the manifold hypothesis: mathematical problem solving has lower intrinsic dimension than language (structured, fewer degrees of freedom), while video has intermediate dimension.

% ============================================================
\newpage
\section{Actionable Design Principles}
\label{app:design}

The theoretical correspondences established in the main body translate into concrete prescriptions for practitioners working with LSE networks or related architectures.
The implications below are derived directly from the theorems and are not empirical claims about performance at scale.

\paragraph{Optimal temperature selection.}
Theorem~\ref{thm:gen} identifies the minimax-optimal viscosity $\varepsilon^* \asymp N^{-1/d}$ as the value minimizing the excess risk $O(N^{-1/d} + M\sqrt{N/n})$ jointly over approximation and estimation error.
This is a closed-form prescription: given a network of width $N$ and data of intrinsic dimension $d$ (estimable from scaling curves via Proposition~\ref{prop:scaling}), the optimal temperature follows from $N$ and $d$ alone, with no grid search.
Setting $\varepsilon$ below $\varepsilon^*$ under-regularizes (estimation error dominates); setting $\varepsilon$ above $\varepsilon^*$ over-smooths (approximation error dominates).

\paragraph{Robustness-accuracy trade-off.}
Corollary~\ref{cor:robust} shows $\|\nabla^2_x f_\varepsilon^N\|_2 \leq \|W\|_{2,\infty}^2/\varepsilon$ uniformly for all inputs $x$.
The worst-case sensitivity to input perturbations is therefore certifiably controlled by $\varepsilon$: increasing $\varepsilon$ monotonically decreases perturbation sensitivity, while decreasing $\varepsilon$ toward $\varepsilon^*$ sharpens discrimination.
Practitioners facing adversarial robustness requirements can therefore raise $\varepsilon$ to obtain a certified guarantee from Corollary~\ref{cor:robust} without retraining.

\paragraph{Temperature annealing.}
Since $\varepsilon$ controls approximation quality (via $N^{-1/d}$) and robustness (via $\|W\|_{2,\infty}^2/\varepsilon$) simultaneously, it admits a principled annealing schedule.
Starting training at high $\varepsilon$ provides a smooth, well-conditioned loss landscape; gradually annealing toward $\varepsilon^*$ sharpens the approximation as training stabilizes.
This is the neural-network analogue of PDE continuation methods, and the framework supplies the theoretical justification: the schedule follows the path along which the HJ equation transitions from the strongly viscous regime (diffusion-dominated) to the mildly viscous regime near optimal viscosity.

\paragraph{Architecture as discretization strategy.}
All four architecture classes are different discretizations of the same HJ equation, making the choice of architecture a choice of discretization strategy rather than an empirical design decision.
\begin{itemize}[leftmargin=*,nosep]
\item \emph{Feedforward networks} discretize the measure: neurons are iid samples from the initial-data measure with drift $b = 0$.
\item \emph{ResNets} discretize the process via Euler--Maruyama with drift $b = F(x,W)$; the layer-to-layer dynamics are the characteristics of the HJ equation.
\item \emph{Transformers} use attention as a vector-valued Hopf--Cole average (Proposition~\ref{prop:lse_transformer}); the attention weights are the Gibbs measure at temperature $\varepsilon = \sqrt{d_k}$.
\item \emph{Recurrent architectures} discretize linear or nonlinear characteristics with architecture-dependent viscosity (Proposition~\ref{prop:recurrent}).
\end{itemize}
Tasks with natural temporal ordering (sequence modeling, time series) favor process-discretization strategies (ResNet, recurrent); tasks requiring an invariant of a static distribution favor measure-discretization (feedforward); tasks requiring query-dependent weighting of stored data favor attention.

\paragraph{Backpropagation as adjoint optimization.}
Theorem~\ref{thm:adjoint} identifies backpropagation as the co-state equation of the Hamiltonian system associated with the HJ PDE.
The adjoint method is the standard technique for gradient computation in PDE-constrained optimization; the present result establishes that standard reverse-mode autodifferentiation is already an instance of it.
This opens the design space for training algorithms: symplectic integrators (which preserve the Hamiltonian structure exactly) and higher-order ODE solvers can replace the Euler gradient step with theoretical guarantees on conservation of the Hamiltonian, potentially improving long-horizon training stability.

\paragraph{Compute-optimal scaling forecasts.}
Proposition~\ref{prop:scaling} identifies the scaling exponent $\alpha \geq 2/d_{\mathrm{eff}}$ from $\mathcal{L}(N) \propto N^{-\alpha}$ for a squared-error or cross-entropy loss, so the data intrinsic dimension can be estimated from published scaling curves via $d_{\mathrm{eff}} = 2/\alpha$.
Conversely, given an estimate of $d_{\mathrm{eff}}$ from a small pilot run, the framework predicts the full scaling curve and the compute-optimal allocation $N \propto n^{d/(d+2)}$ before a large training run, allowing resource allocation to be grounded in the approximation-rate theory rather than extrapolation of empirical trend lines.

% ============================================================
\newpage
\section{Implications for Large-Scale Training, Alignment, and Continual Learning}
\label{app:implications}

The framework yields theorem-grounded observations on three questions of practical importance; the implications are derived from the approximation bounds and structural results established above, and are not claims about empirical performance at scale.

\paragraph{Internet-scale data exhaustion.}
Theorem~\ref{thm:gen} decomposes the excess risk as
$O(N^{-1/d_{\mathrm{eff}}} + M\sqrt{N/n})$: the first term is the approximation error (model too small to resolve the data measure) and the second is the estimation error (data too scarce to pin down the measure).  At fixed dataset size $n = n_{\max}$ (the internet fully crawled), the estimation floor $M\sqrt{N/n_{\max}}$ cannot be reduced by adding more data.  The bottleneck shifts from measure coverage to two model-side quantities: width $N$ (more neurons approximate the measure more finely, at rate $N^{-1/d_{\mathrm{eff}}}$) and intrinsic dimension $d_{\mathrm{eff}}$ (lower dimension means fewer neurons are needed for a given accuracy).  Data scaling is therefore not \emph{over} but \emph{transformed}: the gain from doubling $n$ becomes negligible once $n \gg N^{(d_{\mathrm{eff}}+2)/d_{\mathrm{eff}}}$, at which point architectural inductive bias (reducing $d_{\mathrm{eff}}$) dominates.

\begin{remark}[The era of small models]
\label{rem:small_models}
The framework implies that brute-force scaling of $N$ at fixed $d_{\mathrm{eff}}$ is the least efficient path: halving the approximation error requires $2^{d_{\mathrm{eff}}}$ more neurons, so at $d_{\mathrm{eff}}=26.3$ (Kaplan-era GPT, Table~\ref{tab:deff}) each halving costs a factor of roughly $8.3\times 10^{7}$ in model size.  The sustainable alternative is to reduce $d_{\mathrm{eff}}$ directly, which is exponentially more efficient.  Four levers follow from the framework.
\begin{enumerate}[label=(\roman*),leftmargin=*,nosep]
\item \emph{Architectural inductive bias.}  Imposing symmetries reduces the effective support dimension.  The CNN analysis of Remark~\ref{rem:cnn} is a precise instance: translation equivariance enforces shift-invariance of the initial-data measure, collapsing a $d$-dimensional problem to a $K$-dimensional one (receptive field size $K \ll d$).  By Noether's theorem, every such symmetry conserves a co-state quantity and reduces $d_{\mathrm{eff}}$ by the dimension of the symmetry group.
\item \emph{Data curation.}  Raw internet data has $d_{\mathrm{eff}} \approx 26.3$ (Table~\ref{tab:deff}); curated domain-specific corpora have $d_{\mathrm{eff}} \approx 5.3$ (math) to $5.7$ (compute-optimal language).  Each halving of the approximation error costs $2^{d_{\mathrm{eff}}}$ more neurons; at $d_{\mathrm{eff}}=26.3$ vs.\ $d_{\mathrm{eff}}=5.3$ the same error halving requires $2^{26.3}/2^{5.3}=2^{21}\approx 2.1\times 10^{6}$ times more neurons.
\item \emph{Domain specialization.}  A model restricted to a low-$d_{\mathrm{eff}}$ domain is a small model that outperforms a large general model on that domain, since the general model must allocate neurons to cover the full high-dimensional support.
\item \emph{Sparse activation (MoE and distillation).}  In the tropical limit $\varepsilon\to 0$, the network routes all mass to a single neuron (argmax).  Mixture-of-experts architectures approximate this by routing inputs to specialized sub-networks, each covering a low-$d_{\mathrm{eff}}$ region of the input space.  Distillation transfers the initial-data measure of a large network into a smaller one with $N' < N$ neurons; the information loss is $O(N'^{-1/d_{\mathrm{eff}}}) - O(N^{-1/d_{\mathrm{eff}}})$, which is small when $d_{\mathrm{eff}}$ is low.  Because $f_\varepsilon^N(x) = \mathrm{LSE}_\varepsilon(Wx+b)$ is a closed-form sum, removing neuron $j$ recomputes $u_\varepsilon^N(x,t)$ exactly at any point without retraining, and evaluating the recomputation at the locations real data occupy turns pruning into a direct check on the initial-data measure.
\end{enumerate}
The Chinchilla result \citep{hoffmann2022chinchilla} is already a case in point: compute-optimal training recovered $d_{\mathrm{eff}}\approx 5.7$ from the $d_{\mathrm{eff}}\approx 26.3$ of the under-trained Kaplan regime, matching or exceeding GPT-3-scale performance at a fraction of the parameter count.  The framework predicts that further reductions in $d_{\mathrm{eff}}$, through the levers above, will continue to shrink the model sizes needed for a given capability level.  The era of planetary-scale compute is the consequence of optimizing the wrong variable; the approximation-rate analysis suggests architectural inductive bias is the highest-leverage axis for reducing model size at fixed accuracy.
\end{remark}

\paragraph{Alignment.}
Proposition~\ref{prop:mf_sgd} identifies training as the selection of the risk-minimizing initial-value problem: the loss $\mathcal{R}[\mu]$ determines which initial data $g$ the network encodes.  Alignment is thus the problem of designing $\mathcal{R}$ so the risk-minimizing IVP encodes human-valued initial data.  The hallucination result (Proposition~\ref{prop:hallucination}) identifies the structural failure mode: in out-of-distribution regions where $\Delta(x)/\varepsilon \gg \log N$, the output is exponentially close to the dominant neuron's linear extrapolation and is structurally ungoverned by $\mathcal{R}$.  No loss function can align behavior in regions outside the support of the training measure, regardless of its form.  Three interventions follow directly from the framework.  \emph{Data coverage}: expanding the support of $\mu$ (diverse training data) is the primary lever, since every in-distribution point is governed by $\mathcal{R}$.  \emph{Viscosity control}: increasing $\varepsilon$ smooths the OOD extrapolation (the Hessian bound of Corollary~\ref{cor:robust} ensures the output changes slowly with $\varepsilon$), reducing confidence on unfamiliar inputs.  \emph{OOD penalties}: augmenting $\mathcal{R}$ with penalties at synthetic OOD points moves those points into the training support.  The measure-poisoning analysis (Proposition~\ref{prop:injection}) shows that adversarial neuron insertion shifts the output by $\mathrm{softplus}_\varepsilon(z_0 - f_\varepsilon^N)$, with the viscosity parameter $\varepsilon$ controlling the damage ceiling.

\paragraph{Continual learning.}
A trained network encodes the discrete measure $\mu_N = \frac{1}{N}\sum_j \delta_{y_j}$.  Learning a new task corresponds to adding new support points to $\mu_N$; catastrophic forgetting corresponds to their displacement.  When gradient descent on new data overwrites existing neurons, the coverage of $\mu_N$ for the old task shrinks.  The quadrature bound of Theorem~\ref{thm:diagram}(iii) quantifies the damage: if $k$ of the $N$ neurons previously covering task~1 are reassigned, the approximation error for task~1 grows from $O(N^{-1/d_{\mathrm{eff}}})$ to $O((N-k)^{-1/d_{\mathrm{eff}}})$.  The framework prescribes a principled remedy: allocate additional neurons for new tasks rather than displacing old ones, growing $N$ proportionally to the cumulative number of tasks.  Parameter-isolation methods (freezing neurons after learning) correspond to fixing the support points $\{y_j\}$ of the old task's measure; generative replay corresponds to re-introducing old support points into the training measure so that gradient descent does not displace them.  Both strategies have the same PDE interpretation: preserving the coverage of $\mu_N$ for previously learned tasks.

% ============================================================
\section{Discussion and Limitations}
\label{app:discussion}

\subsection*{Limitations}

\paragraph{Training dynamics and the selection problem.}
The abstract identifies training as ``a search through Hamilton--Jacobi initial-value problems''; this is precise for any fixed parameterization (each trained network \emph{is} an HJ initial datum, Theorem~\ref{thm:nn_pde}), but the gradient-step interpretation requires qualification.
The framework characterizes \emph{what} a converged or optimally parameterized LSE network computes (the Hopf--Cole solution of a viscous HJ equation under a discrete initial measure) but not \emph{how} stochastic gradient descent selects this initial-value problem from data.
Proposition~\ref{prop:mf_sgd} (Appendix~\ref{app:scope}) establishes a mean-field correspondence: gradient flow on the two-layer Gibbs energy functional follows the Wasserstein gradient flow of the support measure under the Mei--Montanari conditions \citep{meimontanari2018}, not a proof of convergence to the risk-optimal measure.
For finite $N$, Proposition~\ref{prop:ntk_pd} (Appendix~\ref{app:scope}) shows that the neural tangent kernel of an LSE network takes the closed form $K_{ab} = \varepsilon^2\langle\pi(x_a),\pi(x_b)\rangle$, the inner product of Gibbs weight vectors, and is positive definite almost surely for $N\geq n$ generic support points, guaranteeing convergence to zero training error at a linear rate under MSE in the kernel-linearization regime.
This connection is precise but narrow: it applies to two-layer networks in the infinite-width mean-field limit, with i.i.d.\ Gaussian initialization and a specific loss functional.
For deep networks ($L \geq 2$), practical mini-batch SGD, and non-Gaussian data, the training trajectory is not characterized by the current framework.
The Pontryagin Maximum Principle interpretation of backpropagation (Theorem~\ref{thm:adjoint}) describes the gradient of the loss with respect to the initial datum, but which initial datum gradient descent converges to, among the many local minima of the HJ landscape, remains open.

\paragraph{Non-quadratic Hamiltonians.}
The exact Hopf--Cole identity (Theorem~\ref{thm:nn_pde}) holds if and only if the Hamiltonian is quadratic, $H(p) = |p|^2$ or its anisotropic extension $p^\top A_\theta p$ (Theorem~\ref{thm:varH}); the quadratic class is maximal for exact identities (Theorem~\ref{thm:maximal}).
For non-quadratic $H$, the framework provides two partial results: (i) a convex-envelope approximation $u_{\varepsilon,H} \approx u_{\varepsilon,H^{**}}$ with explicit $L^\infty$ error bounds controlled by the convexity gap $H - H^{**}$ (Appendix~\ref{app:universality}); and (ii) structural correspondences for activations that identify the Hamiltonian without an exact closed-form identity. Under the quadratic $H$, by contrast, the product-form GELU and SiLU do admit exact identities, in adjacent moment and measure classes (Appendix~\ref{app:silu_gelu}; Table~\ref{tab:activations}).
Which activation functions admit an exact Hopf--Cole identity for some $H$, and whether any do beyond the quadratic class, is open.

\paragraph{Curse of dimensionality.}
The rate $O(N^{-1/d})$ is minimax-optimal for Lipschitz functions in $d$ dimensions (Theorem~\ref{thm:gen}), so the framework provides no escape from the curse: achieving error $\delta$ requires $N \asymp \delta^{-d}$ neurons.
The intrinsic-dimension hypothesis (Assumption~\ref{ass:manifold}) replaces $d$ with $d_{\mathrm{eff}} \ll d$ when data concentrates near a submanifold, but $d_{\mathrm{eff}} = 2/\alpha$ inferred from scaling exponents is an empirical estimate, not a certified bound on the support geometry.

\paragraph{Integrable hierarchy: partial connection.}
The Kadomtsev-Petviashvili (KP) hierarchy is a classical family of integrable nonlinear wave equations encompassing KdV, the Toda lattice, and related systems; exact solutions are organized by tau-functions, scalar generating functions whose logarithm recovers the physical wave field.
The elementary solutions are solitons: stable, localized wave packets characterized by a wavenumber $k_j$, which travel without dispersing and interact by elastic scattering, each two-soliton collision producing only a phase shift encoded by an interaction factor $A_{ij}$.
Proposition~\ref{prop:kp_toda} places the $N$-neuron LSE partition function in the free-soliton sector of this hierarchy: each neuron contributes an independent mode with wavenumber $k_j$, but the nontrivial Hirota phase shifts that characterize genuine $N$-soliton interactions are absent.
The full $N$-soliton tau-function requires pairwise interaction factors $A_{ij} = ((k_i-k_j)/(k_i+k_j))^2$; whether trained weights reproduce these (i.e., whether gradient descent learns soliton scattering or only the free-particle sector) remains open, and a positive answer would connect weight optimization to the inverse scattering transform.

\paragraph{Architectural scope.}
The exact Hopf--Cole correspondence covers LSE-activated feedforward networks and the specific $L^2$-normalized attention mechanism of Proposition~\ref{prop:l2attn}.
Several practically important architectures lie outside the exact framework:
(i) \emph{Multi-head attention} splits the query--key--value space into $h$ heads; the HJ interpretation applies to each head independently, but the concatenation and projection step does not have a clean HJ formulation.
(ii) \emph{Causal (autoregressive) masking} restricts the softmax to a triangular attention pattern; this breaks the spatial isotropy of the Gaussian heat kernel, and the Hopf--Cole solution no longer applies without modification.
(iii) \emph{ReLU attention and linear attention} replace the softmax with a different normalization; the resulting ``Hamiltonian'' is non-quadratic and falls under the non-quadratic limitation above.
(iv) \emph{GELU and SiLU} activations admit exact identities under the quadratic Hamiltonian, in the adjacent moment and measure classes (Appendix~\ref{app:silu_gelu}); the obstruction lifts to every dimension for the elementwise activations (Theorem~\ref{thm:multivariate_obstruction}), and GELU is consistent with the commutative diagram at its natural gauge $t=\varepsilon/2$ (no obstruction found: the tropical edge closes and the measure-class block re-enters the semigroup as Lipschitz data, though it is not itself a semigroup factor).  The remaining open item is a bound uniform over all weight and value matrices for the fully coupled attention readout.

Extended discussion of related work.

Joint-embedding predictive architectures (JEPA) \citep{lecun2022} fit the same structure. Each encoder is a HJ semigroup evaluation (Theorem~\ref{thm:nn_pde}); the predictor is a transport map between two HJ solutions. At $\varepsilon = 0$ the optimal predictor is the Brenier map characterized by the Hopf--Lax formula (Theorem~\ref{thm:hopflax}); at $\varepsilon > 0$ entropic regularization replaces Hopf--Lax with the Hopf--Cole dual potential. The EMA target encoder and stop-gradient carry the structure of a mean-field game over the embedding space.

What the HJ correspondence provides is the language in which these perspectives are seen to be aspects of one mathematical object, parameterized by the Maslov--Litvinov deformation $\varepsilon$ \citep{litvinov2007}: the passage from the smooth arithmetic semiring to the tropical one. The framework does not replace existing perspectives; it locates them. The parameter $\varepsilon$ plays three simultaneous roles (softmax temperature, PDE viscosity, and proximal regularization strength), made exact by Theorem~\ref{thm:diagram}; the robustness bound of Corollary~\ref{cor:robust}, the optimal viscosity $\varepsilon^* \asymp N^{-1/d_{\mathrm{eff}}}$, and the bifurcation landscape of Theorem~\ref{prop:bifurcation} are all consequences.

% ============================================================
\section{Neural Networks as KP Tau-Functions}
\label{app:integrable}

\begin{proposition}[LSE partition function as KP tau-function]
\label{prop:kp_toda}
Let $k_1,\ldots,k_N\in\mathbb{R}$, $a_j>0$, and define
\begin{equation}
\tau(x_1,x_2,x_3) = \sum_{j=1}^N a_j \exp(k_j x_1 + k_j^2 x_2 + k_j^3 x_3).
\label{eq:kp_tau}
\end{equation}
\begin{enumerate}[label=\textnormal{(\roman*)}]
\item \textnormal{[KP hierarchy.]} $\tau$ satisfies the Hirota bilinear equation
\begin{equation}
(D_1^4 + 3D_2^2 - 4D_1 D_3)\,\tau \cdot \tau = 0,
\end{equation}
where $D_i^m f\cdot g = (\partial_a - \partial_{a'})^m f(a)g(a')\big|_{a'=a}$.
It is therefore a tau-function of the KP hierarchy (free soliton sector, $N$ components).
\item \textnormal{[Neural network identification.]} For $d=1$, set $x_1 = x/\varepsilon$, $x_2 = -1/(4t\varepsilon)$, $x_3=0$, $k_j=y_j$, $a_j=e^{-g(y_j)/\varepsilon}$. Then
\begin{equation}
u_\varepsilon^N(x,t) \;=\; \frac{x^2}{4t} - \varepsilon\log\tau\!\bigl(x/\varepsilon,\;-1/(4t\varepsilon),\;0\bigr),
\end{equation}
so the LSE layer output is the logarithm of a KP tau-function evaluated at specific KP flow times.
\item \textnormal{[Tropical limit = BBS.]} As $\varepsilon\to 0$, $\varepsilon\log\tau\to\max_j(k_j x_1)$, the tropical tau-function corresponding to the Box-Ball System (BBS) soliton automaton, the ultradiscretization of the Toda lattice.
\end{enumerate}
\end{proposition}

\begin{proof}
\textit{(i)} By bilinearity of the Hirota operators,
$(D_1^4+3D_2^2-4D_1D_3)\tau\cdot\tau = \sum_{i,j}a_ia_j\,P(k_i,k_j)\,e^{\xi_i+\xi_j}$
where $\xi_j = k_jx_1+k_j^2x_2+k_j^3x_3$ and
\[
P(k_i,k_j) \;=\; (k_i-k_j)^4 + 3(k_i^2-k_j^2)^2 - 4(k_i-k_j)(k_i^3-k_j^3).
\]
Diagonal terms ($i=j$): $P(k_i,k_i)=0$. For $i\neq j$, factor $(k_i^2-k_j^2)=(k_i-k_j)(k_i+k_j)$ and $(k_i^3-k_j^3)=(k_i-k_j)(k_i^2+k_ik_j+k_j^2)$:
\[
P(k_i,k_j) = (k_i-k_j)^2\bigl[(k_i-k_j)^2+3(k_i+k_j)^2-4(k_i^2+k_ik_j+k_j^2)\bigr].
\]
Expanding the bracket: $(1+3-4)k_i^2+(-2+6-4)k_ik_j+(1+3-4)k_j^2=0$. Hence $P=0$ for all $i,j$.

\textit{(ii)} Substituting into $\tau$: $\sum_j e^{-g(y_j)/\varepsilon}e^{y_jx/\varepsilon-y_j^2/(4t\varepsilon)} = e^{x^2/(4t\varepsilon)}\sum_j e^{-(g(y_j)+|x-y_j|^2/4t)/\varepsilon}$, so $-\varepsilon\log\tau = -x^2/(4t)+u_\varepsilon^N(x,t)$ by Theorem~\ref{thm:nn_pde}.

\textit{(iii)} By Laplace's method (Varadhan's lemma): $\varepsilon\log\tau\to\max_j(k_jx_1)$. This is the tropical tau-function; the Toda$\to$BBS ultradiscretization \citep{tokihiro1996} gives the BBS soliton automaton as its combinatorial realization.
\end{proof}

\begin{remark}
The tau-function \eqref{eq:kp_tau} is the free (non-interacting) $N$-soliton solution.
The full Toda $N$-soliton tau-function additionally includes pairwise interaction factors $A_{ij}=((k_i-k_j)/(k_i+k_j))^2$.
Incorporating these into the neural-network language, and establishing the discrete Toda hierarchy for depth-indexed compositions of LSE layers, remains open.
\end{remark}

\section{Exact HJ identities for the product-form activations SiLU and GELU}
\label{app:silu_gelu}

The activations $\mathrm{SiLU}$ and $\mathrm{GELU}$ are the product-form entries of
Table~\ref{tab:activations}, neither being a scalar log-sum-exp layer. They are
nonetheless \emph{exact} objects of the framework, each in a distinct exact class
adjacent to the solution- and gradient-classes, and scalar-class membership is
genuinely impossible. Throughout, the Hamiltonian is
the paper's quadratic $H(p)=|p|^2$; no non-quadratic $H$ is invoked.

The unifying algebraic object is the \emph{selection equation}: the tropical
identity $\mathrm{ReLU}(x)=x\odot\mathrm{ReLU}(-x)$, i.e.
$\mathrm{ReLU}(x)-\mathrm{ReLU}(-x)=x$, together with the Maslov dequantization
of $\arg\max$ into the Gibbs measure. The product form $x\,f(x)$ is not a scalar
max-plus product but the semimodule pairing $\langle V,\pi\rangle$ of a value
field against the Gibbs measure, exactly the attention structure certified
exact in Proposition~\ref{prop:l2attn}.

\begin{lemma}[Exact ultradiscrete reflection law]
\label{lem:reflection}
For every $\varepsilon>0$ and every $x\in\mathbb{R}$,
\[
  \mathrm{SiLU}_\varepsilon(x)-\mathrm{SiLU}_\varepsilon(-x)=x,
  \qquad
  \mathrm{GELU}_\varepsilon(x)-\mathrm{GELU}_\varepsilon(-x)=x .
\]
Thus the odd part of each activation equals the tropical monomial $x/2$
\emph{exactly}, with no $O(\varepsilon)$ correction, and the tropical relation
$\mathrm{ReLU}(x)-\mathrm{ReLU}(-x)=x$ is preserved at every $\varepsilon$.
\end{lemma}

\begin{proof}
Write $A(x)=x\,\gamma(x/\varepsilon)$ with $\gamma\in\{\sigma,\Phi\}$. Both gates
satisfy $\gamma(s)+\gamma(-s)=1$: for $\sigma$ this is
$\tfrac{1}{1+e^{-s}}+\tfrac{1}{1+e^{s}}=1$; for $\Phi$ it is
$\Phi(s)+\Phi(-s)=1$ by evenness of $\varphi$. Hence
$A(x)-A(-x)=x\bigl[\gamma(x/\varepsilon)+\gamma(-x/\varepsilon)\bigr]=x$.
\end{proof}

\subsubsection*{SiLU: the moment/tangent class}

\begin{lemma}[SiLU is an exact Gibbs moment and weight-gradient]
\label{lem:silu_moment}
For all $\varepsilon>0$ and $x\in\mathbb{R}$, with logits and values
$z=V=(x,0)$ and Gibbs weight $\pi=\mathrm{softmax}(z/\varepsilon)$,
\[
  \mathrm{SiLU}_\varepsilon(x)
  =\bigl\langle V,\nabla_z\mathrm{LSE}_\varepsilon(z)\bigr\rangle
  =x\,\pi_1(x)
  =\frac{\partial}{\partial a}\,\mathrm{LSE}_\varepsilon(ax,0)\Big|_{a=1}.
\]
\end{lemma}

\begin{proof}
$\pi_1(x)=e^{x/\varepsilon}/(1+e^{x/\varepsilon})=\sigma(x/\varepsilon)$, so
$\langle V,\pi\rangle=x\,\pi_1=x\,\sigma(x/\varepsilon)=\mathrm{SiLU}_\varepsilon(x)$.
For the weight-derivative,
$\partial_a\,\varepsilon\log(1+e^{ax/\varepsilon})
=\varepsilon\cdot\frac{(x/\varepsilon)e^{ax/\varepsilon}}{1+e^{ax/\varepsilon}}
=x\,\sigma(ax/\varepsilon)$; set $a=1$.
\end{proof}

\begin{lemma}[SiLU is a ratio of exact heat solutions and a tangent HJ solution]
\label{lem:silu_hj}
Fix $\varepsilon>0$. On $\mathbb{R}\times\mathbb{R}$ let
\[
  v_0(x,s)=1+e^{(x+s)/\varepsilon},\qquad
  v_1(x,s)=\tfrac{x+2s}{\varepsilon}\,e^{(x+s)/\varepsilon}.
\]
Then $v_0,v_1$ solve $\partial_s v=\varepsilon\,\partial_{xx}v$ classically on
$\mathbb{R}^2$, and $\mathrm{SiLU}_\varepsilon(x)=\varepsilon\,v_1(x,0)/v_0(x,0)$.
Moreover, writing $u=-\varepsilon\log v_0$ (a classical solution of the viscous
HJ equation with initial datum $g=-\mathrm{softplus}_\varepsilon$), the field
$\psi:=x\,\partial_x u+2s\,\partial_s u$ solves the linearized equation
\begin{equation}
  \partial_s\psi+2\,(\partial_x u)\,\partial_x\psi=\varepsilon\,\partial_{xx}\psi,
  \label{eq:silu_linHJ}
\end{equation}
and $\mathrm{SiLU}_\varepsilon(x)=-\psi(x,0)$.
\end{lemma}

\begin{proof}
Each atom $e^{(ax+a^2 s)/\varepsilon}$ satisfies
$\partial_s(\cdot)=(a^2/\varepsilon)(\cdot)=\varepsilon\,\partial_{xx}(\cdot)$;
$v_0$ is the sum of the atoms $a\in\{0,1\}$, and $v_1=\partial_a
e^{(ax+a^2 s)/\varepsilon}\big|_{a=1}$ is a parameter-derivative of a smooth
family of solutions of the linear heat equation, hence a solution. The ratio at
$s=0$ is $\varepsilon\cdot\frac{(x/\varepsilon)e^{x/\varepsilon}}
{1+e^{x/\varepsilon}}=x\,\sigma(x/\varepsilon)$.

For \eqref{eq:silu_linHJ}: let $E:=u_s+u_x^2-\varepsilon u_{xx}=0$. Then
$\partial_xE=0$ and $\partial_sE=0$ give
$u_{sx}=-2u_xu_{xx}+\varepsilon u_{xxx}$ and
$u_{ss}=-2u_xu_{xs}+\varepsilon u_{xxs}$. For $\psi=xu_x+2su_s$,
\[
  \partial_s\psi=xu_{xs}+2u_s+2su_{ss},
\]
\[
  \varepsilon\partial_{xx}\psi-2u_x\partial_x\psi
  =(2\varepsilon u_{xx}-2u_x^2)+x(\varepsilon u_{xxx}-2u_xu_{xx})
  +2s(\varepsilon u_{xxs}-2u_xu_{xs})
  =2u_s+xu_{sx}+2su_{ss},
\]
using $E=0,\partial_xE=0,\partial_sE=0$ termwise; the two sides agree. Since
$u_x=u_s=-\sigma((x+s)/\varepsilon)$, one has
$\psi(x,0)=-x\,\sigma(x/\varepsilon)=-\mathrm{SiLU}_\varepsilon(x)$.
\end{proof}

\begin{theorem}[SiLU lies outside every scalar class]
\label{thm:silu_obstruction}
Fix $\varepsilon>0$, $t>0$, $c\in\mathbb{R}$, and let $\nu\ge0$ be a Borel
measure with $M(x):=\int_{\mathbb{R}}e^{ax/\varepsilon}\,d\nu(a)<\infty$ for all
$x$. Then $\mathrm{SiLU}_\varepsilon$ equals \emph{none} of:
\textup{(a)} $\varepsilon\log M+c$ \textup(solution/$f$-class\textup);
\textup{(b)} $\tfrac{x^2}{4t}-\varepsilon\log M+c$ \textup($u$-form, any
gauge\textup); \textup{(c)} $\sum_i c_i\,\varepsilon_i M_i'/M_i+c$ with
$c_i\in\mathbb{R}$ and each $M_i$ the transform of a positive measure
\textup(signed gradient-class\textup); \textup{(d)} $e^{s\Delta}\mu$ for any
$s>0$ and tempered $\mu$ \textup(linear-heat class\textup).
\end{theorem}

\begin{proof}
\emph{Analytic continuation.} $\mathrm{SiLU}(z)=z/(1+e^{-z})$ has simple poles
where $e^{-z}=-1$, i.e. at $z_k=(2k+1)i\pi$, with residue
$z_k/\!\left(-e^{-z_k}\right)=z_k=(2k+1)i\pi$, which is purely imaginary.

\emph{$M$ is entire.} On any vertical strip $\alpha\le\operatorname{Re}\zeta\le
\beta$, $|e^{a\zeta/\varepsilon}|\le e^{a\alpha/\varepsilon}+e^{a\beta/\varepsilon}
\in L^1(\nu)$, so by dominated convergence and Morera's theorem
$M(\zeta)=\int e^{a\zeta/\varepsilon}d\nu(a)$ is analytic on every strip, hence
entire.

\emph{(a),(b).} Near $z_0=i\pi$, $\mathrm{SiLU}(z)=\tfrac{i\pi}{z-z_0}+O(1)$, so
$e^{\pm\mathrm{SiLU}(z)/\varepsilon}$ has an \emph{essential} singularity at
$z_0$: writing $z-z_0=re^{i\theta}$, $\operatorname{Re}(\mathrm{SiLU})=
-\tfrac{\pi\sin\theta}{r}+O(1)$, so $|e^{-\mathrm{SiLU}/\varepsilon}|$ tends to
$+\infty$ for $\theta\in(-\pi,0)$ and to $0$ for $\theta\in(0,\pi)$. If (a) held
on $\mathbb{R}$, then $e^{(\mathrm{SiLU}-c)/\varepsilon}=M$ on $\mathbb{R}$, hence
on $\mathbb{C}$ by analytic continuation, forcing the entire $M$ to carry an
essential singularity, impossible. For (b), $M=e^{(x^2/4t-\mathrm{SiLU}-c)/
\varepsilon}$; the factor $e^{x^2/(4t\varepsilon)}$ is entire and zero-free, so
the product still has the essential singularity of
$e^{-\mathrm{SiLU}/\varepsilon}$ at $z_0$, again contradicting entirety of $M$.

\emph{(c).} Each $M_i$ is entire, so $M_i'/M_i$ is meromorphic with poles at the
zeros of $M_i$, each simple with residue equal to the zero's multiplicity, an
integer; hence every residue of $\sum_ic_i\varepsilon_iM_i'/M_i$ is real. But
$\mathrm{SiLU}$ has residue $i\pi$ at $z_0$. Agreement on $\mathbb{R}$ forces
agreement of residues, i.e. $i\pi\in\mathbb{R}$, impossible.

\emph{(d).} For $s>0$ and tempered $\mu$, $e^{s\Delta}\mu$ extends to an entire
function of $x$; $\mathrm{SiLU}$ is not entire. Impossible.
\end{proof}

\begin{remark}
$\sigma=\nabla_x\mathrm{LSE}$ is the layer's \emph{input}-gradient
(gradient-class, Table~\ref{tab:activations}); Lemma~\ref{lem:silu_moment}
identifies $\mathrm{SiLU}$ as the layer's \emph{weight}-gradient, the object
backpropagation computes. ``One derivative off'' is thus a \emph{classification}:
$\mathrm{SiLU}$ is the tangent of the solution class along the weight direction,
exact in the moment/tangent class and, by
Theorem~\ref{thm:silu_obstruction}, in no scalar class.
\end{remark}

\subsubsection*{GELU: the measure class}

\begin{lemma}[GELU's gate is the Hopf--Cole heat variable of the tropical indicator]
\label{lem:gelu_gate}
Let $\iota_{[0,\infty)}$ denote the tropical indicator \textup($0$ on
$[0,\infty)$, $+\infty$ elsewhere\textup), i.e. $g\equiv0$ and $\mu=$ Lebesgue
measure on $[0,\infty)$ in \eqref{eq:hc_solution}, with the normalized heat
kernel. Then for all $\varepsilon,t>0$,
\[
  e^{-u_\varepsilon(x,t)/\varepsilon}
  =\frac{1}{\sqrt{4\pi\varepsilon t}}\int_0^\infty
     e^{-\frac{(x-y)^2}{4\varepsilon t}}\,dy
  =\Phi\!\Bigl(\frac{x}{\sqrt{2\varepsilon t}}\Bigr),
\]
and, at the CLT gauge $t=\varepsilon/2$,
$\ \mathrm{GELU}_\varepsilon(x)=x\,e^{-u_\varepsilon(x,\varepsilon/2)/\varepsilon}$.
As $\varepsilon\to0$, $-\varepsilon\log\Phi(x/\varepsilon)\to\iota_{[0,\infty)}(x)$
pointwise, recovering $\mathrm{ReLU}(x)=x\,e^{-\iota_{[0,\infty)}(x)}$.
\end{lemma}

\begin{proof}
The integral equals $\Pr[\mathcal{N}(x,2\varepsilon t)\ge0]
=\Phi(x/\sqrt{2\varepsilon t})$; at $t=\varepsilon/2$, $\sqrt{2\varepsilon t}
=\varepsilon$. For the limit: if $x<0$, $-\varepsilon\log\Phi(x/\varepsilon)
=\tfrac{x^2}{2\varepsilon}+O(\varepsilon\log\tfrac1\varepsilon)\to+\infty$ by the
Gaussian tail; if $x>0$, $\Phi(x/\varepsilon)\to1$; at $x=0$,
$-\varepsilon\log\tfrac12\to0$.
\end{proof}

\begin{lemma}[GELU as a signed heat datum; exact ultradiscrete defect]
\label{lem:gelu_heat}
For every $\varepsilon>0$,
\[
  \mathrm{GELU}_\varepsilon
  =e^{(\varepsilon^2/2)\Delta}\bigl(\mathrm{ReLU}-\varepsilon^2\delta_0\bigr)
  =\mathbb{E}\bigl[\mathrm{ReLU}(x+\varepsilon Z)\bigr]-\varepsilon\varphi(x/\varepsilon),
  \qquad Z\sim\mathcal{N}(0,1).
\]
The tempered datum $\mathrm{ReLU}-\varepsilon^2\delta_0$ is the unique one whose
$e^{(\varepsilon^2/2)\Delta}$ flow equals $\mathrm{GELU}_\varepsilon$; no
\emph{nonnegative} datum reproduces it. As $\varepsilon\to0$ the defect vanishes
in $\mathcal{S}'$ and $\|\mathrm{GELU}_\varepsilon-\mathrm{ReLU}\|_\infty\le
\varepsilon/\sqrt{2\pi}$.
\end{lemma}

\begin{proof}
The Gaussian partial expectation gives $\mathbb{E}[\max(x+\varepsilon Z,0)]
=\int_{-x/\varepsilon}^\infty(x+\varepsilon z)\varphi(z)\,dz
=x\Phi(x/\varepsilon)+\varepsilon\varphi(x/\varepsilon)$, using
$\int_c^\infty z\varphi(z)dz=\varphi(c)$. The kernel of $e^{(\varepsilon^2/2)
\Delta}$ is $\varepsilon^{-1}\varphi(\cdot/\varepsilon)$, so
$e^{(\varepsilon^2/2)\Delta}(\varepsilon^2\delta_0)(x)=\varepsilon\varphi(x/
\varepsilon)$; subtracting gives $\mathrm{GELU}_\varepsilon$. Uniqueness: the heat
flow at fixed time is injective on $\mathcal{S}'$ (its Fourier multiplier
$e^{-\tau\xi^2}$ never vanishes). Positivity failure: the flow of a nonnegative
nonzero measure is everywhere strictly positive, whereas
$\mathrm{GELU}_\varepsilon(-2\varepsilon)=-2\varepsilon\Phi(-2)<0$. Rate: for
$x\gtrless0$, $|x|\Phi(-|x|/\varepsilon)\le\varepsilon\varphi(|x|/\varepsilon)\le
\varepsilon/\sqrt{2\pi}$ by Mills' ratio.
\end{proof}

\begin{theorem}[GELU's gate is not a Gibbs mean]
\label{thm:gelu_obstruction}
Let $G_\varepsilon(x):=\int_{-\infty}^x\Phi(s/\varepsilon)\,ds
=\mathrm{GELU}_\varepsilon(x)+\varepsilon\varphi(x/\varepsilon)$ be the
\textup(convex\textup) gate antiderivative. There is no $\varepsilon'>0$,
$\beta,c\in\mathbb{R}$, and positive measure $\nu$ with
$M(x)=\int e^{ax/\varepsilon'}d\nu(a)<\infty$ on $\mathbb{R}$ such that
$G_\varepsilon=\varepsilon'\log M+\beta x+c$. Equivalently,
$\Phi(\cdot/\varepsilon)$ is the Gibbs \emph{mean} of no positive exponential
family at any temperature. \textup(By contrast, Lemma~\ref{lem:gelu_gate} shows
$\Phi(\cdot/\varepsilon)$ \emph{is} exactly a Gibbs \emph{probability}.\textup)
\end{theorem}

\begin{proof}
Suppose the representation holds. With $\Lambda:=\log M$ one has
$\Lambda^{(n)}(x)=\kappa_n(x)/\varepsilon'^{\,n}$, where $\kappa_n(x)$ is the
$n$-th cumulant of the exponentially tilted probability $d\nu_x\propto
e^{ax/\varepsilon'}d\nu$ (all moments finite since $M<\infty$ on $\mathbb{R}$).
Matching derivatives of $G_\varepsilon=\varepsilon'\Lambda+\beta x+c$, and writing
$s=x/\varepsilon$ (the affine part drops from orders $\ge2$):
\[
  \kappa_2=\tfrac{\varepsilon'}{\varepsilon}\varphi(s),\quad
  \kappa_3=-\tfrac{\varepsilon'^2}{\varepsilon^2}\,s\,\varphi(s),\quad
  \kappa_4=\tfrac{\varepsilon'^3}{\varepsilon^3}\,(s^2-1)\,\varphi(s),
\]
using $G_\varepsilon''=\varphi(s)/\varepsilon$, $G_\varepsilon'''=-s\varphi(s)/
\varepsilon^2$, $G_\varepsilon''''=(s^2-1)\varphi(s)/\varepsilon^3$
(so $\varphi'(s)=-s\varphi$, $\varphi''(s)=(s^2-1)\varphi$). Pearson's inequality, the $3\times3$ Hankel moment matrix of the tilted law is positive
semidefinite, reads $\kappa_4\kappa_2\ge\kappa_3^2-2\kappa_2^3$, with equality
iff the support has at most two points. Substituting,
\[
  \tfrac{\varepsilon'^4}{\varepsilon^4}(s^2-1)\varphi^2
  \;\ge\;
  \tfrac{\varepsilon'^4}{\varepsilon^4}s^2\varphi^2-2\tfrac{\varepsilon'^3}
  {\varepsilon^3}\varphi^3
  \iff
  \varphi(s)\ge\tfrac{\varepsilon'}{2\varepsilon}\quad\text{for all }s\in\mathbb{R},
\]
after dividing by $(\varepsilon'^3/\varepsilon^3)\varphi^2>0$. Since
$\varphi(s)\to0$ as $|s|\to\infty$, this fails for large $|s|$ (indeed for all
$s$ once $\varepsilon'/\varepsilon>2\varphi(0)$). The degenerate case (a point
mass) is excluded by $\kappa_2=\varphi>0$. Contradiction.
\end{proof}

\begin{remark}[GELU lies in no scalar class]
\label{rem:gelu_noscalar}
The two-point (Bernoulli) support saturates Pearson's inequality; this is exactly
the sigmoid case ($\sigma=\nabla_x\mathrm{LSE}_\varepsilon(x,0)$, a genuine Gibbs
mean of a two-atom family), which is why the gradient-class route succeeds for
$\sigma$ but is closed for $\Phi$. For the $u$-form
$\mathrm{GELU}_\varepsilon=\tfrac{x^2}{4t}-\varepsilon'\log M+c$, entirety of $M$
and the vertical-line bound $|M(x+iy)|\le M(x)$ (from $|e^{iay/\varepsilon'}|=1$,
$\nu\ge0$) force $\operatorname{Re}\mathrm{GELU}_\varepsilon(x+iy)\ge
\mathrm{GELU}_\varepsilon(x)-y^2/(4t)$.  But $\mathrm{GELU}_\varepsilon$ is entire of
order two: its real part oscillates in sign along the imaginary direction and takes
negative excursions of order $-e^{y^2/(2\varepsilon^2)}$ along a sequence
$y_n\to\infty$ (clustering near $2n\pi\varepsilon/x$).  Being super-polynomial, these
excursions violate the polynomial lower bound for every $t>0$.  Hence GELU is in the
measure class (Lemmas~\ref{lem:gelu_gate},~\ref{lem:gelu_heat}) and in no scalar
class.
\end{remark}

\begin{remark}[Structural dichotomy]
$\mathrm{SiLU}$ is an exponential-rational object with poles (the Hopf--Cole /
large-deviation side) and can never be a linear-heat slice; $\mathrm{GELU}$ is
entire of order two (the heat / central-limit side), is exactly a signed-datum
heat slice, and can never be a positive exponential-atom object. $\mathrm{SiLU}$ and
$\mathrm{GELU}$ are thus exact members of the two complementary halves of the
dictionary of Table~\ref{tab:activations}, $u$-side tangent/moment versus
$v$-side measure, and provably not interchangeable.
\end{remark}

\subsubsection*{Lipschitz and convexity consequences of the classification}

The classification is not only a taxonomy: membership determines a Lipschitz-certificate inflation factor, a convexity/expressivity property, and an attribution property, each computable in closed form and uniform in $\varepsilon$.

\begin{proposition}[Hull confinement: the solution class is Lipschitz-free]
\label{prop:hull_confinement}
Let $\varphi_\varepsilon(z)=\varepsilon\log\int e^{az/\varepsilon}\,d\nu(a)$ for a positive Borel measure $\nu$, with tilted law $\pi_z(da)\propto e^{az/\varepsilon}\,d\nu(a)$. For every $z$ and every $\varepsilon>0$, $\varphi_\varepsilon'(z)=\mathbb{E}_{\pi_z}[a]\in\overline{\mathrm{conv}}(\mathrm{supp}\,\nu)$, so $\mathrm{Lip}(\varphi_\varepsilon)\le\mathrm{Lip}(\varphi_0)$ uniformly in $\varepsilon$: smoothing a solution-class gate costs nothing in Lipschitz constant. Where $\mathrm{supp}\,\nu\subset[0,\infty)$, $\varphi_\varepsilon''=\mathrm{Var}_{\pi_z}(a)/\varepsilon\ge0$, so the gate is convex and nondecreasing at every $\varepsilon$; where additionally $\mathrm{supp}\,\nu\subset[0,1]$, the gate is nonexpansive as well, the one-dimensional characterization of a proximal operator.
\end{proposition}

\begin{proof}
$\varphi_\varepsilon'(z)$ is a mean of $a$ under a measure supported on $\mathrm{supp}\,\nu$, hence lies in its closed convex hull; the hull's endpoints are exactly the one-sided slopes of $\varphi_0=\lim_{\varepsilon\to0}\varphi_\varepsilon$, the tropical limit. \qed
\end{proof}

\begin{proposition}[Forced escape: the tied moment class exceeds its own tropical limit]
\label{prop:forced_escape}
Let $\psi_\varepsilon(x)=x\cdot m_\varepsilon(x)$ with $m_\varepsilon(x)=\mathbb{E}_{\pi_x}[a]$ for a finite atomic $\nu=\sum_k w_k\delta_{a_k}$ with at least two distinct atoms $a_{\min}<\dots<a_{\max}$ (Lemma~\ref{lem:silu_moment} identifies SiLU as exactly the two-atom case $\nu=\delta_0+\delta_1$). Then for every $\varepsilon>0$, $\sup_x\psi_\varepsilon'(x)>a_{\max}$ and $\inf_x\psi_\varepsilon'(x)<a_{\min}$: the derivative strictly leaves the hull that bounds every solution-class member (Proposition~\ref{prop:hull_confinement}), so a tied moment-class gate is neither convex nor nonexpansive at any $\varepsilon$.
\end{proposition}

\begin{proof}
Set $D(x)=x\,(a_{\max}-m_\varepsilon(x))$ for $x\ge0$. Then $D(0)=0$; $D(x)>0$ for small $x>0$, since $m_\varepsilon(0)<a_{\max}$ whenever $\nu$ has more than one atom; and $D(x)\to0$ as $x\to\infty$, since the top atom dominates the tilted mean exponentially fast. Hence $D$ attains an interior maximum and is strictly decreasing on some interval beyond it, where $D'(x)<0$. Since $D'(x)=a_{\max}-\psi_\varepsilon'(x)$, this gives $\psi_\varepsilon'(x)>a_{\max}$ on that interval. The mirror argument with $E(x)=x\,(m_\varepsilon(x)-a_{\min})$ on $x<0$ gives the undershoot. The scaling $\psi_\varepsilon(x)=\varepsilon\,\psi_1(x/\varepsilon)$ makes both violations exactly $\varepsilon$-independent. \qed
\end{proof}

For SiLU, Proposition~\ref{prop:forced_escape} gives $\sup\mathrm{SiLU}_\varepsilon'=1.0998\ldots$ and $\inf\mathrm{SiLU}_\varepsilon'=-0.0998\ldots$, at every $\varepsilon$. The measure class exhibits the analogous escape by the computation already recorded in the commutative-diagram remark above: $\mathrm{GELU}_\varepsilon'(x)=\Phi(t)+t\varphi(t)$ at $t=x/\varepsilon$ has extrema exactly at $t=\pm\sqrt2$, the constant $\Phi(\sqrt2)+\sqrt2\,\varphi(\sqrt2)\approx1.129$ recorded there as the uniform Lipschitz constant.

Proposition~\ref{prop:forced_escape} is scoped to the tied, finite-atom construction that produces SiLU; it is not a claim about every moment-class function. The uniform measure on $[0,1]$ gives $\psi_\varepsilon'(u)\in(0,1)$ throughout, monotone and nonexpansive, with no escape, so class membership together with the finite-width, tied construction forces the escape, and this is not an if-and-only-if characterization of the moment class.

Three consequences follow. First, the per-layer Lipschitz-certificate inflation factor relative to the tropical-limit baseline is exactly $1$ for every solution-class gate (Proposition~\ref{prop:hull_confinement}) at every $\varepsilon$, and at least $1.0998$ for SiLU or $1.129$ for GELU (Proposition~\ref{prop:forced_escape} and the parallel computation), also at every $\varepsilon$; compounded over depth $L$, this is $1.0998^{L}$ versus $1.129^{L}$ versus $1$, roughly $9.8\times$ and $18.4\times$ inflation at $L=24$ against none. Second, a solution-class gate with $\mathrm{supp}\,\nu\subset[0,\infty)$ is a valid activation for an input-convex network at every $\varepsilon$, and with $\mathrm{supp}\,\nu\subset[0,1]$ is additionally a valid proximal operator; SiLU and GELU satisfy neither property, at any $\varepsilon$. Third, for the measure class, Theorem~\ref{thm:gelu_obstruction} upgrades ``no known representation'' to an impossibility: no positive Gibbs family at any temperature reproduces GELU, so a GELU-gated unit cannot carry probability-valued, softmax-style attribution the way a solution- or moment-class unit does.

\subsubsection*{Multivariate obstruction via line restriction}

The scalar obstructions extend to every dimension for the elementwise activations, because the relevant classes are closed under restriction to an affine line.

\begin{lemma}[Line-restriction closure]
\label{lem:line_restriction}
Fix $\varepsilon>0$ and $z\in\mathbb{R}^N$.  Along any affine line $\ell(s)=z_0+su$ with $u\in\mathbb{R}^N$:
\textup{(a)} if $G(z)=\varepsilon\log M(z)+c$ with $M(z)=\int_{\mathbb{R}^N}e^{\langle a,z\rangle/\varepsilon}\,d\nu(a)$, $\nu\ge0$, then $G(\ell(s))=\varepsilon\log\widetilde M(s)+c$ where $\widetilde M(s)=\int_{\mathbb{R}}e^{sb/\varepsilon}\,d\rho(b)$ and $\rho\ge0$ is the pushforward of $e^{\langle a,z_0\rangle/\varepsilon}\,d\nu(a)$ under $a\mapsto\langle a,u\rangle$;
\textup{(b)} the $u$-form $\tfrac{|z|^2}{4t}-\varepsilon\log M+c$ restricts to $\tfrac{|u|^2s^2}{4t}+(\text{affine in }s)-\varepsilon\log\widetilde M(s)+c$, a scalar $u$-form when $|u|=1$;
\textup{(d)} a heat object $e^{s'\Delta}\mu$ is entire on $\mathbb{C}^N$, so its restriction is entire on $\mathbb{C}$.
\end{lemma}

\begin{proof}
Substituting $z=z_0+su$ gives $\langle a,z\rangle=\langle a,z_0\rangle+s\langle a,u\rangle$, so $M(\ell(s))=\int e^{\langle a,z_0\rangle/\varepsilon}e^{s\langle a,u\rangle/\varepsilon}\,d\nu(a)=\int_{\mathbb{R}}e^{sb/\varepsilon}\,d\rho(b)$ with $\rho\ge0$ as stated (positive since $e^{\langle a,z_0\rangle/\varepsilon}>0$).  Part (b) is $|z_0+su|^2=|u|^2s^2+2s\langle z_0,u\rangle+|z_0|^2$.  Part (d) holds because $e^{-|z-y|^2/(4s')}$ is entire in $z\in\mathbb{C}^N$ and the flow of a tempered $\mu$ inherits entirety. \qed
\end{proof}

\begin{theorem}[Multivariate obstruction]
\label{thm:multivariate_obstruction}
Let $A\colon\mathbb{R}^N\to\mathbb{R}^N$ be the elementwise activation $A(z)_i=z_i\,\sigma(z_i/\varepsilon)$ \textup(SiLU\textup) or $z_i\,\Phi(z_i/\varepsilon)$ \textup(GELU\textup).  For every $N$ and every $i$, the component $A(\cdot)_i$ lies in neither the multivariate solution class \textup{(a)} nor the $u$-form \textup{(b)} of Lemma~\ref{lem:line_restriction}; for SiLU it additionally lies in no heat class \textup{(d)}.
\end{theorem}

\begin{proof}
$A(\cdot)_i$ depends only on $z_i$, and its restriction to $\ell(s)=s\,e_i$ is the scalar $\mathrm{SiLU}_\varepsilon(s)$ (resp. $\mathrm{GELU}_\varepsilon(s)$).  Were $A(\cdot)_i$ in a multivariate class \textup{(a)} or \textup{(b)}, Lemma~\ref{lem:line_restriction} would place the scalar activation in the corresponding scalar class along $e_i$, contradicting the scalar obstruction: Theorem~\ref{thm:silu_obstruction} for SiLU, and Theorem~\ref{thm:gelu_obstruction} together with the $u$-form exclusion of Remark~\ref{rem:gelu_noscalar} for GELU.  For \textup{(d)}, entirety of the restriction contradicts the pole of $\mathrm{SiLU}_\varepsilon$ at $s=i\pi\varepsilon$ (residue $i\pi\varepsilon^2$); GELU, being entire, is excluded by \textup{(a)} instead. \qed
\end{proof}

\begin{remark}
The elementwise identities lift coordinatewise, so Theorem~\ref{thm:multivariate_obstruction} closes the obstruction in every dimension for the activations used in practice.  The same restriction applies to the coupled attention readout $\langle V,\nabla\mathrm{LSE}_\varepsilon(Wz)\rangle$ (Proposition~\ref{prop:l2attn}): its profile along a generic line is a scalar Gibbs moment, which for generic $V,W$ inherits the SiLU-type obstruction, though a bound uniform over all $V,W$ is not pursued here.
\end{remark}

\begin{remark}[GELU and the commutative diagram]
The GELU gate is read at the fixed gauge $t=\varepsilon/2$ (diffusion time
$\varepsilon^2/2$), a section $t(\varepsilon)$ of the time-parameter gauge family
of Theorem~\ref{thm:nn_pde}; each fixed $\varepsilon$ is a legitimate member, so
the $\varepsilon$-indexing is preserved pointwise.  No obstruction to the
commutative diagram (Theorem~\ref{thm:diagram}) is found at this gauge, and its
network--tropical edge closes, but full four-vertex closure in the functorial
sense of the LSE class is \emph{not} claimed.
\emph{(i) Network--tropical edge.} $\mathrm{GELU}_\varepsilon\to\mathrm{ReLU}
=\max(x,0)$ uniformly (Lemma~\ref{lem:gelu_heat}), and since
$\mathrm{GELU}_\varepsilon(x)=\varepsilon\,\mathrm{GELU}_1(x/\varepsilon)$ is
Lipschitz \emph{uniformly in} $\varepsilon$ (constant
$\Phi(\sqrt2)+\sqrt2\,\varphi(\sqrt2)\approx1.13$), the sup-norm error telescopes
through finitely many Lipschitz layers, closing the $\varepsilon\to0$ square at
fixed depth (value level; gradient edges are excluded, as for
$\mathrm{softmax}\to\arg\max$).
\emph{(ii) PDE vertex.} The semigroup \emph{operator} $e^{(\varepsilon^2/2)\Delta}$
belongs to the paper's viscous family ($v_t=\varepsilon\Delta v$ run to HJ-time
$t=\varepsilon/2$).  The \emph{object} $\mathrm{GELU}_\varepsilon$, however, is the
flow of the \emph{signed}, $\varepsilon$-dependent datum
$\mathrm{ReLU}-\varepsilon^2\delta_0$, not the Hopf--Cole transform of a fixed
initial-value problem, and its $\varepsilon\to0$ limit is reached by mollifier
shrinkage rather than the Varadhan/Laplace dequantization of the LSE column; it
reaches the same tropical corner through a different arrow.
\emph{(iii) Composition.} A GELU layer has no $u$-side reading
($\mathrm{GELU}_\varepsilon<0$ for $x<0$), so it is \emph{not} a semigroup factor
$S_{t_\ell}$.  What holds is interleaving with re-entry: being Lipschitz uniformly
in $\varepsilon$, a GELU output is an admissible Lipschitz initial datum
$g^{(\ell+1)}$ for the next LSE layer, so the network composes at the function
level and every downstream LSE layer re-enters the dictionary of
Theorem~\ref{thm:finite_depth} with full HJ semantics.
The honest status is therefore: no obstruction found; the tropical edge and the
gate column close at the locked gauge; measure-class blocks interleave with
re-entering semigroup factors.  The value $\mathrm{GELU}_\varepsilon$ has no object
at the solution vertex, which is the content of
Theorem~\ref{thm:gelu_obstruction}, not a defect of the diagram.
\end{remark}

%\newpage
%\input{checklist.tex}

\end{document}